\documentclass{article}

     \usepackage[final]{neurips_2023}

\usepackage[utf8]{inputenc} 
\usepackage[T1]{fontenc}    
\usepackage{hyperref}       
\usepackage{url}            
\usepackage{booktabs}       
\usepackage{amsfonts, amsmath}       
\usepackage{nicefrac}       
\usepackage{microtype}      
\usepackage{xcolor}         
\usepackage{bbm}
\usepackage{times}
\usepackage{algorithm, algorithmic}
\usepackage{mathtools}
\usepackage{latexsym}
\usepackage{relsize}
\usepackage{amsthm}
\usepackage{tabularx}

\usepackage{multirow}
\usepackage{makecell}

\usepackage{tablefootnote}

\usepackage{amsmath,amsfonts,amssymb,mathrsfs,amsthm,bm}
\usepackage{mathtools,natbib}
\usepackage{hyperref}
\hypersetup{colorlinks,citecolor=blue,linkcolor=black}
\usepackage{amsthm,verbatim} 
\usepackage{latexsym}
\usepackage{relsize}
\usepackage{thm-restate}
\usepackage{appendix}

\newtheorem{theorem}{Theorem}[section]

\newtheorem{proposition}[theorem]{Proposition}
\newtheorem{lemma}[theorem]{Lemma}
\newtheorem{corollary}[theorem]{Corollary}

\newtheorem{definition}[theorem]{Definition}

\newtheorem{assumption}[theorem]{Assumption}

\newcommand{\newreptheorem}[2]{%
\newenvironment{rep#1}[1]{%
 \def\rep@title{#2 \ref{##1}}%
 \begin{rep@theorem}}%
 {\end{rep@theorem}}}

\newreptheorem{theorem}{Theorem}
\newreptheorem{lemma}{Lemma}
\newreptheorem{proposition}{Proposition}
\newreptheorem{claim}{Claim}
\newreptheorem{corollary}{Corollary}
\newreptheorem{mainlemma}{Main Lemma}

\newtheorem{remark}[theorem]{Remark}

\usepackage[capitalise]{cleveref}
\usepackage{xcolor}
\usepackage{dsfont}

\newcommand{\defeq}{\stackrel{\text{def}}{=}}

\numberwithin{equation}{section}

\def\Y{{\mathcal Y}}

\newcommand{\A}{\mathcal{A}}

\newcommand{\y}{\ensuremath{\mathbf y}}

\newcommand{\K}{\ensuremath{\mathcal K}}
\newcommand{\M}{\ensuremath{\mathcal M}}

\def\e{\mathbf{e}}

\def\x{\mathbf{x}}
\def\y{\mathbf{y}}
\def\w{\mathbf{w}}

\def\regret{\text{Regret}}

\newcommand{\ignore}[1]{}

\newcommand{\R}{\mathbb{R}}

\newcommand{\wadv}{\matw^{\mathrm{adv}}}
\newcommand{\eadv}{\mate^{\mathrm{adv}}}
\newcommand{\estoch}{\mate^{\mathrm{stoch}}}
\newcommand{\wstoch}{\matw^{\mathrm{stoch}}}

\newcommand{\matw}{\mathbf{w}}

\newcommand{\mate}{\mathbf{e}}

\makeatletter
\newcommand{\neutralize}[1]{\expandafter\let\csname c@#1\endcsname\count@}
\makeatother

\newtheorem*{theorem*}{Theorem}
\newtheorem*{lemma*}{Lemma}
\newtheorem*{corollary*}{Corollary}
\newtheorem*{proposition*}{Proposition}
\newtheorem*{claim*}{Claim}
\newtheorem*{fact*}{Fact}
\newtheorem*{observation*}{Observation}

\theoremstyle{definition}

\newtheorem*{definition*}{Definition}
\newtheorem*{remark*}{Remark}
\newtheorem*{example*}{Example}

 \theoremstyle{plain}

\DeclareMathAlphabet{\mathbfsf}{\encodingdefault}{\sfdefault}{bx}{n}

\DeclareMathOperator*{\argmin}{arg\,min}

\newcommand{\E}{\mathbb{E}}

\newcommand{\poly}{\mathrm{poly}}

\newcommand{\reals}{\mathbb{R}}
\newcommand{\eps}{\varepsilon}

\renewcommand{\geq}{~\ge~}

\let\oldtfrac\tfrac
\renewcommand{\tfrac}[2]{\smash{\oldtfrac{#1}{#2}}}

\let\nablaold\nabla
\renewcommand{\nabla}{\nablaold\mkern-2.5mu}

\def\xnat{\mathbf{x}^{\mathbf{nat}}}
\def\ynat{\mathbf{y}^{\mathbf{nat}}}

\newcommand{\uv}{\ensuremath{\mathbf u}}

\newcommand{\pF}{\mathring{F}}
\newcommand{\pf}{\mathring{f}}
\newcommand{\pS}{\mathring{\mathbf{S}}}
\newcommand{\pC}{\mathring{\mathbf{C}}}

\title{Optimal Rates for Bandit Nonstochastic Control}

\author{%
  Y. Jennifer Sun $^*$  \\
  Princeton University\\
  \texttt{ys7849@princeton.edu} \\
   \And
   Stephen Newman $^*$ \\
  Princeton University\\
  \texttt{sn9581@princeton.edu} \\
   \AND
  Elad Hazan \\
  Princeton University \& Google DeepMind \\
 \texttt{ehazan@princeton.edu} \\
}

\begin{document}

\maketitle

\begin{abstract}
Linear Quadratic Regulator (LQR) and Linear Quadratic Gaussian (LQG) control are foundational and extensively researched problems in optimal control. We investigate LQR and LQG problems with semi-adversarial perturbations and time-varying adversarial bandit loss functions. The best-known sublinear regret algorithm of~\cite{gradu2020non} has a $T^{\frac{3}{4}}$ time horizon dependence, and the authors posed an open question about whether a tight rate of $\sqrt{T}$ could be achieved. We answer in the affirmative, giving an algorithm for bandit LQR and LQG which attains optimal regret (up to logarithmic factors) for both known and unknown systems.
A central component of our method is a new scheme for bandit convex optimization with memory, which is of independent interest. 
\end{abstract}


\section{Introduction}
Linear-Quadratic Regulator (LQR) and the more general Linear-Gaussian (LQG) control problems have been extensively studied in the field of control theory due to their wide range of applications  and admittance of analytical solutions by the seminal works of \cite{bellman1954theory} and \cite{kalman1960new}. LQR and LQG control problems study the design of a feedback control policy for a linear dynamical system with the goal of minimizing cumulative, possibly time-varying quadratic costs. The discrete version of the problem studies the control of the following linear dynamical system governed by dynamics $(A,B,C)$ \footnote{The LQR/LQG dynamics can be generalized to time-varying linear dynamical systems. Here we restrict ourselves to linear time-invariant systems for simplicity.}:
\begin{align*}
\x_{t+1} = A \x_t + B \uv_t + \w_t \,\,, \ \y_{t}  = C \x_{t}+\e_t\, ,
\end{align*}
where at time $t$, $\x_t$ represents the system's state, $\uv_t$ represents the control exerted on the system, and $\{\w_t\}_{t=1}^T$ represents a sequence of i.i.d. centered Gaussian perturbations injected to the system. In the generality of LQG, the system's states are not accessible. Instead, the algorithm has access to an observation $\y_t$, which is a linear function of state perturbed by a sequence $\{\e_t\}_{t=1}^T$ of i.i.d. centered Gaussian noises. The cost is a quadratic function of both the observation and the control exerted. The goal in LQR/LQG problems is to find a control policy $\pi$ in some policy class $\Pi$ that minimizes the cumulative cost over a finite time horizon $T$. With $\y_t^{\pi},\uv_t^{\pi}$ denoting the observation and control at time $t$ resulted from executing policy $\pi$, the objective is formally given by
\begin{align*}
\underset{\pi\in\Pi}{\mathrm{minimize}} \ \ \  J_T(\pi)\defeq \sum_{t=1}^T c_t(\y_t^{\pi},\uv_t^{\pi})=\sum_{t=1}^T {\y_t^{\pi}}^{\top}Q_t\y_t^{\pi}+{\uv_t^{\pi}}^{\top} R_t\uv_t^{\pi}. 
\end{align*}
Variations of this problem have garnered considerable interest. In the recent literature of online nonstochastic control, several setting-based generalizations to the linear-control framework have been explored, including
\begin{itemize}
  \item Adversarially chosen cost functions that are not known in advance (\cite{agarwal2019online}). This generalization is important for a variety of real-world applications with model-external negative feedback, including zero-sum game-playing and defending against adversarial learning \citep{lowd2005adversarial} in applications.

  \item Adversarial perturbations in the dynamics (\cite{agarwal2019online}), which permit the modeling of misspecification and nonstochastic noise (\cite{ghai2022robust}). 
  
  \item The more challenging case of \textit{bandit control} (\cite{gradu2020non}, \cite{cassel2020bandit}, \cite{ghai2023online}), where only the cost incurred may be observed, and no gradient or higher order information is available. Recently, bandit control has seen its applications in model-free RL and meta optimizaton (\cite{chen2023nonstochastic}). 

\end{itemize}

Taken together, these settings give rise to a general setting in differentiable reinforcement learning that strictly contains a variety of classical problems in optimal and robust control. 
Naturally, when adversarial costs and perturbations are considered, an optimal solution is not defined a priori. Instead, the primary performance metric is \emph{regret}: the difference between the total cost of a control algorithm and that of the best controller from a specific policy class in hindsight.

This general setting of bandit online control was considered in the recent work of~\cite{gradu2020non}, whose proposed Bandit Perturbation Controller (\texttt{BPC}) algorithm has a provable regret guarantee of $\tilde{O}(T^{\frac{3}{4}})$ when compared with the policy class of disturbance action controllers for fully observed systems. Similar setting has also been studied by~\cite{cassel2020bandit}, who established an optimal regret up to logarithmic factor of $\tilde{O}(\sqrt{T})$ for fully observable systems under stochastic perturbations and adversarially chosen cost functions. However, bandit control for partially observable systems (e.g. LQG) is less understood. Thus, these developments in the search for efficient, low-regret bandit online control algorithms leave a central open question (also stated by by~\cite{gradu2020non}):

\noindent\fbox{\begin{minipage}{\dimexpr\textwidth-2\fboxsep-2\fboxrule\relax}
\centering
Can we achieve optimal regret $O(\sqrt{T})$ with \textbf{bandit LQG} and \textbf{nonstochastic noise}?
\end{minipage}}

Our work answers this question up to logarithmic factors. Our novel Ellipsoidal Bandit Perturbation Controller (\texttt{EBPC}) achieves a $\tilde{O}(\sqrt{T})$ regret guarantee in the presence of semi-adversarial perturbations in bandit LQG problems with strongly convex cost functions, with the additional generality of possibly unknown system dynamics. By \cite{shamir2013complexity}, this is asymptotically optimal up to logarithmic factors, as bandit optimization over quadratics reduces to bandit control with quadratic losses under $A=0, B=I$. Our work therefore resolves the upper-bound/lower-bound gap for this generalization of LQR/LQG. The following table gives a comprehensive comparison between the regret guarantee of \texttt{EBPC} and existing results in literature. 

\begin{table}[h!]
\caption{Comparison of previous results to our contributions. \label{table:main}}
\begin{center}
\begin{tabular}{ |c|c|c|c|c|c| }
 \hline
Algorithm & Noise & Observation & Feedback & System & Regret \\
 \hline
 \cite{agarwal2019online} & Adversarial & full & full &  known &  $\tilde{O}(\sqrt{T})$\\ 
 \hline
 \cite{agarwal2019logarithmic} & Stochastic & full & full & known & $\tilde{O}(1)$\\ 
 \hline
\cite{foster2020logarithmic} & Adversarial  & full  & full &  known &$\tilde{O}(1)$\\
 \hline
\cite{simchowitz2020improper} & Semi-Adv. & partial & full &  known &$\tilde{O}(1)$\\
 \hline
 \cite{simchowitz2020improper} & Semi-Adv. & partial & full &  unknown &$\tilde{O}(\sqrt{T})$\\
 \hline
\cite{gradu2020non} & Adversarial & full  & bandit &  unknown &$\tilde{O}(T^{\frac{3}{4}}) $ \\
 \hline
\cite{cassel2020bandit} & Stochastic & full  & bandit & known & $\tilde{O}(\sqrt{T}) $ \\
 \hline
 \cite{cassel2020bandit} & Adversarial & full  & bandit & known & $\tilde{O}(T^{\frac{2}{3}}) $ \\
 \hline
 \textbf{Theorem~\ref{thm:control-regret-known}} & \textbf{Semi-Adv.} & \textbf{partial} &  \textbf{bandit} & \textbf{known} & \textbf{$\tilde{O}(\sqrt{T})$}\\
 \hline
 \textbf{Theorem~\ref{thm:control-regret-unknown}} & \textbf{Semi-Adv.} & \textbf{partial} &  \textbf{bandit} & \textbf{unknown} & \textbf{$\tilde{O}(\sqrt{T})$}\\
 \hline
\end{tabular}
\end{center}
\end{table}

\subsection{Related work}

\paragraph{Online Nonstochastic Control and Online LQR.} 
In the last decade, much research has been devoted to the intersection of learning and control. \cite{abbasi2011regret} and \cite{ibrahimi2012efficient} considered the problem of learning a controller in LQR for known quadratic cost functions and stochastic/martingale difference perturbation sequence when the dynamics of the system is unknown, and achieved $\tilde{O}(\sqrt{T})$-regret in this case. \cite{dean2018regret} provided the first provable low-regret, efficient algorithm to solve LQR problems with known cost functions and stochastic perturbations. \cite{cohen2018online} extended this result to changing quadratic costs with stochastic perturbations and provided a regret guarantee of $O(\sqrt{T})$. \cite{lale2021adaptive} consider the LQG problem with stochastic noise and unknown systems.

More recently, interest has turned to \textit{nonstochastic control}, in which the cost functions and the perturbations can be adversarially chosen \cite{agarwal2019online}. A broad spectrum of control problems were reconsidered from the nonstochastic perspective, and several different generalizations were derived. To highlight a few: 
\begin{itemize}
  \setlength{\itemsep}{1pt}
  \setlength{\parskip}{0pt}
  \setlength{\parsep}{0pt}
    \item \cite{agarwal2019logarithmic} showed $O(\mathrm{poly}(\log T))$-regret for adversarially chosen strongly convex cost functions and stochastic noises, \cite{hazan2020nonstochastic} extended the setting of \cite{agarwal2019online} to unknown systems and achieved $O(T^{\frac{2}{3}})$-regret, and \cite{simchowitz2020making} tightened this bound to $\tilde{O}(\sqrt{T})$ and $O(\mathrm{poly}(\log T))$ for known systems. These approaches in studying the control of an unknown system depend on oracle access to a linear stabilizing controller.
    \item \cite{chen2021black} relaxed this assumption and provided the first efficient, low-regret algorithm for online nonstochastic control under the assumption that the system is controllable.
    \item \cite{cassel2020logarithmic,simchowitz2020naive,plevrakis2020geometric,cassel2021online} showed an $\Omega(\sqrt{T})$ regret lower bound for unknown systems in LQR with full cost feedback.
\end{itemize}
See \cite{hazan2022introduction} for a comprehensive text detailing these results.



\paragraph{Online Bandit Convex Optimization with Memory.} A classical approach to control of stable/stabilizable linear dynamical systems is to reduce control problems to online convex optimization with memory. In our setting, the learner iteratively plays a decision $x_t$ in a convex set $\K \subseteq \reals^d$ and suffers an adversarially chosen loss $F_t(x_{t-H+1:t})$, where $x_{t-H+1:t}$ is the sequence of points $x_{t-H+1},...,x_{t}$. In particular, the loss depends on the last $H$ points played by the algorithm, and the only information revealed to the learner is the scalar loss that they incurred. The goal is to minimize \textit{regret}, the difference between the loss actually suffered and the loss suffered under the best single play in hindsight:
\begin{align*}
\regret_T\defeq \sum_{t=H}^T F_t(x_{t-H+1:t})-\min_{x\in\K}\sum_{t=H}^T F_t(x,\dots,x).
\end{align*}
Since the loss function is unknown to the learner in the bandit setting, online bandit convex optimization algorithms make use of low-bias estimators of the true gradient or Hessian. Therefore, it is standard to measure the algorithm's performance by \textit{expected regret} over the stochasticity injected when creating such estimators. 

Online convex optimization with memory in the full information setting, where the loss function is known to the learner, was proposed by \cite{anava2015online}. The work of \cite{agarwal2019online}, was the first to connect this to control, and to give a regret bound for online control with adversarial perturbations.

In the bandit setting, \cite{gradu2020non} used bandit convex optimization with memory to derive regret bounds for online control with bandit feedback. Their work builds upon the bandit convex optimization method of \cite{flaxman2005online} to obtain a $\tilde{O}(T^{\frac{3}{4}})$ regret bound for general convex loss functions. Recently, the work of \cite{ghai2023online} improves upon \cite{gradu2020non} in its dimension dependence, improving the algorithm's applicability to high-dimensional system. 
 
We focus on the time dependence in the bandit LQR/LQG setting, where the loss functions are strongly convex and smooth. It is thus natural to use the techniques of \cite{hazan2014bandit}, who obtained a $\tilde{O}(\sqrt{T})$ regret guarantee for bandit convex optimization without memory. This bound is tight up to logarithmic factors as proved by \cite{shamir2013complexity}. 

\paragraph{Online Learning with Delay.}
One technical difficulty in extending OCO with memory to the bandit setting arises from the requirement of independence between every play and the noises injected from the recent $H$ steps. We resolve this issue by adapting online learning with delay to the subroutine algorithm used in our BCO algorithm. Online learning with delay was introduced by \cite{quanrud2015online}. In particular, \cite{flaspohler2021online} relates online learning with delay to online learning with optimism and established a sublinear regret guarantee for mirror descent algorithms. A similar delay scheme was seen in \citep{gradu2020non}.

\subsection{Notations and organization}


\paragraph{Notation.} For convenience, we denote $\bar{H}\defeq H-1$. We use lowercase bold letters (e.g. $\x,\y,\uv,\w,\e$) to denote the states, observations, controls, and noises of the dynamical system, and $d_{\x}, d_{\uv}, d_{\y}$ to denote their corresponding dimensions. We use $S^{n-1}$ to denote the unit sphere in $\R^n$, as $S^{n-1}\cong \R^{n-1}$. For a differentiable function $F:(\R^n)^H\rightarrow\R$, we denote the gradient of $F$ with respect to its $i$th argument vector by $\nabla_{i} F(\cdot)$. $\rho(\cdot)$ acting on a square matrix measures the spectral radius of the matrix. For a sequence $M=(M^{[i]})_{i\in I}$, we use $\|M\|_{\ell_1,\mathrm{op}}$ to denote the sum of the operator norm: $\|M\|_{\ell_1,\mathrm{op}}\defeq\sum_{i\in I}\|M^{[i]}\|_{\mathrm{op}}$. 
We use $O(\cdot)$ to hide all universal constants, $\tilde{O}(\cdot)$ to hide $\mathrm{poly}(\log T)$ terms, and $\mathcal{O}(\cdot)$ to hide all natural parameters.

\paragraph{Organization.} Our method has two main components: a novel algorithm for BCO with memory (\texttt{EBCO-M}), and its application to building a novel bandit perturbation controller (\texttt{EBPC}). Section~\ref{sec:ons-setup} describes our problem setting. Section~\ref{sec:bco-m} gives \texttt{EBCO-M} and its near-optimal regret guarantee. Section~\ref{sec:controller} introduces \texttt{EBPC} and its regret guarantees to both known and unknown systems.

\section{The Bandit LQG Problem} \label{sec:ons-setup}

In this section we provide necessary background and describe and formalize the main problem of interest. 
We consider control of linear time-invariant dynamical systems of the form
\begin{align}
\x_{t+1} = A \x_t + B \uv_t + \w_t \,\,, \ \y_{t}  = C \x_{t}+\e_t\,\, . \label{eqn:lds}
\end{align}
with dynamics matrices $A\in\R^{d_{\x}\times d_{\x}},B\in\R^{d_{\x}\times d_{\uv}},C\in\R^{d_{\y}\times d_{\x}}$. Here, consistent with previous notations, $\x_t\in\R^{d_{\x}}$ is the state of the system at time $t$, $\uv_t\in\R^{d_{\uv}}$ is the control applied at time $t$, and $\w_t\in\R^{d_{\x}}, \e_t\in \R^{d_{\y}}$ are the system and measurement perturbations. At each timestep, the learner may observe $\y_t\in\R^{d_{\y}}$, which usually represents a possibly noisy projection of the state $\x_t$ onto some (possibly and usually) low-dimensional space. 

In the online bandit setting, the learner is asked to perform a control $\uv_t$ at time $t$. After the control is performed, the adversary chooses a quadratic cost function $c_t(\y_{t},\uv_{t})$. The learner observes the scalar $c_t(\y_t,\uv_t)\in\R_{+}$ and the signal $\y_t$, but no additional information about $c_t(\cdot,\cdot)$, $\x_t$, $\w_t$, or $\e_t$ is given. The goal of the learner is to minimize \emph{expected regret}, where \emph{regret} against the controller class $\Pi$ is defined as
\begin{align}
\regret_T\defeq \sum_{t=1}^T c_t (\y_t ,\uv_t) - \min_{\pi \in \Pi} \sum_{t=1}^T c_t(\y^\pi_t , \uv^\pi_t)\ \label{eqn:regret-def}, 
\end{align}
where $\uv_t^{\pi}$ is the control exerted at time $t$ by policy $\pi$ and $\y_t^{\pi}$ is the time-$t$ observation that would have occurred against the same costs/noises if the control policy $\pi$ were carried out from the beginning. 

In controlling linear dynamical systems with partial observations, we often make use of the system's counterfactual signal had no controls been performed since the beginning of the instance:

\begin{definition}[Nature's $\y$]
\label{def:nature-y}
Nature's $\y$ at time $t$, denoted by $\ynat_t$, is the signal that the system would have generated at time $t$ under $\uv_{1:t}=0$. We may compute this as
\begin{align*}
\xnat_{t+1}=A\xnat_t+\w_t\ , \ \ \ \ynat_{t}=C\xnat_{t},
\end{align*}
or, equivalently, $\ynat_t=\e_t+\sum_{i=1}^{t-1}CA^{t-i-1}\w_i$. 
\end{definition}

Critically, this may be calculated via the Markov operator:

\begin{definition}[Markov operator] 
\label{def:markov-operator}
The Markov operator $G=[G^{[i]}]_{i\in\mathbb{N}}$ corresponding to a linear system parametrized by $(A, B, C)$ as in Eq.(\ref{eqn:lds}) is a sequence of matrices in $\R^{d_{\y}\times d_{\uv}}$ such that $G^{[i]}\defeq CA^{i-1}B, \ G^{[0]}\defeq \mathbf{0}_{d_{\y}\times d_{\uv}}$. 
\end{definition}
It follows immediately that $\ynat_t$ may be computed from observations as $\ynat_t=\y_t-\sum_{i=1}^{t}G^{[i]}\uv_{t-i}$. 

\subsection{Assumptions}
We impose four core assumptions on the problem:
\begin{assumption} [Stable system]
\label{assumption:stable-system}
We assume the system is stable: the spectral radius $\rho(A)<1$. 
\end{assumption}
Note that this assumption is trivially generalized to the standard assumption that that the system has a known stabilizing controller $K$, as we may reformulate our system as stable via $A'=A+BK, B'=B$. This generalized assumption is standard in literature, and has the following important consequence:

\begin{remark} [Decay of stable systems]
That the system is stable implies that $\exists P\succ \mathbf{0}_{d_{\x}\times d_{\x}}$, $P\in\mathrm{Sym}(d_{\x})$ such that $r P\succeq  A^{\top}PA$ for some $0\le r<1$, and therefore $\exists \kappa$ depending on $\|B\|_{\mathrm{op}}, \|C\|_{\mathrm{op}}, \sigma_{\min}(P)$ such that $\|G^{[i]}\|_{\mathrm{op}}\le \kappa r^{i-1}$. Then with $H=O(\log T)$, we can assume that $\|G\|_{\ell_1,\mathrm{op}}=\sum_{i=0}^{\infty}\|G^{[i]}\|_{\mathrm{op}}\le R_G$ and $\psi_G(H)\defeq \sum_{i=H}^{\infty}\|G^{[i]}\|_{\mathrm{op}}\le \frac{R_G}{T}$. 
\end{remark}

\begin{assumption} [Noise model]
\label{assumption:noise-model}
The perturbations $\{\w_t,\e_t\}_{t=1}^T$ are assumed to be semi-adversarial: $\w_t,\e_t$ decompose as sums of adversarial and stochastic components $\w_t=\wadv_t+\wstoch_t$ and $\e_t=\eadv_t+\estoch_t$. The stochastic components of the perturbations are assumed to come from distributions satisfying $\E[\wstoch_t]=\E[\estoch_t]=0$, $\E\left[\wstoch_t{\wstoch_t}^\top\right]\succeq\sigma_{\w}^2I$, $\E\left[\estoch_t{\estoch_t}^\top\right]\succeq\sigma_{\e}^2I$, $\sigma_{\e}>0$.  $\{\w_t,\e_t\}_{t=1}^T$ are bounded such that $\|\ynat_t\|_2\le R_{\mathrm{nat}}$, $\forall t$, for some parameter $R_{\mathrm{nat}}$. 
\end{assumption}
The bound on $\ynat$ is implied by bounded noise, which is a standard assumption in literature, and the stability of the system. The semi-adversarial assumption is also seen in prior work~\citep{simchowitz2020improper}, and is a necessary condition for our analysis: we depend on the regret guarantee of a bandit online convex optimization with memory algorithm which requires the strong convexity of the expected loss functions conditioned on all but the $\Theta(\poly(\log T))$ most recent steps of history. This assumption is essentially equivalent to the adversarial assumption in applications: in almost all systems, noise is either endemic or may be injected. We also emphasize that this assumption is much weaker than that of previous optimal-rate work: even in the known-state, known-dynamic case, the previous optimal guarantee in the bandit setting depended on \textit{no} adversarial perturbation (see Table \ref{table:main}).

\begin{assumption} [Cost model]
\label{assumption:cost-model}
The cost functions $c_t(\cdot,\cdot)$ are assumed to be quadratic, $\sigma_c$-strongly convex, $\beta_c$-smooth, i.e. $c_t(\y,\uv)=\y^{\top}Q_t\y+\uv^{\top}R_t\uv$ with $\beta_c I\succeq Q_t\succeq \sigma_c I,\beta_c I\succeq R_t\succeq \sigma_c I$ $\forall t$. They are also assumed to obey the following Lipschitz condition: $\forall (\y,\uv), (\y',\uv')\in\R^{d_{\y}+d_{\uv}}$, 
\begin{align}
\label{equation:lipshcitz-condition}
|c_t(\y,\uv)-c_t(\y',\uv')|\le L_c(\|(\y,\uv)\|_2\vee\|(\y',\uv')\|_2)\|(\y-\y',\uv-\uv')\|_2.
\end{align}
\end{assumption}
These conditions are relatively standard for bandit convex optimization algorithms, and are needed for the novel BCO-with-memory algorithm which underpins our control algorithm.

\begin{assumption} [Adversary]
\label{assumption:adversary-model}
$\{c_t(\cdot,\cdot),\wadv_t,\eadv_t\}_{t=1}^T$ is chosen by the adversary ahead of time.
\end{assumption}
The oblivious adversary assumption is standard in literature (see \cite{simchowitz2020improper,gradu2020non}).

\subsection{Disturbance Response Controllers}
Regret compares the excess cost from executing our proposed control algorithm with respect to the cost of the best algorithm \textit{in hindsight} from a given policy class. In particular, low regret against a rich policy class is a very strong near-optimality guarantee. We take the comparator policy class $\Pi$ to be the set of disturbance response controllers (DRC), formally given by the following definition. 

\begin{definition} [Disturbance Response Controllers] The disturbance response controller (DRC) and the DRC policy class are defined as:
\begin{itemize}
\item A \textbf{disturbance response controller} $\pi_M$ of length $H\in\mathbb{Z}_{++}$ for stable systems is parameterized by $M=(M^{[j]})_{j=0}^{\bar{H}}$, a sequence of $H$ matrices in $\R^{d_{\uv}\times d_{\y}}$ s.t. the control at time $t$ given by $\pi_M$ is $\uv_t^{\pi_M}=\sum_{j=0}^{\bar{H}} M^{[j]}\ynat_{t-j}$. We shorthand $\uv_t^M\defeq \uv_t^{\pi_M}$. 

\item The \textbf{DRC policy class} parametrized by $H\in\mathbb{Z}_{++}$, $R\in\R_+$ is the set of all disturbance response controller with bounded length $H$ and norm $R$: $\M(H,R)=\{M=(M^{[j]})_{j=0}^{\bar{H}}\mid \|M\|_{\ell_1,\mathrm{op}}=\sum_{j=0}^{\bar{H}}\|M^{[j]}\|_{\mathrm{op}}\le R\}$. 
\end{itemize}
\end{definition}
Previous works have demonstrated the richness of the DRC policy class. In particular, Theorem 1 from \cite{simchowitz2020improper} has established that the DRC policy class generalizes the state-of-art benchmark class of stabilizing linear dynamic controllers (LDC) with error $e^{-\Theta(H)}$.

\subsection{Approach and Technical Challenges}

The classical approach in online nonstochastic control of stable/stabilizable systems is to reduce to a problem of online convex optimization with memory. This insight relies on the exponentially decaying effect of past states and controls on the present, which allows approximating the cost functions as functions of the most recent controls. 

A core technical challenge lies in the bandit convex optimization problem obtained from the bandit control problem. In the bandit setting, no gradient information is given to the learner, and thus the learner needs to construct a low-bias gradient estimator. Previous work uses the classical spherical gradient estimator proposed by~\cite{flaxman2004online}, but the regret guarantee is suboptimal. We would like to leverage the ellipsoidal gradient estimator proposed by \cite{hazan2014bandit}. However, when extending to loss functions with memory, there is no clear mechanism for obtaining a low-bias bound for general convex functions. We exploit the quadratic structure of the LQR/LQG cost functions to build \texttt{EBCO-M} (Algorithm~\ref{alg:BCO-quadratic}), which uses ellipsoidal gradient estimators. We note that even outside of the control applications, \texttt{EBCO-M} may be of independent interests in bandit online learning theory.

\section{BCO with Memory: Quadratic and Strongly Convex Functions}
\label{sec:bco-m}

As with previous works, our control algorithm will depend crucially on a generic algorithm for bandit convex optimization with memory (BCO-M). We present a new online
bandit convex optimization with memory algorithm that explores the structure of quadratic costs to achieve near-optimal regret.

\subsection{Setting and working assumptions}

In the BCO-M setting with memory length $H$, we consider an algorithm playing against an adversary. At time $t$, the algorithm is asked to play its choice of $y_t$ in the convex constraint set $\K$. The adversary chooses a loss function $F_t:\K^H\rightarrow\R_+$ which takes as input the algorithm's current play as well as its previous $\bar{H}$ plays. The algorithm then observes a cost $F_t(y_{t-\bar{H}},\dots,y_t)$ (and no other information about $F_t(\cdot)$) before it chooses and plays the next action $y_{t+1}$. The goal is to minimize regret with respect to the expected loss, which is the excessive loss incurred by the algorithm compared to the best fixed decision in $\K$:
\begin{align*}
\regret_T\defeq \sum_{t=H}^T \E[F_t(y_{t-\bar{H}},\dots,y_t)]-\min_{x\in\K} \sum_{t=H}^T \E[F_t(x,\dots,x)].
\end{align*}
For notation convenience, we will at times shorthand $y_{t-\bar{H}:t}\defeq (y_{t-\bar{H}},\dots,y_t)\in\K^H$. 

\subsubsection{BCO-M assumptions} 
\label{sec:working-assumptions}
We make the following assumptions on the loss functions $\{F_t\}_{t=H}^T$ and the constraint set $\K$.

\begin{assumption} [Constraint set]
\label{assumption:contraint-set-K}
$\K$ is convex, closed, and bounded with non-empty interior. $\textrm{diam}(\K)=\underset{z,z'\in\K}{\sup}\|z-z'\|_2\le D$. 
\end{assumption}

\begin{assumption} [Loss functions]
\label{assumption:loss-functions}
The loss functions chosen by the adversary obeys the following regularity and curvature assumptions: 
\begin{itemize} 
\item  $F_t:\K^{H}\rightarrow R_+$ is quadratic and $\beta$-smooth:
\begin{itemize}
\item Quadratic: $\exists W_t\in \R^{nH\times nH}, b_t\in\R^{nH}, c_t\in\R$ such that $F_t(w)=w^{\top}W_tw+b_t^{\top}w+c_t$, $\forall w\in\K^{H}$. 
\item Smooth: $W_t\preceq \beta I_{nH\times nH}$. 
\end{itemize}
\item  $F_t:\K^H\rightarrow\mathbb{R}_+$ is $\sigma$-strongly convex in its induced unary form: $f_t:\K\rightarrow\mathbb{R}_+$ with $f_t(z)=F_t(z,\dots,z)$ is $\sigma$-strongly convex, i.e. $f_t(z)\ge f_t(z')+\nabla f_t(z')^{\top}(z-z')+\frac{\sigma}{2}\|z-z'\|_2^2$., $\forall z,z'\in\K$. 
\item $F_t$ satisfies the following diameter and gradient bound on $\K$: $\exists B,L>0$ such that 
\begin{align*}
B=\sup_{w,w'\in\K^H} |F_t(w)-F_t(w')|, \ \ L=\sup_{w\in\K^H} \|\nabla F_t(w)\|_2. 
\end{align*}
\end{itemize}
\end{assumption}

In the online control problems, when formulating the cost function $c_t$ as a function $F_t$ of the most recent $H$ controls played, the function $F_t$ itself may depend on the entire history of the algorithm through step $t-H$. Therefore, it is essential to analyze the regret guarantee of our BCO-M algorithm when playing against an adversary that can be $(t-H)$-adaptive, giving rise to the following assumption. 

\begin{assumption} [Adversarial adaptivity]
\label{assumption:adversary}
The adversary chooses $F_t$ independently of the noise $u_{t-\bar{H}:t}$ which is drawn by the algorithm in the $H$ most recent steps, but possibly not independently of earlier noises.
\end{assumption}
Note that Assumption~\ref{assumption:adversary} is minimal for BCO: if this fails, then in the subcase of a delayed loss, the adversary may fully control the agent's observations, resulting in no possibility of learning.

\paragraph{Self-concordant barrier.} The algorithm makes use of a \textit{self-concordant barrier} $R(\cdot)$ of $\K$ as the regularization function in the updates. 
\begin{definition} [Self-concordant barrier]
A three-time continuously differentiable function $R(\cdot)$ over a closed convex set $\K\subset\R^n$ with non-empty interior is a $\nu$-self-concordant barrier of $\K$ if it satisfies the following two properties:
\begin{enumerate}
\item (Boundary property) For any sequence $\{x_n\}_{n\in\mathbb{N}}\subset \mathrm{int}(\K)$ such that $\lim_{n\rightarrow\infty}x_n=x\in\partial\K$, $\lim_{n\rightarrow\infty}R(x_n)=\infty$. 
\item (Self-concordant) $\forall x\in\mathrm{int}(\K)$, $h\in\R^n$, 
\begin{enumerate}
\item $|\nabla^3R(x)[h,h,h]|\le 2|\nabla^2 R(x)[h,h]|^{3/2}$.
\item $|\langle \nabla R(x),h\rangle|\le \sqrt{\nu}|\nabla^2 R(x)[h,h]|^{1/2}$. 
\end{enumerate}
\end{enumerate}
\end{definition}

\subsection{Algorithm specification and regret guarantee}
We present \texttt{EBCO-M} (Algorithm~\ref{alg:BCO-quadratic}) for online bandit convex optimization with memory. The key novelty is the use of an ellipsoidal gradient estimator. It is difficult to establish a low-bias guarantee for ellipsoidal gradient estimator for general convex loss functions. However, thanks to the quadratic structure of the loss functions in LQR/LQG problems, we can show provable low bias for the ellipsoidal gradient estimator, and therefore achieve optimal regret.

\begin{algorithm}
\caption{Ellipsoidal BCO with memory (\texttt{EBCO-M})}
\label{alg:BCO-quadratic}
\begin{algorithmic}[1]
\STATE Input: Convex, closed set $\K\subseteq\R^n$ with non-empty interior, time horizon $T$, memory length $H$, step size $\eta$, $\nu$-self-concordant barrier $R(\cdot)$ over $\K$, convexity strength parameter $\sigma$.
\STATE Initialize $x_t=\argmin_{x\in\K}R(x)$, $\forall t=1,\ldots,H$.
\STATE Compute $A_t=(\nabla^2 R(x_t)+\eta\sigma tI)^{-1/2}$, $\forall t=1,\ldots,H$. 
\STATE Sample $u_1,\ldots,u_{H}\sim S^{n-1}$ i.i.d. uniformly at random. 
\STATE Set $y_t=x_t+A_tu_t$, $\forall t=1,\ldots,H$. 
\STATE Set $g_t=0$, $\forall t=1,\dots,\bar{H}$. 
\STATE Play $y_1,\ldots,y_{\bar{H}}$. 
\FOR {$t = H, \ldots, T$} 
\STATE Play $y_t$, suffer loss $F_t(y_{t-\bar{H}:t})$. 
\STATE Store $g_t=nF_t(y_{t-\bar{H}:t})\sum_{i=0}^{\bar{H}}A_{t-i}^{-1}u_{t-i}$. 
\STATE Set $x_{t+1}=\argmin_{x\in\K}\sum_{s=H}^t \left(g_{s-\bar{H}}^{\top}x+\frac{\sigma}{2}\|x-x_{s-\bar{H}}\|^2\right)+\frac{1}{\eta}R(x)$. 
\STATE Compute $A_{t+1}=(\nabla^2 R(x_{t+1})+\eta\sigma (t+1)I)^{-1/2}$.
\STATE Sample $u_{t+1}\sim S^{n-1}$ uniformly at random.
\STATE Set $y_{t+1}=x_{t+1}+A_{t+1}u_{t+1}$.
\ENDFOR
\end{algorithmic}
\end{algorithm}

Before analyzing the regret, we first make note of two properties of Algorithm~\ref{alg:BCO-quadratic}.

\begin{remark} [Delayed dependence]
\label{rmk:independence}
In Algorithm \ref{alg:BCO-quadratic}, $x_t$ is independent of $u_{t-\bar{H}:t}$, $\forall t$, and therefore $A_t$ is independent of $u_{t-\bar{H}:t}$ as $A_t$ is determined by $x_t$.
\end{remark}

\begin{remark} [Correctness]
\label{rmk:correctness}
$y_t$ played by Algorithm \ref{alg:BCO-quadratic} lies in $\K$: $\|y_t-x_t\|_{\nabla^2R(x_t)}^2=\|A_tu_t\|_{\nabla^2 R(x_t)}^2\le \|u_t\|_2^2=1$, and by Proposition~\ref{prop:self-concordant-barrier}, the Dikin ellipsoid centered at $x_t$ is contained in $\K$.
\end{remark}

\begin{theorem} [EBCO-M regret with strong convexity]
\label{thm:bco-quadratic}
For any sequence of cost functions $\{F_t\}_{t=H}^T$ satisfying Assumption~\ref{assumption:loss-functions}, constraint set $\K$ satisfying Assumption~\ref{assumption:contraint-set-K}, adversary satisfying Assumption~\ref{assumption:adversary}, and $H=\mathrm{poly}\left(\log T\right)$, Algorithm~\ref{alg:BCO-quadratic} satisfies the regret bound
\begin{align*}
\regret_T(\texttt{EBCO-M})=\sum_{t=H}^T \E[F_t(y_{t-\bar{H}:t})]-\min_{x\in\K}\sum_{t=H}^T \E[f_t(x)]\le \tilde{\mathcal{O}}\left(\frac{\beta n}{\sigma}\sqrt{T}\right),
\end{align*}
where expectation is taken over the randomness of the exploration noises $u_{1:T}$, with $\tilde{\mathcal{O}}$ hiding all natural parameters ($B,D,L$) and logarithmic dependence on $T$. 
\end{theorem}

\begin{corollary} [EBCO-M regret with conditional strong convexity]
\label{cor:regret-expected-convex}
Suppose Algorithm~\ref{alg:BCO-quadratic} is run on $\K$ satisfying Assumption~\ref{assumption:contraint-set-K} against an adversary satisfying Assumption~\ref{assumption:adversary} with a sequence of cost functions $\{F_t\}_{t=H}^T$ such that 
\begin{enumerate}
    \item  $F_t$ is quadratic, convex, $\beta$-smooth, has diameter bound $B$ and gradient bound $L$.  
    \item $F_t$ is conditionally $\sigma$-strongly convex in its induced unary form: $\bar{f}_{t}(z)\defeq \E[f_t(z)\mid u_{1:t-H}, f_{H:t-H}]$ is $\sigma$-strongly convex. 
\end{enumerate}
Then, Algorithm~\ref{alg:BCO-quadratic} satisfies the same regret bound attained in Theorem~\ref{thm:bco-quadratic}, i.e. 
\begin{align*}
\regret_T(\texttt{EBCO-M})=\sum_{t=H}^T \E[F_t(y_{t-\bar{H}:t})]-\min_{x\in\K}\sum_{t=H}^T \E[f_t(x)]\le \tilde{\mathcal{O}}\left(\frac{\beta n}{\sigma}\sqrt{T}\right).
\end{align*}
\end{corollary} 


\section{Bandit Controller: Known and Unknown Systems}
\label{sec:controller}
We will now use our BCO-with-memory algorithm to find an optimal controller (as in \cite{gradu2020non}), arguing that regret in choice of controller transfers into the setting discussed in the previous section. We first consider the case where the system is known, and then reduce the unknown system case to the known system case.

\subsection{Known systems}
Applying Algorithm \ref{alg:BCO-quadratic} to predict controllers\footnote{Notation: while our controller $M$ is typically a tensor, it should be thought of as the output vector of Algorithm \ref{alg:BCO-quadratic}. As such, the relevant vector and matrix operations in that algorithm will correspond to tensor operations here, and the notation reflects that correspondence. In particular, the inner product on line 11 is an all-dimension tensor dot product and $A$ is a square ``matrix'' which acts on tensors of shape $(H, d_\uv, d_\y)$.} with losses given by control losses, we obtain Algorithm \ref{alg:BCO-control}.
\begin{algorithm}
\caption{Ellipsoidal Bandit Perturbation Controller (\texttt{EBPC})}
\label{alg:BCO-control}
\begin{algorithmic}[1]
\STATE Input: Time horizon $T$, memory length $H$, Markov operator $G$. BCO-M parameters $\sigma, \eta$. Self-concordant barrier $R(\cdot)$ over $\M(H,R)\subset \R^{H\times d_{\uv}\times d_{\y}}$. 
\STATE Initialize $M_1=\dots=M_H=\underset{{M\in\M(H,R)}}{\argmin} R(M)$.
\STATE Compute $A_i=(\nabla^2 R(M_i)+\eta\sigma tI)^{-1/2}$, $\forall i=1,\dots,H$. 
\STATE Sample $\eps_1,\dots,\eps_H\sim S^{H\times d_{\uv}\times d_{\y}-1}$ i.i.d. uniformly at random. 
\STATE Set $\widetilde{M}_i=M_i+\eps_i$, $\forall i=1,\dots,H$. Set $g_i=0$, $\forall i=1,\dots,\bar{H}$. 
\STATE Play control $\uv_{i}=0$, incur cost $c_i(\y_i,\uv_i)$, $\forall i=1,\dots,\bar{H}$. 
\FOR {$t = H, \ldots, T$}
\STATE Play control $\uv_t=\uv_{t}(\widetilde{M}_t)=\sum_{i=0}^{\bar{H}} \widetilde{M}_t^{[i]} \ynat_{t-i}$, incur cost $c_t(\y_t,\uv_t)$.
\STATE Observe $\y_{t+1}$ and compute signal  $\ynat_{t+1}=\y_{t+1}-\sum_{i=1}^{t}G^{[i]}\uv_{t-i}$.
\STATE Store $g_t=d_{\uv}d_{\y}H c_t(\y_t,\uv_t)\sum_{i=0}^{\bar{H}}A_{t-i}^{-1}\eps_{t-i}$.
\STATE Update $M_{t+1}=\underset{{M\in\M(H,R)}}{\argmin} \sum_{s=H}^t \left(\langle g_{s-\bar{H}}, M\rangle+\frac{\sigma}{2}\|M-M_{s-\bar{H}}\|^2\right)+\frac{1}{\eta}R(M)$. 
\STATE Compute $A_{t+1}=(\nabla^2 R(M_{t+1})+\eta\sigma (t+1)I)^{-1/2}$. 
\STATE Sample $\eps_{t+1}\sim S^{H\times d_\uv\times d_\y-1}$ uniformly at random. Set $\widetilde{M}_{t+1}=M_{t+1}+A_{t+1}\eps_{t+1}$.
\ENDFOR
\end{algorithmic}
\end{algorithm}

\begin{theorem} [Known system control regret]
\label{thm:control-regret-known}
Consider a linear dynamical system governed by known dynamics $(A,B,C)$ and the interaction model with adversarially chosen cost functions and perturbations satisfying Assumption~\ref{assumption:stable-system}, \ref{assumption:noise-model}, \ref{assumption:cost-model}, \ref{assumption:adversary-model}. Then running Algorithm~\ref{alg:BCO-control} with $H=\Theta(\mathrm{poly}(\log T))$, $\sigma=\sigma_c(\sigma_{\e}^2+\sigma_{\w}\frac{\sigma_{\min}(C)}{1+\|A\|_{\mathrm{op}}^2})$, and $\eta=\Theta\left(\frac{1}{d_{\uv}d_{\y}L_cH^3\sqrt{T}}\right)$ guarantees
\begin{align*}
\E[\regret_T(\texttt{EBPC})]\le \tilde{\mathcal{O}} \left(\frac{\beta_cd_{\uv}d_{\y}}{\sigma_c}\sqrt{T}\right),
\end{align*}
where regret is defined as in Eq.(\ref{eqn:regret-def}), the expectation is taken over the exploration noises $\eps_{1:T}$ of the algorithm as well as the stochastic components $\{\wstoch_t, \estoch_t\}_{t=1}^T$ of the perturbations $\{\w_t,\e_t\}_{t=1}^T$, and $\tilde{\mathcal{O}}(\cdot)$ hides all universal constants, natural parameters, and logarithmic dependence on $T$. 
\end{theorem}

\subsection{Unknown systems: control after estimation}

Note that \texttt{EBPC} (Algorithm~\ref{alg:BCO-control}) relies on the access to the system's Markov operator $G$, which is available if and only if the system dynamics $(A,B,C)$ are known. When the system dynamics is unknown, we can identify the system using a system estimation algorithm, obtain an estimated Markov operator $\hat{G}$, and run \texttt{EBPC} with $G\leftarrow \hat{G}$. Algorithm~\ref{alg:est} outlines the estimation method of system dynamics via least squares. 
\begin{algorithm}
\caption{System estimation via least squares (\texttt{SysEst-LS})}
\label{alg:est}
\begin{algorithmic}[1]
\STATE Input: estimation sample size $N$, system length $H$.
\STATE Initialize: $\hat{G}^{[t]}=0$, $\forall t\ge H$. 
\STATE Sample and play $\uv_t\sim N(0,I_{d_{\uv}\times d_{\uv}})$, $\forall t=1,\dots,N$. 
\STATE Set $\hat{G}^{[0:\bar{H}]}=\argmin \sum_{t=H}^N\|\y_t-\sum_{i=0}^{\bar{H}}\hat{G}^{[i]}\uv_{t-i}\|_2^2$. 
\STATE Return $\hat{G}$. 
\end{algorithmic}
\end{algorithm}

\begin{theorem} [Unknown system control regret]
\label{thm:control-regret-unknown}
Consider a linear dynamical system governed with unknown dynamics $(A,B,C)$ and the interaction model with adversarially chosen cost functions and perturbations satisfying Assumption~\ref{assumption:stable-system}, \ref{assumption:noise-model}, \ref{assumption:cost-model}, \ref{assumption:adversary-model}. Suppose we obtain an estimated Markov operator $\hat{G}$ from Algorithm~\ref{alg:est} with $N=\lceil \sqrt{T}\rceil$ and $H=\Theta(\mathrm{poly}\log T)$. Then Algorithm~\ref{alg:BCO-control} with $G\leftarrow\hat{G}$, $H\leftarrow 3H$, $\sigma=\frac{1}{8}\sigma_c\sigma_{\e}^2$, and $\eta=\Theta\left(\frac{1}{d_{\uv}d_{\y}L_cH^3\sqrt{T}}\right)$ guarantees
\begin{align*}
\E[\regret_T(\texttt{EBPC})]\le \tilde{\mathcal{O}} \left(\frac{\beta_cd_{\uv}d_{\y}}{\sigma_c}\sqrt{T}\right),
\end{align*}
where regret is defined as in Eq.(\ref{eqn:regret-def}), the expectation is taken over the exploration noises $\eps_{1:T}$ in Algorithm~\ref{alg:BCO-control}, the sampled Gaussian controls $\uv_{1:N}$ in Algorithm~\ref{alg:est}, and the stochastic components  $\{\wstoch_t, \estoch_t\}_{t=1}^T$ of the perturbations $\{\w_t, \e_t\}_{t=1}^T$, and $\tilde{\mathcal{O}}(\cdot)$ hides all universal constants, natural parameters, and logarithmic dependence on $T$. 
\end{theorem}

\section{Discussion and conclusion}

We solve the open problem put forth by \cite{gradu2020non} on the optimal rate for online bandit control for the case of LQR/LQG control, improving to regret $\tilde{O}(\sqrt{T})$ from $\tilde{O}(T^{\frac{3}{4}})$ in the semi-adversarial noise model and for strongly convex LQR/LQG cost functions. Our method builds upon recent advancements in bandit convex optimization for quadratic functions, providing the first near-optimal regret algorithm for bandit convex optimization with memory in a nonstochastic setting.

It would be interesting to investigate (1) whether the results can be extended to fully adversarial noise, (2) whether a similar stable controller recovery as seen in~\citep{chen2021black} for fully observable systems can be established for partially observable systems, and whether that can be incorporated to extend our result to stabilizable systems even without access to a stabilizing controller. 

\bibliographystyle{plainnat}
\bibliography{COLT-submission/ref.bib}

\begin{thebibliography}{37}
\providecommand{\natexlab}[1]{#1}
\providecommand{\url}[1]{\texttt{#1}}
\expandafter\ifx\csname urlstyle\endcsname\relax
  \providecommand{\doi}[1]{doi: #1}\else
  \providecommand{\doi}{doi: \begingroup \urlstyle{rm}\Url}\fi

\bibitem[Abbasi-Yadkori and Szepesv{\'a}ri(2011)]{abbasi2011regret}
Yasin Abbasi-Yadkori and Csaba Szepesv{\'a}ri.
\newblock Regret bounds for the adaptive control of linear quadratic systems.
\newblock In \emph{Proceedings of the 24th Annual Conference on Learning
  Theory}, pages 1--26. JMLR Workshop and Conference Proceedings, 2011.

\bibitem[Abernethy et~al.(2008)Abernethy, Hazan, and
  Rakhlin]{abernethy2008competing}
Jacob Abernethy, Elad Hazan, and Alexander Rakhlin.
\newblock Competing in the dark: An efficient algorithm for bandit linear
  optimization.
\newblock In \emph{21st Annual Conference on Learning Theory, COLT 2008}, 2008.

\bibitem[Agarwal et~al.(2019{\natexlab{a}})Agarwal, Bullins, Hazan, Kakade, and
  Singh]{agarwal2019online}
Naman Agarwal, Brian Bullins, Elad Hazan, Sham Kakade, and Karan Singh.
\newblock Online control with adversarial disturbances.
\newblock In \emph{International Conference on Machine Learning}, pages
  111--119. PMLR, 2019{\natexlab{a}}.

\bibitem[Agarwal et~al.(2019{\natexlab{b}})Agarwal, Hazan, and
  Singh]{agarwal2019logarithmic}
Naman Agarwal, Elad Hazan, and Karan Singh.
\newblock Logarithmic regret for online control.
\newblock \emph{Advances in Neural Information Processing Systems}, 32,
  2019{\natexlab{b}}.

\bibitem[Anava et~al.(2015)Anava, Hazan, and Mannor]{anava2015online}
Oren Anava, Elad Hazan, and Shie Mannor.
\newblock Online learning for adversaries with memory: price of past mistakes.
\newblock \emph{Advances in Neural Information Processing Systems}, 28, 2015.

\bibitem[Bellman(1954)]{bellman1954theory}
Richard Bellman.
\newblock The theory of dynamic programming.
\newblock \emph{Bulletin of the American Mathematical Society}, 60\penalty0
  (6):\penalty0 503--515, 1954.

\bibitem[Cassel and Koren(2020)]{cassel2020bandit}
Asaf Cassel and Tomer Koren.
\newblock Bandit linear control.
\newblock \emph{Advances in Neural Information Processing Systems},
  33:\penalty0 8872--8882, 2020.

\bibitem[Cassel et~al.(2020)Cassel, Cohen, and Koren]{cassel2020logarithmic}
Asaf Cassel, Alon Cohen, and Tomer Koren.
\newblock Logarithmic regret for learning linear quadratic regulators
  efficiently.
\newblock In \emph{International Conference on Machine Learning}, pages
  1328--1337. PMLR, 2020.

\bibitem[Cassel and Koren(2021)]{cassel2021online}
Asaf~B Cassel and Tomer Koren.
\newblock Online policy gradient for model free learning of linear quadratic
  regulators with {\(\surd\)}{T} regret.
\newblock In \emph{International Conference on Machine Learning}, pages
  1304--1313. PMLR, 2021.

\bibitem[Chen and Hazan(2021)]{chen2021black}
Xinyi Chen and Elad Hazan.
\newblock Black-box control for linear dynamical systems.
\newblock In \emph{Conference on Learning Theory}, pages 1114--1143. PMLR,
  2021.

\bibitem[Chen and Hazan(2023)]{chen2023nonstochastic}
Xinyi Chen and Elad Hazan.
\newblock A nonstochastic control approach to optimization.
\newblock \emph{arXiv preprint arXiv:2301.07902}, 2023.

\bibitem[Cohen et~al.(2018)Cohen, Hasidim, Koren, Lazic, Mansour, and
  Talwar]{cohen2018online}
Alon Cohen, Avinatan Hasidim, Tomer Koren, Nevena Lazic, Yishay Mansour, and
  Kunal Talwar.
\newblock Online linear quadratic control.
\newblock In \emph{International Conference on Machine Learning}, pages
  1029--1038. PMLR, 2018.

\bibitem[Dean et~al.(2018)Dean, Mania, Matni, Recht, and Tu]{dean2018regret}
Sarah Dean, Horia Mania, Nikolai Matni, Benjamin Recht, and Stephen Tu.
\newblock Regret bounds for robust adaptive control of the linear quadratic
  regulator.
\newblock \emph{Advances in Neural Information Processing Systems}, 31, 2018.

\bibitem[Flaspohler et~al.(2021)Flaspohler, Orabona, Cohen, Mouatadid, Oprescu,
  Orenstein, and Mackey]{flaspohler2021online}
Genevieve~E Flaspohler, Francesco Orabona, Judah Cohen, Soukayna Mouatadid,
  Miruna Oprescu, Paulo Orenstein, and Lester Mackey.
\newblock Online learning with optimism and delay.
\newblock In \emph{International Conference on Machine Learning}, pages
  3363--3373. PMLR, 2021.

\bibitem[Flaxman et~al.(2004)Flaxman, Kalai, and McMahan]{flaxman2004online}
Abraham~D Flaxman, Adam~Tauman Kalai, and H~Brendan McMahan.
\newblock Online convex optimization in the bandit setting: gradient descent
  without a gradient.
\newblock \emph{arXiv preprint cs/0408007}, 2004.

\bibitem[Flaxman et~al.(2005)Flaxman, Kalai, and McMahan]{flaxman2005online}
Abraham~D Flaxman, Adam~Tauman Kalai, and H~Brendan McMahan.
\newblock Online convex optimization in the bandit setting: gradient descent
  without a gradient.
\newblock In \emph{Proceedings of the sixteenth annual ACM-SIAM symposium on
  Discrete algorithms}, pages 385--394, 2005.

\bibitem[Foster and Simchowitz(2020)]{foster2020logarithmic}
Dylan Foster and Max Simchowitz.
\newblock Logarithmic regret for adversarial online control.
\newblock In \emph{International Conference on Machine Learning}, pages
  3211--3221. PMLR, 2020.

\bibitem[Ghai et~al.(2022)Ghai, Chen, Hazan, and Megretski]{ghai2022robust}
Udaya Ghai, Xinyi Chen, Elad Hazan, and Alexandre Megretski.
\newblock Robust online control with model misspecification.
\newblock In \emph{Learning for Dynamics and Control Conference}, pages
  1163--1175. PMLR, 2022.

\bibitem[Ghai et~al.(2023)Ghai, Gupta, Xia, Singh, and Hazan]{ghai2023online}
Udaya Ghai, Arushi Gupta, Wenhan Xia, Karan Singh, and Elad Hazan.
\newblock Online nonstochastic model-free reinforcement learning.
\newblock \emph{arXiv preprint arXiv:2305.17552}, 2023.

\bibitem[Gradu et~al.(2020)Gradu, Hallman, and Hazan]{gradu2020non}
Paula Gradu, John Hallman, and Elad Hazan.
\newblock Non-stochastic control with bandit feedback.
\newblock \emph{Advances in Neural Information Processing Systems},
  33:\penalty0 10764--10774, 2020.

\bibitem[Gradu et~al.(2021)Gradu, Hallman, Suo, Yu, Agarwal, Ghai, Singh,
  Zhang, Majumdar, and Hazan]{gradu2021deluca}
Paula Gradu, John Hallman, Daniel Suo, Alex Yu, Naman Agarwal, Udaya Ghai,
  Karan Singh, Cyril Zhang, Anirudha Majumdar, and Elad Hazan.
\newblock Deluca--a differentiable control library: Environments, methods, and
  benchmarking.
\newblock \emph{arXiv preprint arXiv:2102.09968}, 2021.

\bibitem[Hazan(2016)]{hazan2016introduction}
Elad Hazan.
\newblock Introduction to online convex optimization.
\newblock \emph{Foundations and Trends{\textregistered} in Optimization},
  2\penalty0 (3-4):\penalty0 157--325, 2016.

\bibitem[Hazan and Levy(2014)]{hazan2014bandit}
Elad Hazan and Kfir Levy.
\newblock Bandit convex optimization: Towards tight bounds.
\newblock \emph{Advances in Neural Information Processing Systems}, 27, 2014.

\bibitem[Hazan and Singh(2022)]{hazan2022introduction}
Elad Hazan and Karan Singh.
\newblock Introduction to online nonstochastic control.
\newblock \emph{arXiv preprint arXiv:2211.09619}, 2022.

\bibitem[Hazan et~al.(2007)Hazan, Agarwal, and Kale]{hazan2007logarithmic}
Elad Hazan, Amit Agarwal, and Satyen Kale.
\newblock Logarithmic regret algorithms for online convex optimization.
\newblock \emph{Machine Learning}, 69\penalty0 (2):\penalty0 169--192, 2007.

\bibitem[Hazan et~al.(2020)Hazan, Kakade, and Singh]{hazan2020nonstochastic}
Elad Hazan, Sham Kakade, and Karan Singh.
\newblock The nonstochastic control problem.
\newblock In \emph{Algorithmic Learning Theory}, pages 408--421. PMLR, 2020.

\bibitem[Ibrahimi et~al.(2012)Ibrahimi, Javanmard, and
  Roy]{ibrahimi2012efficient}
Morteza Ibrahimi, Adel Javanmard, and Benjamin Roy.
\newblock Efficient reinforcement learning for high dimensional linear
  quadratic systems.
\newblock \emph{Advances in Neural Information Processing Systems}, 25, 2012.

\bibitem[Kalman(1960)]{kalman1960new}
Rudolph~Emil Kalman.
\newblock A new approach to linear filtering and prediction problems.
\newblock 1960.

\bibitem[Lale et~al.(2021)Lale, Azizzadenesheli, Hassibi, and
  Anandkumar]{lale2021adaptive}
Sahin Lale, Kamyar Azizzadenesheli, Babak Hassibi, and Anima Anandkumar.
\newblock Adaptive control and regret minimization in linear quadratic gaussian
  (lqg) setting.
\newblock In \emph{2021 American Control Conference (ACC)}, pages 2517--2522.
  IEEE, 2021.

\bibitem[Lowd and Meek(2005)]{lowd2005adversarial}
Daniel Lowd and Christopher Meek.
\newblock Adversarial learning.
\newblock In \emph{Proceedings of the eleventh ACM SIGKDD international
  conference on Knowledge discovery in data mining}, pages 641--647, 2005.

\bibitem[Plevrakis and Hazan(2020)]{plevrakis2020geometric}
Orestis Plevrakis and Elad Hazan.
\newblock Geometric exploration for online control.
\newblock \emph{Advances in Neural Information Processing Systems},
  33:\penalty0 7637--7647, 2020.

\bibitem[Quanrud and Khashabi(2015)]{quanrud2015online}
Kent Quanrud and Daniel Khashabi.
\newblock Online learning with adversarial delays.
\newblock \emph{Advances in neural information processing systems}, 28, 2015.

\bibitem[Shamir(2013)]{shamir2013complexity}
Ohad Shamir.
\newblock On the complexity of bandit and derivative-free stochastic convex
  optimization.
\newblock In \emph{Conference on Learning Theory}, pages 3--24. PMLR, 2013.

\bibitem[Simchowitz(2020)]{simchowitz2020making}
Max Simchowitz.
\newblock Making non-stochastic control (almost) as easy as stochastic.
\newblock \emph{Advances in Neural Information Processing Systems},
  33:\penalty0 18318--18329, 2020.

\bibitem[Simchowitz and Foster(2020)]{simchowitz2020naive}
Max Simchowitz and Dylan Foster.
\newblock Naive exploration is optimal for online lqr.
\newblock In \emph{International Conference on Machine Learning}, pages
  8937--8948. PMLR, 2020.

\bibitem[Simchowitz et~al.(2020)Simchowitz, Singh, and
  Hazan]{simchowitz2020improper}
Max Simchowitz, Karan Singh, and Elad Hazan.
\newblock Improper learning for non-stochastic control.
\newblock In \emph{Conference on Learning Theory}, pages 3320--3436. PMLR,
  2020.

\bibitem[Zinkevich(2003)]{zinkevich2003online}
Martin Zinkevich.
\newblock Online convex programming and generalized infinitesimal gradient
  ascent.
\newblock In \emph{Proceedings of the 20th International Conference on Machine
  Learning (ICML-03)}, pages 928--936, 2003.

\end{thebibliography}

\newpage 

\tableofcontents

\newpage

\appendix

\section{Notations and Organization}
\label{app:notations-and-organization}

\subsection{Organization}

\subsubsection{Appendix~\ref{app:experiments}: Experiments}
Appendix~\ref{app:experiments} provides brief emprical results in a standard control problem with a few classic perturbation patterns, and compares to classical LQR control and the more advanced control of \cite{gradu2020non}.

\subsubsection{Appendix~\ref{app:bco-quadratic}: Proof of \texttt{EBCO-M} regret guarantee}

Appendix~\ref{app:bco-quadratic} proves regret guarantee (Corollary~\ref{cor:regret-expected-convex}) for \texttt{EBCO-M} (Algorithm~\ref{alg:BCO-quadratic}) under Assumption~\ref{assumption:contraint-set-K}, a relaxed \ref{assumption:loss-functions}, and \ref{assumption:adversary}:

\begin{itemize}
\item Section~\ref{sec:self-concordant-barriers}: properties of self-concordant barriers used in the proof
\item Section~\ref{sec:gradient-analysis}: conditional bias guarantee for the proposed gradient estimator in Algorithm~\ref{alg:BCO-quadratic}
\item Section~\ref{sec:bco-m-regret-analysis}: regret analysis for Algorithm~\ref{alg:BCO-quadratic}
\end{itemize}

\subsubsection{Appendix~\ref{app:bco-control-knowns}: Proof of \texttt{EBPC} Regret Guarantee for Known Systems}
Appendix~\ref{app:bco-control-knowns} proves \texttt{EBPC} regret guarantee for known systems as stated in Theorem~\ref{thm:control-regret-known}: 
\begin{itemize}
\item Section~\ref{sec:with-history-construction}: construction of with-history loss functions based on cost functions
\item Section~\ref{sec:control-loss-regularity}: establishes the following regularity conditions for with-history loss functions
\begin{itemize}
\item Construction of with-history functions and unary forms: Definition~\ref{def:with-history-construction}. 
\item Norm bound on $\y_t,\uv_t$: Lemma~\ref{lem:yu-bound-known}
\item Diameter bound $B$ of $c_t$ and $F_t$: Lemma~\ref{lem:diam-bound-known}
\item Diameter bound $D$ of $\M(H,R)$: Lemma~\ref{lem:diam-bound-known}
\item Lipschitz bound $L_F$ of $F_t$: Lemma~\ref{lem:with-history-regularity-known}
\item Conditional strong convexity parameter $\sigma_{f}$ of $f_t$: Lemma~\ref{lem:with-history-regularity-known}
\item Smoothness parameter $\beta_F$ of $F_t$: Lemma~\ref{lem:with-history-regularity-known}
\end{itemize}
\item Section~\ref{sec:known-control-regret-decomposition}: \texttt{EBPC} regret analysis for known systems 
\end{itemize}

\subsubsection{Appendix~\ref{sec:unknown-system-control-regret}: Proof of \texttt{EBPC} Regret Guarantee for Unknown Systems}
Appendix~\ref{sec:unknown-system-control-regret} proves \texttt{EBPC} regret guarantee for unknown systems as stated in Theorem~\ref{thm:control-regret-unknown}:
\begin{itemize}
\item Section~\ref{sec:sys-est-error}: system estimation error guarantee
\item Section~\ref{sec:unknown-construction-with-history}: construction of with-history loss functions and pseudo loss functions
\item Section~\ref{sec:rftl-d-error}: regret guarantee for Regularized Follow-the-Leader with Delay (RFTL-D) with erroneous gradients
\item Section~\ref{sec:pseudo-loss-regularity}: regularity conditions for pseudo-loss and with-history loss functions:
\begin{itemize}
\item Construction of with-history functions, pseudo loss functions, and unary forms: Definition~\ref{def:with-history-loss-unknown}, \ref{def:pseudo-loss}
\item Norm bound on $\hat{\y}_{t}^{\mathbf{nat}},\y_t,\uv_t$: Lemma~\ref{lem:yyu-norm-bound-unknown}
\item Diameter bound $B$ of $c_t$ and $\hat{F}_t$, $\pF_t$: Lemma~\ref{lem:diameter-bounds-unknown}
\item Diameter bound $D$ of $\M(H^+,R^+)$: Lemma~\ref{lem:diameter-bounds-unknown}
\item Lipschitz bound $L_{\pF}$ and $L_{\hat{F}}$ for $\pF_t$ and $\hat{F}_{t}$: Lemma~\ref{lem:regularity-unknown}
\item Smoothness parameter $\beta_{\pF}$ and $\beta_{\hat{F}}$ for $\pF_t$ and $\hat{F}_{t}$: Lemma~\ref{lem:regularity-unknown}
\item Conditional strong convexity parameter $\sigma_{\pf}$ and $\sigma_{\hat{f}}$ for $\pf_t$ and $\hat{f}_{t}$: Lemma~\ref{lem:regularity-unknown}
\end{itemize}
\item Section~\ref{sec:unknown-system-regret}: \texttt{EBPC} regret analysis for unknown systems
\end{itemize}

\subsection{Complete List of Notations}
\begin{itemize}
\item \textbf{Asymptotic equivalence.} We use $\lesssim,\gtrsim,\asymp$, or equivalently, $O(\cdot), \Omega(\cdot), \Theta(\cdot)$, to denote asymptotic inequalities and equivalence. In particular, $a\lesssim b$ ($a=O(b)$), $a\gtrsim b$ if $\exists$ universal constant $c$ such that $a\le cb$, $a\ge cb$, respectively. $a\asymp b$ if $a\lesssim b$ and $a\gtrsim b$.  

\item \textbf{Derivative.} For $f:\R^m\rightarrow\R^n$, we use $\mathbf{D}f\in\R^{m\times n}$ to denote its derivative.

\item \textbf{Spectral radius.} For $A\in\R^{n\times n}$, $\rho(A)$ measures $A$'s spectral radius, or maximum of the absolute values of $A$'s eigenvalues. 

\item \textbf{Norms.}

\begin{tabular}{llll}
    Notation & Meaning  & Domain$^{*}$ & Definition\\
    $\|\cdot\|_p$ & $\ell_p$-norm  & $\R^n$ & $v\mapsto(\sum_{i=1}^n v_i^p)^{\frac{1}{p}}$\\
    $\|\cdot\|_F$ & Frobenius norm & $\R^{m\times n\times r}$ & $M\mapsto (\sum_{i=1}^m \sum_{j=1}^n \sum_{k=1}^r M_{ijk}^2)^{\frac{1}{2}}$\\
    $\|\cdot\|_{\mathrm{op}}$ & operator norm & $\R^{m\times n}$ & $M\mapsto \sup_{v\in\R^n,\|v\|_2=1}\|Mv\|_2$\\
    $\|\cdot\|_{M}$, \ $M\in\R^{n\times n}$ & local norm induced by $M$ & $\R^n$ & $v\mapsto (v^{\top}Mv)^{\frac{1}{2}}$\\
    $\|\cdot\|_{\ell_1,\mathrm{op}}$ & $\ell_1$-operator norm & $(\R^{m\times n})^{\mathbb{N}}$ & $(M_i)_{i\in I\subseteq \mathbb{N}}\mapsto \sum_{i\in I}\|M_i\|_{\mathrm{op}}$\\
    $\|\cdot\|^*$ & dual norm of $\|\cdot\|$ & same as $\|\cdot\|$ & $v\mapsto\sup\{\langle u,v\rangle:\|u\|\le 1\}$\\
    $\|\cdot\|_t$, $\|\cdot\|_{t,t+1}$ & local norm at time $t$ & $\R^n$ & see Definition~\ref{def:local-norms} 
\end{tabular}
$^*:$ $m,n,r$ are arbitrary dimensions that may be specifically defined throughout the paper.

\item \textbf{System, dynamics, and parameters.}

\begin{tabular}{ll}
 $d_{\x}, d_{\uv}, d_{\y}$ & dimension of states, controls, observations \\
 $A,B,C$ & system matrices for linear dynamical system \\
 $G$ & Markov operator for linear dynamical system \\
 $\hat{G}$ & estimated Markov operator \\
 $\x_t \in \reals^{d_\x}$  & state at time $t$  \\
 $\uv_t \in \reals^{d_\uv} $  & control at time $t$ \\
 $\w_t \in \reals^{d_\x} $  & system perturbation (disturbance) at time $t$ \\
 $\e_t \in \reals^{d_\y}$ & state-observation projection noise at time $t$ \\
 $\y_t \in \reals^{d_\y} $  & observation at time $t$ \\
 $\ynat_t \in \R^{d_\y}$ & nature's $\y$, the would-be observation at time $t$ assuming no controls are ever played \\
 $\hat{\y}_{t}^{\mathbf{nat}} \in \R^{d_{\y}}$ & algorithm calculated nature's $\y$ using the estimated Markov operator $\hat{G}$ \\
 $H, \bar{H},H^+, \overline{H^+}$ & history length of a policy class, $\bar{H}=H-1$, $H^+=3H$, $\overline{H^+}=H^+-1$ \\
 $R, R^+$ & DRC policy class $\ell_1$-operator norm bound, $R^+=2R$ \\
 $R_{\mathrm{nat}}$ & nature's $\y$ $\ell_2$-norm bound \\
 $R_G$ & $\ell_1$-operator norm bound on $G$ \\
 $\M(H,R)$ & DRC policy class with length $H$ and $\ell_1$-operator norm bound $R$
\end{tabular}

\item \textbf{Cost and loss functions.}

\begin{tabular}{lll}
    Notation & Meaning & Domain \\ $c_t(\cdot,\cdot)$ & cost function for controlling linear dynamical system & $\R^{d_{\y}}\times \R^{d_{\uv}}$ \\
    $F_t(\cdot), \hat{F}_{t}(\cdot)$ & with-history loss function with history length $H$ & $\K^{H}$ for convex Euclidean set $\K$ \\
    $f_t(\cdot), \hat{f}_{t}(\cdot)$ & unary form induced by $f_t(x)=F_t(x,\dots,x)$ & some convex Euclidean set $\K$ \\
    $\pF_t(\cdot),\pf_t(\cdot)$ &  pseudo-loss and induced unary form & $\K^H,\K$ for convex Euclidean set $\K$\\
    $B$ & bound on function diameter\\
    $D$ & bound on constraint set diameter \\
    $L_c, L_F, L_{\pF}, L_{\hat{F}}$ & Lipschitz bound on function $c_t, F_t, \pF_t, \hat{F}_{t}$ \\
    $\beta_{c}, \beta_{F}, \beta_{\pF}, \beta_{\hat{F}}$ & smoothness parameter of $c_t, F_t, \pF_t, \hat{F}_{t}$ \\
    $\sigma_c$ & strong convexity parameter of $c_t$ \\
    $\sigma_f, \sigma_{\pf}, \sigma_{\hat{f}}$ & conditional strong convexity parameter of $f_t, \pf_t, \hat{f}_{t}$
\end{tabular}

\end{itemize}

\newpage
\section{Experiments}
\label{app:experiments}

To compare our controller against previous work, we test our control scheme empirically in the same settings as \cite{gradu2020non}. Our experiments use the package Deluca developed by \cite{gradu2021deluca}. We test control of a barely-stable LDS -- a damped double-integrator system given by 
\[A=\begin{bmatrix}
    .9&.9\\
    -0.01&.9\\
\end{bmatrix}, B=\begin{bmatrix}
    0\\
    1\\
\end{bmatrix}\]

We attempt control under several different classes of noise. Relevant details are below:
\begin{itemize}
    \item As the controller of \cite{gradu2020non} does not support partial observation, we test in the full-observation case.
    \item Both controllers are given access to the optimal LQR controller $K$ (that is, we run Algorithm~\ref{alg:BCO-control} as opposed to Algorithm~\ref{alg:est} for simplicity of comparison).
    \item State is initialized randomly, and perturbations are stochastic (to facilitate direct comparison with the experiments of Gradu et al., who did the same).
    \item We test both algorithms with $H=5$, which was found to produce nearly-optimal results for both algorithms (theoretical performance is increasing in $H$, but converges with exponential falloff to a supremum).
    \item Noise magnitude is chosen arbitrarily across experiments. However, as the results are linear in magnitude (since both the systems and the control algorithms are linear), direct comparison to the experimental results of \cite{gradu2020non} is possible via scaling.
\end{itemize}

We also make two important nonstandard modifications to the experimental setup. Following the example of \cite{gradu2020non}, we searched to find optimal multipliers for learning rate. This was found in their work to substantially enhance the performance of nonstochastic control algorithms against stochastic inputs in practice (due to the fact that stochastic inputs are unlikely to cause systematic learning errors early in the control run) and appears to be present in their experiments. We also test \cite{gradu2020non} under a version of their implementation modified with controller-magnitude bounding to ameliorate divergence issues (still visible in some spiking). We have not been able to determine the source thereof, and we do not have access to the code used to generate the plots visible in \cite{gradu2020non}, so we are unable to determine the source of these spikes. However, this modification strictly improves their performance on the benchmarks, thus maintaining fair comparison.

Moving-average losses are graphed for EBPC, BPC, and LQR for the above problem with the three perturbation types of \cite{gradu2020non}: Gaussian, $c\sin (rx) \begin{bmatrix}
    1\\
    1\\
\end{bmatrix}$ (with period 40), and Gaussian Random Walk. $H=5$ was used for both memory algorithms.

\begin{figure}[h] 
    \minipage{0.5\textwidth}
        \includegraphics[width=\textwidth,height=\textheight,keepaspectratio]{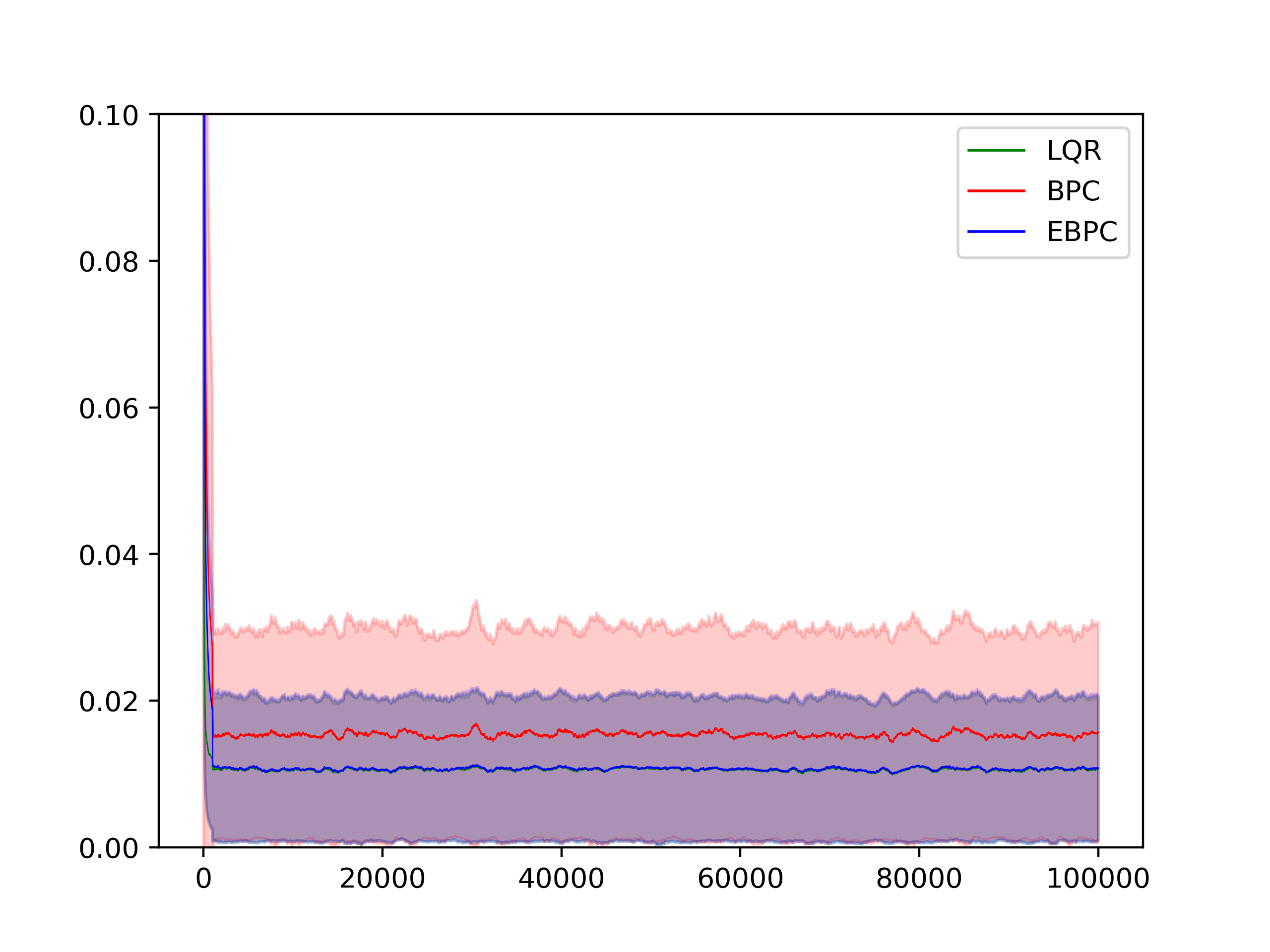}
    \endminipage\hfill
    \hfill
    \minipage{0.5\textwidth}
        \includegraphics[width=\textwidth,height=\textheight,keepaspectratio]{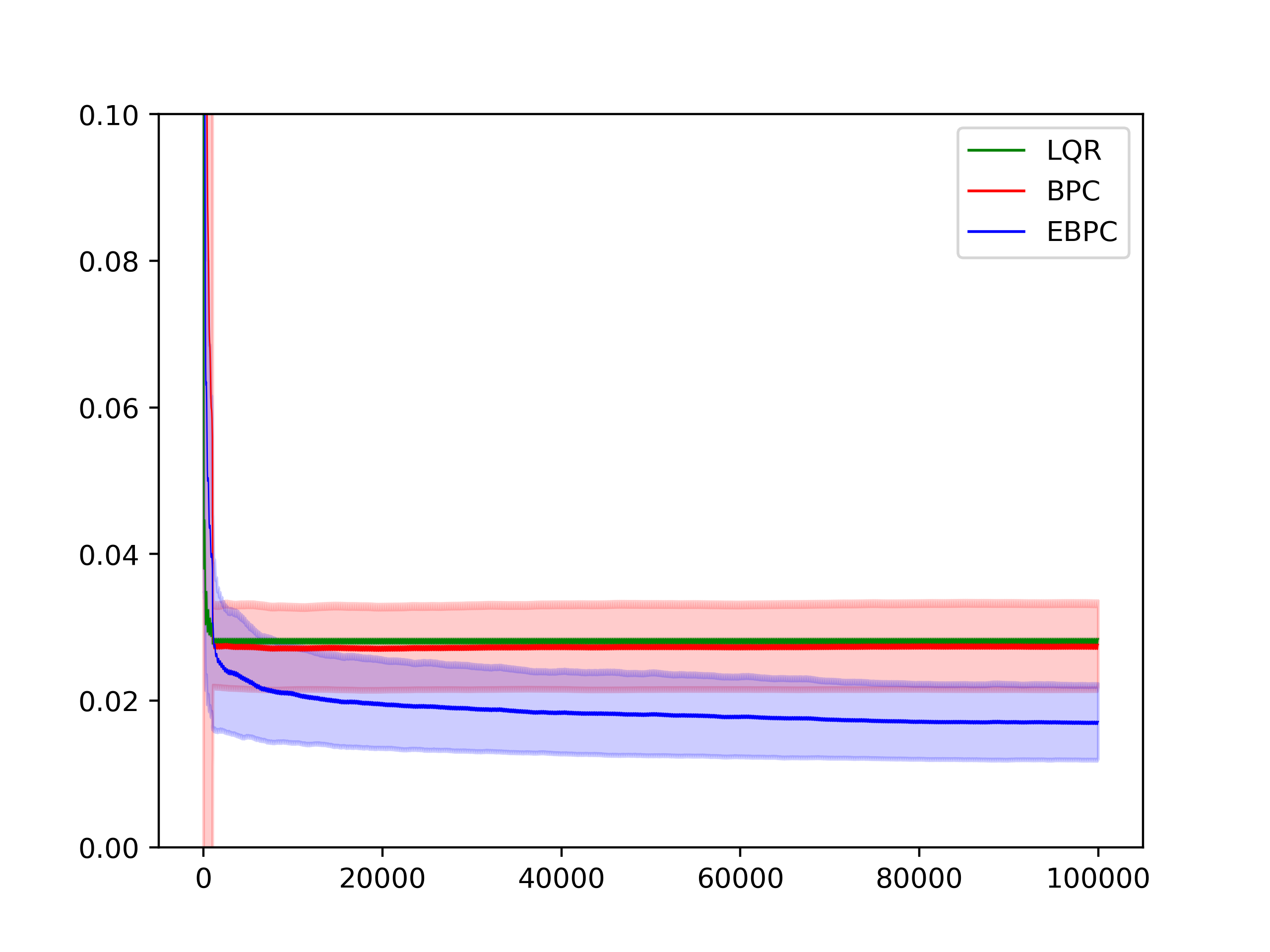}
    \endminipage\hfill
    \centering
    \minipage{0.5\textwidth}
        \includegraphics[width=\textwidth,height=\textheight,keepaspectratio]{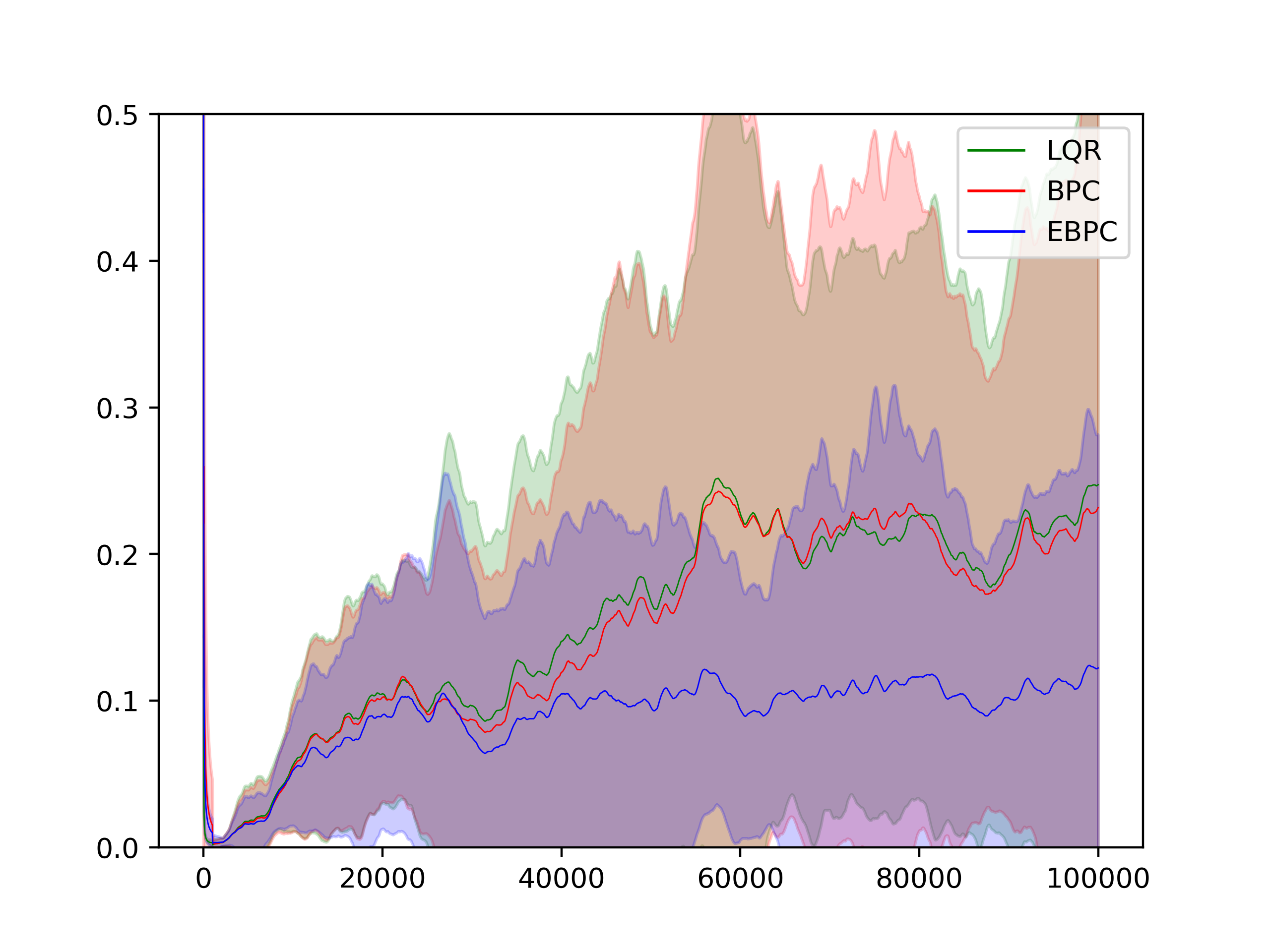}
    \endminipage
    \caption{Loss ($y$-axis) of the three tested algorithms on Gaussian (left), sinusoidal (right), and Gaussian-walk (bottom) perturbation over time($x$-axis). Error bars indicate standard deviation across twelve draws of perturbation and controller randomness.}
    \label{materialflowChart}
\end{figure}

We observe that while our method has higher initial error, it has long-term error substantially lower than that of competing methods in aggregate (except in the sanity-check case of Gaussian noise, where it quickly converges to the LQR error as desired). Critically, it is able to adapt effectively to trends in perturbations more effectively than previous higher-error-rate algorithms, allowing for constant or decreasing error in environments with constant-size or increasing perturbations.

\newpage
\section{Proof of \texttt{EBCO-M} Regret Guarantee}
\label{app:bco-quadratic}
We prove the more general claim of Corollary~\ref{cor:regret-expected-convex}, where the function $F_t$ is assumed to be conditionally $\sigma$-strongly convex. Denote $\bar{F}_{t}(x_{t-\bar{H}:t})=\E[F_t(x_{t-\bar{H}:t})\mid u_{1:t-H},f_{H:t-H}]$. 

Note that in Algorithm~\ref{alg:BCO-quadratic}, with the delayed updates and the initialization $g_1=\dots=g_{\bar{H}}=0$, we have $x_1=\dots=x_{2\bar{H}+1}=\argmin_{x\in\K}R(x)$ and so learning begins only at the $2\bar{H}+1$-th iteration. We can therefore decompose the regret against any $x\in\K$ as
\begin{align*}
\regret_T(x)=\underbrace{\left(\sum_{t=H}^{2\bar{H}} F_t(y_{t-\bar{H}:t}) - f_t(x)\right)}_{(\text{burn-in loss})} + \underbrace{\left(\sum_{t=2\bar{H}+1}^{T} F_t(y_{t-\bar{H}:t})-f_t(x)\right)}_{(\text{effective regret})},
\end{align*}
with burn-in loss crudely bounded by $HB$. We thus turn our attention in bounding the effective regret term. 

The proof of the effective regret bound for Algorithm~\ref{alg:BCO-quadratic} consists of two main parts. In Section~\ref{sec:gradient-analysis}, we show that the proposed gradient estimator $g_t$ has a bounded conditional bias. In Section~\ref{sec:bco-m-regret-analysis}, we perform the analysis of a variant of the Regularized Follow-the-Leader (RFTL) algorithm, adding both a history component and a delayed update. Then, we show that together with the bounded conditional bias of our proposed gradient estimator, this yields an optimal regret bound for the bandit online convex optimization with memory algorithm outlined in Algorithm~\ref{alg:BCO-quadratic}.

\subsection{Self-concordant barriers}
\label{sec:self-concordant-barriers}
The use of self-concordant barriers for bandit optimization is due to \cite{abernethy2008competing}, where the following properties are stated and used.

\begin{proposition}  $\nu$-self-concordant barriers over $\K$ satisfy the following properties:
\label{prop:self-concordant-barrier}
\begin{enumerate}
\item Sum of two self-concordant functions is self-concordant. Linear and quadratic functions are self-concordant.
\item If $x,y\in\K$ satisfies $\|x-y\|_{\nabla^2 R(x)}<1$, then the following inequality holds:
\begin{align*}
(1-\|x-y\|_{\nabla^2 R(x)})^2 \nabla^2 R(x)\preceq \nabla^2 R(y)\preceq \frac{1}{(1-\|x-y\|_{\nabla^2 R(x)})^2}\nabla^2 R(x). 
\end{align*}
\item The Dikin ellipsoid centered at any point in the interior of $\K$ w.r.t. a self-concordant barrier $R(\cdot)$ over $\K$ is completely contained in $\K$. Namely,
\begin{align*}
\{y\in\R^n\mid \|y-x\|_{\nabla^2 R(x)}\le 1\}\subset \K, \ \ \forall x\in \text{int}(\K). 
\end{align*}
where 
\begin{align*}
    \|v\|_{\nabla^2 R(x)}\defeq\sqrt{v^\top \nabla^2 R(x) v}
\end{align*}
\item $\forall x,y\in\text{int}(\K)$:
\begin{align*}
R(y)-R(x)\le \nu\log\frac{1}{1-\pi_x(y)},
\end{align*}
where $\pi_x(y)\defeq\inf\{t\ge0:x+t^{-1}(y-x)\in\K\}$.
\end{enumerate}
\end{proposition}

\subsection{Gradient estimator}
\label{sec:gradient-analysis}
The goal of this section is to establish a bound on the conditional bias of the proposed gradient estimator $g_t$, formally given by the following proposition:
\begin{proposition}
\label{prop:gradient-est-main-prop}
The gradient estimator $g_t=nF_t(y_{t-\bar{H}:t})\sum_{i=0}^{\bar{H}}A_{t-i}^{-1}u_{t-i}$ satisfies the following conditional bias bound in $\ell_2$: $\forall t\ge 2\bar{H}+1$,
\begin{align*}
\left\|\E[g_t\mid u_{1:t-H}, F_{H:t-H}]-\nabla \bar{f}_{t}(x_{t})\right\|_2\le \frac{16\sqrt{\eta}\beta nBH^3}{\sqrt{\sigma(t-2\bar{H})}}.
\end{align*}
\end{proposition}

\begin{lemma}
\label{lem:unbiased-gradient}
The gradient estimator $g_t$ is a conditionally unbiased estimator of the sum of the $H$ coordinate gradients of $F_t:\K^H\rightarrow\R$, i.e. $\forall t>H$, 
\begin{align*}
\E[g_t\mid u_{1:t-H},F_{H:t-H}]=\sum_{i=0}^{\bar{H}}\nabla_i \bar{F}_t(x_{t-\bar{H}:t}), 
\end{align*}
where $\nabla_i \bar{F}_t(z_1,\dots,z_H)=\frac{\partial}{\partial z_i}\bar{F}_t(z_1,\dots,z_H)$. 
\end{lemma}
\begin{proof}
Let $q(x)=\frac{1}{2}x^{\top} Ax+b^{\top}x+c$ be a (possibly random) quadratic function from $\R^n\rightarrow \R$ and $C$ be a (possibly random) symmetric, invertible matrix. Let $x_0\in\R^n$ be a (possibly random) point of evaluation. Let $\mathcal{F}$ be a filtration such that $\{A,B,C,c,x_0\}\in\mathcal{F}$. Let $u\in\R^n$ be a random vector that is drawn from a symmetric distribution such that $\E[uu^{\top}]=\frac{r}{n}I_{n\times n}$ for some $r>0$, and $u$ is independent of $\mathcal{F}$. Then, 
\begin{align*}
\E[C^{-1}uq(x_0+Cu)\mid \mathcal{F}]&=\frac{1}{2} C^{-1} \E[u(x_0+Cu)^{\top}A(x_0+Cu)\mid \mathcal{F}]+C^{-1}\E[ub^{\top}(x_0+Cu)\mid \mathcal{F}]\\
&=\frac{1}{2}C^{-1} \E[uu^{\top}]C(A+A^{\top})x_0+C^{-1}\E[uu^{\top}]Cb\\
&=C^{-1} \E[uu^\top]C\left(\frac{1}{2}(A+A^{\top})x_0+b\right)\\
&=\frac{r}{n}\left(\frac{1}{2}(A+A^{\top})x_0+b\right)\\
&=\frac{r}{n}\nabla q(x_0). 
\end{align*}
Note that in Algorithm \ref{alg:BCO-quadratic}, $u_t$'s are sampled uniformly at random from the unit sphere in $\mathbb{R}^n$, so the distribution is symmetric and $\E[u_tu_t^{\top}]=\frac{1}{n}I_{n\times n}$, and thus $\E[u_{t-\bar{H}:t}u_{t-\bar{H}:t}^{\top}]=\frac{1}{n}I_{nH\times nH}$. Moreover, $\bar{F}_t, x_{t-\bar{H}:t}, A_{t-\bar{H}:t}\in\mathcal{F}_{t-H}$ and $u_{t-\bar{H}:t}$ are independent of $\mathcal{F}_{t-H}$.
Let $\tilde{A}_t \defeq\text{diag}(A_{t-\bar{H}},\dots,A_t)\in\mathbb{R}^{nH\times nH}$ (i.e. the block matrix with diagonal blocks equal to $A_{t-\bar{H}},\dots,A_t$). Then we have
\begin{align*}
\E[n\bar{F}_t(y_{t-\bar{H}:t})\tilde{A}_t^{-1}u_{t-\bar{H}:t}\mid \mathcal{F}_{t-H}]=\nabla \bar{F}_t(x_{t-\bar{H}:t}).
\end{align*}
Consider $\bar{g}_t=n\bar{f}_{t}(y_{t-\bar{H}:t})\sum_{i=0}^{\bar{H}}A_{t-i}^{-1}u_{t-i}$. Note that $\tilde{A}_t^{-1}=\text{diag}(A_{t-\bar{H}}^{-1},\dots,A_t^{-1})$ and by definition of $\bar{g}_t$, we have
\begin{align*}
\E[\bar{g}_t\mid \mathcal{F}_{t-H}]=\sum_{i=0}^{\bar{H}}\nabla_i \bar{F}_t(x_{t-\bar{H}:t}).
\end{align*}
On the other hand, $x_t$ and $A_t$ are completely determined by $\{u_{1:t-H}\}\cup\{f_{H:t-H}\}$, and thus $\sum_{i=0}^{\bar{H}}A_{t-i}^{-1}u_{t-i}, y_{t-\bar{H}:t}$ is determined by $\{u_{1:t}\}\cup\{f_{H:t-H}\}$. Therefore, 
\begin{align*}
\E[g_t\mid u_{1:t-H},F_{H:t-H}]&=\E\left[nF_t(y_{t-\bar{H}:t})\sum_{i=0}^{\bar{H}}A_{t-i}^{-1}u_{t-i} \ \bigg | \ u_{1:t-H},F_{H:t-H}\right]\\
&=\E\left[\E\left[nF_t(y_{t-\bar{H}:t})\sum_{i=0}^{\bar{H}}A_{t-i}^{-1}u_{t-i} \ \bigg | \ u_{1:t},F_{H:t-H}\right] \ \bigg | \ u_{1:t-H},F_{H:t-H}\right]\\
&=\E\left[n\E[F_t(y_{t-\bar{H}:t})\mid u_{1:t},F_{H:t-H}]\sum_{i=0}^{\bar{H}}A_{t-i}^{-1}u_{t-i} \ \bigg | \ u_{1:t-H},F_{H:t-H}\right]\\
&=\E\left[n\bar{F}_t(y_{t-\bar{H}:t})\sum_{i=0}^{\bar{H}}A_{t-i}^{-1}u_{t-i} \ \bigg | \ u_{1:t-H},F_{H:t-H}\right]\\
&=\E[\bar{g}_t\mid u_{1:t-H}].
\end{align*}
We conclude that
\begin{align*}
\E[g_t\mid u_{1:t-H},F_{H:t-H}]=\sum_{i=0}^{\bar{H}}\nabla_i \bar{F}_t(x_{t-\bar{H}:t}). 
\end{align*}
\end{proof}

\begin{definition}[Local norms]
\label{def:local-norms}
Denote the pair of dual norms $\|\cdot\|_t, \|\cdot\|_t^*$ on $\K$ as
\begin{align*}
\|y\|_t&\defeq \|y\|_{A_t^{-2}}=\sqrt{y^{\top}(\nabla^2 R(x_t)+\eta\sigma tI)y}=\sqrt{y^{\top}A_t^{-2}y},\\
\|y\|_t^*&\defeq \|y\|_{A_t^2}=\sqrt{y^{\top}(\nabla^2 R(x_t)+\eta\sigma tI)^{-1}y}=\sqrt{y^{\top}A_t^{2}y}.
\end{align*}
By Taylor expansion, $\forall R$ and $x,y\in\mathrm{dom}(R)$, $\exists z=tx+(1-t)y$ for some $t=t(x,y,R)\in[0,1]$ such that $D_R(x,y):=R(x)-R(y)-R(y)^{\top}(x-y)=\frac{1}{2}\|x-y\|_{\nabla^2 R(z)}^2$. We call $\|\cdot\|_{\nabla^2 R(z)}$ the induced norm by the Bregman divergence w.r.t. $R$ between $x$ and $y$. Denote as $\|\cdot\|_{t,t+1}$ the induced norm by the Bregman divergence w.r.t. $R_t(x)\defeq R(x)+\frac{\eta\sigma}{2}\sum_{s=H}^t \|x-x_{s-\bar{H}}\|_2^2$ between $x_t$ and $x_{t+1}$. Denote its dual norm as $\|\cdot\|_{t,t+1}^{*}$. 
\end{definition}

\begin{lemma} 
\label{lem:history14}
$\forall t\ge H$, assuming $2\eta\|g_{t-\bar{H}}\|_t^*\le 1$, then $\|x_t-x_{t+1}\|_t\le 2\eta\|g_{t-\bar{H}}\|_t^*$.
\end{lemma}
\begin{proof}
From Lemma 14 in~\cite{hazan2014bandit}, $\|x-\argmin_{x}h(x)\|_{\nabla^2 h(x)}\le 2\|\nabla h(x)\|_{\nabla^2 h(x)}^*$, provided $h$ is self-concordant and $\|\nabla h(x)\|_{\nabla^2 h(x)}^*\le 1$. 
Define $\Phi_t(x)\defeq \eta\sum_{s=H}^t g_{s-\bar{H}}^\top x+R_t(x)$, where $R_t(x)= R(x)+\frac{\eta\sigma}{2}\sum_{s=H}^t \|x-x_{s-\bar{H}}\|_2^2$. $\Phi_t(\cdot)$ is self-concordant since it is the sum of a self-concordant function and sum of quadratic functions. Note that $x_{t+1}=\argmin\Phi_t(x)$ by specification of Algorithm \ref{alg:BCO-quadratic} and $\nabla^2\Phi_t=\nabla^2 R_t$. Moreover, $\Phi_t(x)=\Phi_{t-1}(x)+\eta g_{t-\bar{H}}^{\top}x+\frac{\eta\sigma}{2}\|x-x_{t-\bar{H}}\|_2^2$. Since $x_t\in\mathrm{int}(\K)$ and minimizes $\Phi_{t-1}$, $\nabla \Phi_{t}(x)=\eta g_{t-\bar{H}}$. Applying Lemma 14 from \cite{hazan2014bandit}, $\|x_t-x_{t+1}\|_t\le 2\eta\|g_{t-\bar{H}}\|_t^*$.
\end{proof}

\begin{lemma}
    \label{lem:grad-est-norm-bound}
    If $\eta\le \frac{1}{8nH\log H B\sqrt{T}}$, and assume that $H=\mathrm{poly}(\log T)$ then the following inequalities hold deterministically $\forall t\ge H$: $\forall ((t-\bar{H})\vee H)\le s\le t$,
    \[\|g_{s-\bar{H}}\|_t^*\le 2nBH\log H, \ \ \ \|g_{s-\bar{H}}\|_{t,t+1}^*\le 4nBH\log H. \]
\end{lemma}
\begin{proof}
    We will show the joint hypothesis that: (1) $\left(\left(1-\frac{1}{\sqrt{T}}\right)\wedge \frac{t}{t+1}\right)A_t\preceq A_{t+1}\preceq \frac{A_t}{\left(1-\frac{1}{\sqrt{T}}\right)\wedge \frac{t}{t+1}} $; (2) $\|g_{s-\bar{H}}\|_t^*\le 2nBH\log H$, $\forall t-\bar{H}\le s\le t$; (3) $\|g_{s-\bar{H}}\|_{t,t+1}^*\le 4nBH\log H$, $\forall t-\bar{H}\le s\le t$, 
    for all $t$ by simultaneous induction on $t$. We divide our induction into two steps:
    \begin{itemize}
    \item  \textbf{(1), (2), (3) hold for $t=H,\dots,2\bar{H}$:} note that $x_1=\dots=x_H=x=\underset{z\in\K}{\argmin} \ {R(z)}$ and $g_1=\dots=g_{\bar{H}}=0$, thus $x_{H+1}=\dots=x_{2\bar{H}+1}=x$.
    Thus $\forall t=H,\dots,2\bar{H}$, $A_{t+1}\preceq A_t$ holds trivially, to see the bound in the other direction, note that
    \begin{align*}
    A_t= \sqrt{\frac{t+1}{t}}\left(\frac{t+1}{t}\nabla^2 R(x)+\eta\sigma(t+1)I\right)^{-\frac{1}{2}}\preceq \sqrt{\frac{t+1}{t}} A_{t+1} . 
    \end{align*}
    \ignore{
    Thus,
    \begin{align*}
    \sqrt{\frac{t}{t+1}}A_t\preceq A_{t+1}\preceq A_t\preceq\sqrt{\frac{t+1}{t}}A_t. 
    \end{align*}
    }
    $g_t=0$ for $t=1,\dots, \bar{H}$, so (2), (3) follow.

    \item \textbf{Given that (1), (2), (3) hold for all $t<T_0$, show that (1), (2), (3) hold for $t=T_0$:}
    We first prove (2) for $s=t$. The bound holds identically up to constant factor $\le 2$ for $s\in[t-\bar{H},t)$ by induction hypothesis of $A_{t-\bar{H}:t}$. Assume $T_0>2\bar{H}$. Observe that $\left(1-\frac{1}{\sqrt{T}}\right)\wedge \frac{t}{t+1}=\frac{t}{t+1}$ if and only if $t\le \sqrt{T}-1$. On the other hand, since by expression of $g_t=nF_t(y_{t-\bar{H}:t})\sum_{i=0}^{\bar{H}}A_{t-i}^{-1}u_{t-i}$, 
    \begin{align*}
        {\|g_{T_0-\bar{H}}\|_{T_0}^*}^2 = \|g_{T_0-\bar{H}}\|_{A_{T_0}^2}^2
        &\le (nB)^2 \sum_{i,j=0}^{\bar{H}} u_{T_0-\bar{H}-i}^\top A_{T_0-\bar{H}-i}^{-1}A_{T_0}^2A_{T_0-\bar{H}-j}^{-1} u_{T_0-\bar{H}-j}.
    \end{align*}
     Consider the induction hypothesis (1). For $T_0\le\sqrt{T}$, this implies that $\forall i\in[0,\bar{H}]$, there holds $\|A_{T_0-\bar{H}-i}^{-1}A_{T_0}\|_{\mathrm{op}}\le \frac{T_0}{T_0-\bar{H}-i}$, and thus
    \begin{align*}
    {\|g_{T_0-\bar{H}}\|_{T_0}^*}^2\le (nB)^2\sum_{i,j=0}^{\bar{H}} \left(\frac{T_0}{T_0-\bar{H}-i}\right)\left(\frac{T_0}{T_0-\bar{H}-j}\right)=(nB)^2 \left(\sum_{i=\bar{H}}^{2\bar{H}}\frac{T_0}{T_0-i}\right)^2,
    \end{align*}
    which is a decreasing function in $T_0$ and thus attains maximum at $T_0=2\bar{H}+1$, giving that
    \begin{align*}
    {\|g_{T_0-\bar{H}}\|_{T_0}^*}^2\le (nB)^2 \left((2\bar{H}+1)\sum_{i=1}^{H}\frac{1}{i}\right)^2\le 4(nBH)^2(\log(H))^2.  
    \end{align*}

    For $T_0\ge \sqrt{T}+2\bar{H}+1$,  $\|A_{T_0-\bar{H}-i}^{-1}A_{T_0}\|_{\mathrm{op}}\le \left(1-\frac{1}{\sqrt{T}}\right)^{-(\bar{H}+i)}$, so
    \begin{align*}
        {\|g_{T_0-\bar{H}}\|_{T_0}^*}^2&\le (nB)^2\sum_{i,j=0}^{\bar{H}} \left(1-\frac{1}{\sqrt{T}}\right)^{-(2\bar{H}+i+j)}\\
        &=(nB)^2\left(\left(1-\frac{1}{\sqrt{T}}\right)^{-4\bar{H}}\sum_{i=0}^{\bar{H}}\left(1-\frac{1}{\sqrt{T}}\right)^{i}\right)^{2}\\
        &=(nB)^2\left(\left(1-\frac{1}{\sqrt{T}}\right)^{-4\bar{H}}\sqrt{T}\left(1-\left(1-\frac{1}{\sqrt{T}}\right)^{\bar{H}}\right)\right)^2\\
        &\le (nB)^2 \left(\left(1-\frac{4\bar{H}}{\sqrt{T}}\right)^{-1}\sqrt{T}\left(1-\left(1-\frac{\bar{H}}{\sqrt{T}}\right)\right)\right)^2\\
        &\le (nB)^2 \left(\frac{\sqrt{T}}{\sqrt{T}-4\bar{H}}\right)^2 H^2\\
        &\le 4(nBH)^2.
    \end{align*}
    where the second inequality uses the inequality $(1+x)^r\ge 1+rx$ for $x>-1$, integer $r\ge 1$, and the last inequality holds by assumption that $H=\mathrm{poly}(\log T)$.  

    For $T_0\in(\sqrt{T},\sqrt{T}+2\bar{H}+1)$, 
    \begin{align*}
        \|A_{T_0}A_{T_0-\bar{H}-i}^{-1}\|_{\mathrm{op}}&=\underbrace{\|A_{T_0}A_{T_0-1}^{-1}A_{T_0-1}\dots A_{\sqrt{T}}^{-1}\|_{\mathrm{op}}}_{\le \left(1-\frac{1}{\sqrt{T}}\right)^{-(T_0-\sqrt{T})}}\underbrace{\|A_{\sqrt{T}}\dots A_{T_0-\bar{H}-i}^{-1}\|_{\mathrm{op}}}_{\le \frac{\sqrt{T}}{T_0-\bar{H}-i}}.
    \end{align*}
    Thus, letting $\Delta\defeq T_0-\sqrt{T}\in[1,2\bar{H}]$, 
    \begin{align*}
    {\|g_{T_0-\bar{H}}\|_{T_0}^*}^2&\le (nB)^2 \left(1-\frac{1}{\sqrt{T}}\right)^{-\Delta}\left(\sum_{i=0}^{\bar{H}}\frac{\sqrt{T}}{\sqrt{T}+\Delta-\bar{H}-i}\right)^2\\
    &\le (nB)^2 \underbrace{\left(1-\frac{1}{\sqrt{T}}\right)^{-2\bar{H}}}_{\le \left(1-\frac{2\bar{H}}{\sqrt{T}}\right)^{-1}} \underbrace{H^2\left(\frac{\sqrt{T}}{\sqrt{T}-2\bar{H}}\right)^2}_{\le 2H^2 }\\
    &\le 4 (nBH)^2.
    \end{align*}

    \ignore{
    The last inequality is established as follows: consider the function $m:[1,2\bar{H}]\rightarrow \R$ given by $m(\Delta)=\left(1-\frac{1}{\sqrt{T}}\right)^{-\Delta}\frac{\sqrt{T}}{\sqrt{T}+\Delta-2\bar{H}}$. Compute the derivative of $m$, 
    \begin{align*}
        m'(\Delta)&=-\frac{\sqrt{T}\left(1-\frac{1}{\sqrt{T}}\right)^{-\Delta}}{\sqrt{T}+\Delta-2\bar{H}}\left(\frac{1}{\sqrt{T}+\Delta-2\bar{H}}+\log\left(1-\frac{1}{\sqrt{T}}\right)\right)\\
        &\le -\frac{\sqrt{T}\left(1-\frac{1}{\sqrt{T}}\right)^{-\Delta}}{\sqrt{T}+\Delta-2\bar{H}}\underbrace{\left(\frac{1}{\sqrt{T}}+\log\left(1-\frac{1}{\sqrt{T}}\right)\right)}_{<0}<0. 
    \end{align*}
    Thus, $m$ attains maximum at $\Delta=1$, thus yielding the bound on $\|g_{T_0-\bar{H}}\|_{T_0}^*$. 
    }
    
    Then by Lemma \ref{lem:history14} and choice of $\eta$, $\|x_{T_0}-x_{T_0+1}\|_{T_0}\le 2\eta \|g_{T_0-\bar{H}}\|_{T_0}^*\le \frac{1}{\sqrt{T}}$.
    $R_{T_0}(x)$ is self-concordant, and $A_{T_0}^{-1}=\left(\nabla^2 R_{T_0}(x_{T_0})\right)^{\frac{1}{2}}$, so by the local Hessian bound in Proposition~\ref{prop:self-concordant-barrier}, 
    \begin{align*}
       \left(\left(1-\frac{1}{\sqrt{T}}\right)\wedge \frac{t}{t+1}\right)A_{T_0}^{-1}&\preceq(1-\|x_{T_0}-x_{T_0+1}\|_{T_0})A_{T_0}^{-1}\\
       &\preceq A_{T_0+1}^{-1}\\
       &\preceq \frac{A_{T_0}^{-1}}{1-\|x_{T_0}-x_{T_0+1}\|_{T_0}}\\
       &\preceq \frac{1}{\left(1-\frac{1}{\sqrt{T}}\right)\wedge\frac{t}{t+1}}A_{T_0}^{-1},
    \end{align*}
   thus proving (1) for $t=T_0$.

    To prove (4) for $t=T_0$, observe that if $z$ is a convex combination of $x_{T_0}$ and $x_{T_0+1}$, then
    \begin{align*}
    \|z-x_{T_0}\|_{\nabla^2 R_{T_0}(x_{T_0})}\le \|x_{T_0+1}-x_{T_0}\|_{\nabla^2 R_{T_0}(x_{T_0})}\le \frac{1}{\sqrt{T}},
    \end{align*}
    and thus again by Proposition~\ref{prop:self-concordant-barrier},
    \begin{align*}
    (\nabla^2 R(z)+\eta\sigma tI)^{-1}\preceq \left(1-\frac{1}{\sqrt{T}}\right)^{-2}(\nabla^2 R(x_t)+\eta\sigma tI)^{-1},
    \end{align*}
    and thus since $\exists z$ convex combination of $x_{T_0}, x_{T_0+1}$: $\|g_{T_0-\bar{H}}\|_{T_0,T_0+1}^*=\|g_{T_0-\bar{H}}\|_{\nabla^{-2}R_{T_0}(z)}$ and thus $\|g_{T_0-\bar{H}}\|_{T_0,T_0+1}^*\le \left(1-\frac{1}{\sqrt{T}}\right)^{-2}\|g_{T_0-\bar{H}}\|_{t}^*\le 4nBH\log H$. 
    \end{itemize}
\end{proof}

\begin{lemma} [Iterate bound]
\label{lem:iterate-bound}
$\forall t\ge H$, the Euclidean distance between neighboring iterates is bounded by
\begin{align*}
\|x_t-x_{t+1}\|_2\le \frac{4\sqrt{\eta}}{\sqrt{\sigma(t-\bar{H})}}\|g_{t-\bar{H}}\|_{t,t+1}^*\le \frac{16\sqrt{\eta} nBH\log H}{\sqrt{\sigma(t-\bar{H})}}.
\end{align*}
\end{lemma}
\begin{proof}
The second inequality follows from the previous lemma, so we prove the first. Recall $\Phi_t$ as defined in Lemma 14. By Taylor expansion, optimality condition and linearity of $\Phi_t(\cdot)-R_t(\cdot)$,
\begin{align*}
\Phi_t(x_t)=\Phi_t(x_{t+1})+(x_t-x_{t+1})^\top \nabla \Phi_t(x_{t+1})+D_{\Phi_t}(x_t,x_{t+1})\ge \Phi_t(x_{t+1})+D_{\tilde{R}_t}(x_t,x_{t+1}), 
\end{align*}
which by decomposing $\Phi_t$ implies
\begin{align*}
D_{R_t}(x_t,x_{t+1})\le[\Phi_{t-1}(x_t)-\Phi_{t-1}(x_{t+1})]+ \eta g_{t-\bar{H}}^\top(x_t-x_{t+1})\le \eta g_{t-\bar{H}}^\top(x_t-x_{t+1}).
\end{align*}
and thus for some $z=sx_t+(1-s)x_{t+1}$, $s\in[0,1]$, $\|x_t-x_{t+1}\|_{\nabla^2 R_t(z)}^2\le 2 \eta g_{t-\bar{H}}^\top(x_t-x_{t+1})\le 2\eta \|g_{t-\bar{H}}\|_{\nabla^2 R_t(z)}^*\|x_t-x_{t+1}\|_{\nabla^2 R_t(z)}$, thus establishing the bound $\|x_t-x_{t+1}\|_{\nabla^2 R_t(z)}\le 2\eta\|g_{t-\bar{H}}\|_{\nabla^2 R_t(z)}^*$. 
Since $R_t(\cdot)$ is $\eta\sigma (t-\bar{H})$-strongly convex,
\begin{align*}
\|x_{t}-x_{t+1}\|_2&\le \frac{2}{\sqrt{\eta\sigma (t-\bar{H})}} \|x_t-x_{t+1}\|_{\nabla^2R_t(z)}\le \frac{4\sqrt{\eta}}{\sqrt{\sigma(t-\bar{H})}} \|g_{t-\bar{H}}\|_{t,t+1}^*.
\end{align*}
\end{proof}

\begin{corollary}
\label{cor:gradient-est-error}
Define $f_t:\K\rightarrow\R_{+}$ by $f_t(x)\defeq F_t(x,\dots,x)$. We have that $\forall t>2\bar{H}$,
\begin{align*}
\left\|\E[g_t\mid u_{1:t-H},F_{H:t-H}]-\nabla \bar{f}_{t}(x_{t})\right\|_2\le \frac{16\sqrt{\eta}\beta nBH^3}{\sqrt{\sigma(t-2\bar{H})}}.
\end{align*}
\end{corollary}
\begin{proof}
By the earlier bounds,
\begin{align*}
\left\|\E[g_t\mid u_{1:t-H},F_{H:t-H}]-\nabla \bar{f}_{t}(x_t)\right\|_2^2&=\left\|\sum_{i=0}^{\bar{H}}\nabla_i \bar{F}_t(x_{t-\bar{H}:t})-\nabla \bar{f}_t(x_{t})\right\|_2^2 &\text{Lemma \ref{lem:unbiased-gradient}}\\
&=H\|\nabla \bar{F}_t(x_{t-\bar{H}:t})-\nabla \bar{F}_t(x_{t},\cdots,x_{t})\|_2^2 \\
&\le H\beta^2 \|(x_{t-\bar{H}},\cdots,x_t)-(x_{t},\cdots,x_{t})\|_2^2 \\
&\le H\beta^2 \sum_{i=1}^{\bar{H}} \|x_{t}-x_{t-i}\|_2^2\\
&= H\beta^2\sum_{i=1}^{\bar{H}} \sum_{j=1}^i \|x_{t-j+1}-x_{t-j}\|_2^2\\
&\le \frac{256\eta\beta^2 n^2 B^2H^3\log^2 H}{\sigma}\sum_{i=1}^{\bar{H}} \sum_{j=1}^i \frac{1}{t-j-\bar{H}} &\text{Lemma~\ref{lem:iterate-bound}}\\
&\le \frac{256\eta\beta^2 n^2 B^2H^3\log^2 H}{\sigma}\frac{H^2}{2}\frac{1}{t-2\bar{H}}\\
&=\frac{128\eta\beta^2n^2B^2H^5\log^2H }{\sigma(t-2\bar{H})}.
\end{align*}
Then taking the square root of each side yields the desired bound.
\end{proof}

\subsection{Regret analysis}
\label{sec:bco-m-regret-analysis}

The previous section established a conditional bias bound on the gradient estimator $g_t$ used in Algorithm~\ref{alg:BCO-quadratic}. In this section, we use this conditional bias bound together with an analysis on the subroutine algorithm, Regularized Follow-the-Leader with Delay (RFTL-D), to establish a regret guarantee for Algorithm~\ref{alg:BCO-quadratic}. 

\paragraph{Decomposition of effective regret.} Letting $w=\underset{x\in \K}{\argmin} \ \sum_{t=H}^Tf_t(x)$, we divide the expected regret into three parts, which we will bound separately:
\begin{align*}
 \text{Effective-Regret}_T&=\underbrace{\E\left[\sum_{t=2\bar{H}+1}^T F_t(y_{t-\bar{H}:t})-F_t(x_{t-\bar{H}:t})\right]}_{(1:\text{ estimator movement cost})}+\underbrace{\E\left[\sum_{t=2\bar{H}+1}^T F_t(x_{t-\bar{H}:t})-f_t(x_{t})\right]}_{(2:\text{ history movement cost})}\\
 & \ \ \ \ \ +\underbrace{\E\left[\sum_{t=2\bar{H}+1}^T f_t(x_{t})-f_t(w)\right]}_{(3:\text{ RFTL-D effective regret})}.
\end{align*}
To bound the estimator movement cost, note that $\|A_t^2\|_{\mathrm{op}}=\|(\underbrace{\nabla^2 R(x_t)}_{\succeq 0}+\eta\sigma tI)^{-1}\|_{\mathrm{op}}\le \frac{1}{\eta\sigma t}$, and thus
\begin{align*}
(1)&\le \sum_{t=2\bar{H}+1}^T\E\left[\E \left[\nabla F_t(x_{t-\bar{H}:t})^T(A_{t-\bar{H}:t}u_{t-\bar{H}:t})+\frac{\beta}{2}\|(A_{t-\bar{H}:t}u_{t-\bar{H}:t})\|_2^2 \ \bigg | \ \mathcal{F}_{t-H}\right]\right]\\
&=\frac{\beta}{2} \sum_{t=2\bar{H}+1}^T\E\left[\sum_{s=t-\bar{H}}^t \|A_su_s\|_2^2\right]\le  \frac{\beta}{2}\sum_{t=2\bar{H}+1}^T\E\left[\sum_{s=t-\bar{H}}^t \|A _s^2\|_{\mathrm{op}}\right]\le \frac{\beta}{2\eta\sigma} \sum_{t=2\bar{H}+1}^T\sum_{s=t-\bar{H}}^t \frac{1}{s} \\
&\le \frac{\beta H\log T}{2\eta\sigma}.
\end{align*}
To bound the history movement cost, note that by the iterate bound obtained in the analysis of Corollary~\ref{cor:gradient-est-error}, 
\begin{align*}
(2)&=\E\left[\sum_{t=2\bar{H}+1}^T F_t(x_{t-\bar{H}:t})-f_t(x_{t})\right]\le L\sum_{t=2\bar{H}+1}^T\|(x_{t-\bar{H}},\dots,x_t)-(x_t,\dots,x_t)\|_2\\
&\le \frac{16\sqrt{\eta} nLBH^2\log H}{\sqrt{\sigma}}\sum_{t=2\bar{H}+1}^T \frac{1}{\sqrt{t-2\bar{H}}}\le \frac{16\sqrt{\eta T}nLBH^2\log H}{\sqrt{\sigma}}. 
\end{align*}
It remains to bound the last term in the regret decomposition. For this, we analyze RFTL with delay (RFTL-D).

\subsubsection{RFTL with delay (RFTL-D)}
The subroutine algorithm we used in Algorithm \ref{alg:BCO-quadratic} is Regularized-Follow-the-Leader with delay (RFTL-D). We first analyze its regret bound in the full information setting. Consider a sequence of convex loss functions $\{\ell_t\}_{t=H}^T$ and the following algorithm. 
\begin{algorithm}
\caption{RFTL-D}
\label{alg:RFTL-D}
\begin{algorithmic}[1]
\STATE Input: Bounded, convex, and closed set $\K$, time horizon $T$, delayed length $H$, step size $\eta>0$, regularization function $R(\cdot)$. 
\STATE Initialize $x_t=\argmin_{x\in\K}R(x)$, $\forall t=1,\ldots,H$.
\STATE Set $\ell_t=0$, $\forall t=1,\dots,\bar{H}$. 
\FOR {$t = H, \ldots, T$} 
\STATE Play $x_t$, observe and store cost function $\ell_t(x_t)$.
\STATE Update $x_{t+1}=\argmin_{x\in\K}\left\{\sum_{s=H}^t \ell_{s-\bar{H}}(x)+\frac{1}{\eta}R(x)\right\}$.
\ENDFOR
\end{algorithmic}
\end{algorithm}

Again, note that by design of Algorithm~\ref{alg:RFTL-D}, the learning begins only after $2\bar{H}+1$-th iteration. Therefore, it suffices to bound effective regret $\text{Effective-Regret}_T\defeq \sum_{t=2\bar{H}+1}^{T} \ell_t(x_t)-\min_{x\in\K}\sum_{t=2\bar{H}+1}^T\ell_t(x)$. First, we want to establish a regret inequality which is analogous to the standard regret inequality seen in the Regularized Follow-the-Leader algorithm without delay. 

\begin{theorem} [RFTL-D effective regret bound]
\label{thm:RFTL-D-regret}
With convex loss functions bounded by $B$, Algorithm \ref{alg:RFTL-D} guarantees the following regret bound for every $x\in\K$:
\begin{align*}
    \text{Effective-Regret}_T(x)&\le 2\eta \sum_{t=2\bar{H}+1}^T \|\nabla_{t-\bar{H}}\|_{t,t+1,\Phi_t}^*\left\|\sum_{s=t-\bar{H}}^t \nabla_{s-\bar{H}}\right\|_{t,t+1,\Phi_t}^* \\
    & \ \ \ \ \ +\frac{R(x)-R(x_{2\bar{H}+1})}{\eta}+2HB,
\end{align*}
    where $\|\cdot\|_{t,t+1,\Phi_t}$ and $\|\cdot\|_{t,t+1,\Phi_t}^*$ denote the local norm and its dual induced by the Bregman divergence w.r.t. the function $\Phi_t(x)\defeq \eta\sum_{s=H}^t\ell_{s-\bar{H}}(x)+R(x)$ between $x_t$ and $x_{t+1}$.
\end{theorem}
\begin{proof}
The proof of Theorem \ref{thm:RFTL-D-regret} follows from the following lemma.
\begin{lemma}
\label{lem:RFTL-D-regret}
Suppose the cost functions $\ell_t$ are bounded by $B$. Algorithm \ref{alg:RFTL-D} guarantees the following regret bound:
\begin{align*}
\text{Effective-Regret}_T(x)\le \sum_{t=2\bar{H}+1}^T \nabla_{t-\bar{H}}^{\top}(x_{t-\bar{H}}-x_{t+1})+\frac{R(x)-R(x_{2\bar{H}+1})}{\eta}+2HB.
\end{align*}
\end{lemma}
\begin{proof} [Proof of Lemma~\ref{lem:RFTL-D-regret}]
Denote $h_{2\bar{H}}(x)\defeq\frac{1}{\eta}R(x)$, $h_t(x)\defeq\ell_{t-\bar{H}}(x)$, $\forall t\ge 2\bar{H}+1$. Then, by the usual FTL-BTL analysis, $\forall x\in \K$, $T\ge 2\bar{H}$, $\sum_{t=2\bar{H}}^T h_t(x)\ge \sum_{t=2\bar{H}}^T h_t(x_{t+1})$. Thus, we can bound regret by
\begin{align*}
\regret_T(x)&\le \left(\sum_{t=H}^{T-\bar{H}}\ell_t(x_t)-\ell_t(x)\right)+2HB\\
&=\left(\sum_{t=H}^{T-\bar{H}} \ell_t(x_t)-\ell_t(x)\right)+2HB\\
&=\left(\sum_{t=H}^{T-\bar{H}}\ell_t(x_t)-\ell_t(x_{t+H})\right)+\left(\sum_{t=H}^{T-\bar{H}}\ell_t(x_{t+H})-\ell_t(x)\right)+2HB\\
&=\left(\sum_{t=2\bar{H}+1}^{T}\ell_{t-\bar{H}}(x_{t-\bar{H}})-\ell_{t-\bar{H}}(x_{t+1})\right)+\left(\sum_{t=2\bar{H}+1}^{T}\ell_{t-\bar{H}}(x_{t+1})-\ell_{t-\bar{H}}(x)\right)+2HB\\
&\le \sum_{t=2\bar{H}+1}^T \nabla_{t-\bar{H}}^{\top}(x_{t-\bar{H}}-x_{t+1})+\left(\sum_{t=2\bar{H}+1}^{T}h_t(x_{t+1})-h_t(x)\right)+2HB\\
&\le \sum_{t=2\bar{H}+1}^T \nabla_{t-\bar{H}}^{\top}(x_{t-\bar{H}}-x_{t+1})+\frac{R(x)-R(x_{2\bar{H}+1})}{\eta}+2HB,
\end{align*}
where the last inequality follows from the inequality $\sum_{t=2\bar{H}}^T h_t(x)\ge \sum_{t=2\bar{H}}^T h_t(x_{t+1})$, $\forall x\in\K$. 
\end{proof}
Consider the function $\Phi_t(x)\defeq \eta\sum_{s=2\bar{H}+1}^t\ell_{s-\bar{H}}(x)+R(x)$, $t\ge 2\bar{H}+1$. By Taylor expansion and optimality condition, we have that $\forall t\ge 2\bar{H}+1$, 
\begin{align*}
\Phi_t(x_{t-\bar{H}})&=\Phi_t(x_{t+1})+(x_{t-\bar{H}}-x_{t+1})^{\top}\nabla \Phi_t(x_{t+1})+D_{\Phi_t}(x_{t-\bar{H}},x_{t+1})\\
&\ge \Phi_t(x_{t+1})+D_{\Phi_t}(x_{t-\bar{H}},x_{t+1}),
\end{align*}
which implies a bound on the Bregman divergence between $x_{t-\bar{H}}$ and $x_{t+1}$ with respect to $\Phi_t$,
\begin{align*}
D_{\Phi_t}(x_{t-\bar{H}},x_{t+1})&\le  \Phi_t(x_{t-\bar{H}})-\Phi_t(x_{t+1})\\
&\le \underbrace{\Phi_{t-H}(x_{t-\bar{H}})-\Phi_{t-H}(x_{t+1})}_{\le 0}+\eta\sum_{s=t-\bar{H}}^t\nabla_{s-\bar{H}}^\top(x_{t-\bar{H}}-x_{t+1})\\
&\le \eta \left(\left\|\sum_{s=t-\bar{H}}^t \nabla_{s-\bar{H}}\right\|_{t,t+1,\Phi_t}^*\right)\|x_{t-\bar{H}}-x_{t+1}\|_{t,t+1,\Phi_t}\\
&=\eta\left(\left\|\sum_{s=t-\bar{H}}^t \nabla_{s-\bar{H}}\right\|_{t,t+1,\Phi_t}^*\right)\sqrt{2D_{\Phi_t}(x_{t-\bar{H}},x_{t+1})},
\end{align*}
which gives the bound on both the Bregman divergence and the iterate distance in terms of Bregman divergence induced norm between $x_{t-\bar{H}}$ and $x_{t+1}$,
\begin{align*}
D_{\Phi_t}(x_{t-\bar{H}},x_{t+1})&\le 2\eta^2\left\|\sum_{s=t-\bar{H}}^t \nabla_{s-\bar{H}}\right\|_{t,t+1,\Phi_t}^{*2},\\
\|x_{t-\bar{H}}-x_{t+1}\|_{t,t+1,\Phi_t}&\le 2\eta \left\|\sum_{s=t-\bar{H}}^t \nabla_{s-\bar{H}}\right\|_{t,t+1,\Phi_t}^{*}.
\end{align*}\
Following the expression of the regret bound established in Lemma \ref{lem:RFTL-D-regret}, we bound
\begin{align*}
\text{Effective-Regret}_T(x)&\le \sum_{t=2\bar{H}+1}^T\nabla_{t-\bar{H}}^{\top}(x_{t-\bar{H}}-x_{t+1})+\frac{R(x)-R(x_{2\bar{H}+1})}{\eta}+2HB\\
&\le\sum_{t=2\bar{H}+1}^T \|\nabla_{t-\bar{H}}\|_{t,t+1,\Phi_t}^*\|x_{t-\bar{H}}-x_{t+1}\|_{t,t+1,\Phi_t}\\
& \ \ \ \ \ +\frac{R(x)-R(x_{2\bar{H}+1})}{\eta}+2HB\\
&\le 2\eta \sum_{t=2\bar{H}+1}^T \|\nabla_{t-\bar{H}}\|_{t,t+1,\Phi_t}^*\left\|\sum_{s=t-\bar{H}}^t \nabla_{s-\bar{H}}\right\|_{t,t+1,\Phi_t}^*\\
& \ \ \ \ \ +\frac{R(x)-R(x_{2\bar{H}+1})}{\eta}+2HB.
\end{align*}
\end{proof}

\ignore{
\begin{lemma} [FTL-BTL for RFTL-D]
\label{lem:ftl-btl}
Define $h_{\bar{H}}(x)\defeq\frac{1}{\eta}R(x)$, $h_t(x)\defeq\ell_{t-\bar{H}}(x)$, $\forall t\ge H$. Then, $\forall x\in\K, T$, the following inequality holds for the sequence of $\{x_t\}$ returned by Algorithm \ref{alg:RFTL-D}:
\begin{align*}
\sum_{t=\bar{H}}^T h_t(x)\ge \sum_{t=\bar{H}}^T h_t(x_{t+1}).
\end{align*}
\end{lemma}
\begin{proof}
We induct on $T$. The base case $T=\bar{H}$ holds since $h_{\bar{H}}(x)=R(x)\geq R(x_H)=h_{\bar{H}}(x_H)$ holds by definition of $x_H$ for all $x\in\K$. Suppose now that the above inequality holds for $T$. Then
\begin{align*}
\sum_{t=\bar{H}}^{T+1} h_t(x)\ge \sum_{t=\bar{H}}^{T+1}  h_t(x_{T+2}) \ge \sum_{t=\bar{H}}^{T} h_t(x_{t+1})+h_{T+1}(x_{T+2})=\sum_{t=\bar{H}}^{T+1} h_t(x_{t+1}). 
\end{align*}
\end{proof}
}

\begin{corollary}
\label{cor:RFTL-D-strongly-convex}
In Algorithm \ref{alg:RFTL-D}, if the loss functions are assumed to be $\sigma$-strongly smooth and bounded by $B$, and the updates are given by
\begin{align*}
x_{t+1}=\argmin_{x\in\K} \left\{\left(\sum_{s=H}^t \nabla_{s-\bar{H}}^{\top}x+\frac{\sigma}{2}\|x-x_{s-\bar{H}}\|_2^2\right)+\frac{1}{\eta}R(x)\right\},
\end{align*}
then Algorithm \ref{alg:RFTL-D} guarantees the following regret bound:
\begin{align*}
\text{Effective-Regret}_T\le 2\eta \sum_{t=2\bar{H}+1}^T \|\nabla_{t-\bar{H}}\|_{t,t+1}^*\left\|\sum_{s=t-\bar{H}}^t \nabla_{s-\bar{H}}\right\|_{t,t+1}^*+\frac{R(x)-R(x_{2\bar{H}+1})}{\eta}+HB,
\end{align*}
with the local norms defined as in Definition \ref{def:local-norms}.
\end{corollary}
\begin{proof}
We make use of a lemma of \cite{zinkevich2003online} and \cite{hazan2007logarithmic}:
\begin{lemma}
The following inequality holds for two sequences of convex loss functions $\{\ell_t\}_{t=1}^T,\{\tilde{\ell}_t\}_{t=1}^T$ if $\tilde{\ell}_t(x_t)=\ell_t(x_t)$ and $\tilde{\ell}_t(x)\le \ell_t(x)$, $\forall x\in\K$:
\begin{align*}
\sum_{t=1}^T \ell_t(x_t)-\min_{x\in\K} \sum_{t=1}^T \ell_t(x)\le \sum_{t=1}^T \tilde{\ell}_t(x_t)-\min_{x\in\K}\sum_{t=1}^T \tilde{\ell}_t(x).
\end{align*}
\end{lemma}
Since we assume $\ell_t$s to be $\sigma$-strongly convex, we can construct $\tilde{\ell}_t$ that satisfies $\tilde{\ell}_t(x_t)=\ell_t(x_t)$ and $\tilde{\ell}_t(x)\le \ell_t(x)$, $\forall x\in\K$ as the following:
\begin{align*}
\tilde{\ell}_t(x)\defeq \ell_t(x_t)+\nabla \ell_t(x_t)^\top (x-x_{t})+\frac{\sigma}{2}\|x-x_{t}\|_2^2.
\end{align*}
The update then becomes
\begin{align*}
x_{t+1}&=\argmin_{x\in\K} \left\{\left(\sum_{s=H}^t \nabla_{s-\bar{H}}^{\top}x+\frac{\sigma}{2}\|x-x_{s-\bar{H}}\|_2^2\right)+\frac{1}{\eta}R(x)\right\}\\
&=\argmin_{x\in\K}\left\{\sum_{s=H}^t \tilde{\ell}_{s-\bar{H}}(x)+\frac{1}{\eta} R(x)\right\}.
\end{align*}
Note that $\nabla\tilde{\ell}_t(x_t)=\nabla\ell_t(x_t)$. Let $\|x\|_{t,t+1}= \|x\|_{t,t+1,R_t}$, where $R_t(x)= R(x)+\frac{\eta\sigma}{2}\sum_{s=H}^t \|x-x_{s-\bar{H}}\|_2^2$
Then from Theorem~\ref{thm:RFTL-D-regret} and linearity of $\Phi_t-R_t$,
\begin{align*}
\text{Effective-Regret}_T&\le 2\eta \sum_{t=2\bar{H}+1}^T \|\nabla_{t-\bar{H}}\|_{t,t+1,\Phi_t}^*\left\|\sum_{s=t-\bar{H}}^t \nabla_{s-\bar{H}}\right\|_{t,t+1,\Phi_t}^*\\
& \ \ \ \ \ +\frac{R(x)-R(x_{2\bar{H}+1})}{\eta}+HB\\
&=2\eta \sum_{t=2\bar{H}+1}^T \|\nabla_{t-\bar{H}}\|_{t,t+1}^*\left\|\sum_{s=t-\bar{H}}^t \nabla_{s-\bar{H}}\right\|_{t,t+1}^*+\frac{R(x)-R(x_{2\bar{H}+1})}{\eta}+HB.
\end{align*}
\end{proof}


Corollary~\ref{cor:RFTL-D-strongly-convex} implies that the above regret bound holds if we run RFTL-D with the true gradient of $\{f_t\}_{t=H}^T$ in the full information setting. In the bandit setting, Algorithm~\ref{alg:BCO-quadratic} is run with the gradient estimators $g_t$  in place of the actual gradient $\nabla f_t(x_t)$. We introduce the following lemma that bounds the regret of a first-order OCO algorithm $\A$ when using gradient estimators in place of the true gradient:

\begin{lemma}
\label{lem:first-order-OCO-est}
Let $\ell_1,\dots,\ell_T:\K\rightarrow\R_+$ be a sequence of differentiable convex loss functions. Let $\A$ be a first-order OCO algorithm over $\K$ with regret bound
\begin{align*}
\regret_T^{\A}\le D_{\A}(\nabla \ell_1(x_1),\dots,\nabla \ell_T(x_T)).
\end{align*}
Define $x_1\leftarrow\A(\emptyset)$, $x_t\leftarrow\A(g_1,\dots,g_{t-1})$ for $t\le T$. Suppose $\exists B(t)$ such that the gradient estimator $g_t$ satisfies $\left\|\E\left[g_t\mid \mathcal{G}_{t}\right]-\nabla \ell_t(x_t)\right\|_2\le B(t)$, where $\mathcal{G}_{t}$ is any filtration such that $\ell_t,x_t\in\mathcal{G}_{t}$. Then $\forall x\in\K$,
\begin{align*}
\E\left[\sum_{t=1}^T \ell_t(x_t)-\ell_t(x)\right]\le \E[D_{\A}(g_1,\dots,g_T)]+D\sum_{t=1}^T B(t).
\end{align*}
\end{lemma}
\begin{proof}
Define $q_t(x)\defeq \ell_t(x)+(g_t-\nabla \ell_t(x_t))^{\top}x$. Then $\nabla q_t(x_t)=g_t$. Since $\A$ is a first-order OCO algorithm, $\A(q_1,\dots,q_{t-1})=\A(g_1,\dots,g_{t-1})$, $\forall t$. Moreover, $\forall x\in\K$,
\begin{align*}
\sum_{t=1}^T q_t(x_t)-q_t(x)\le D_{\A}(g_1,\dots,g_T).
\end{align*}
By assumption, $\forall t,x$, 
\begin{align*}
\E[q_t(x_t)-q_t(x)]&=\E[\ell_t(x_t)-\ell_t(x)]-\E[(g_t-\nabla \ell_t(x_t))^\top (x-x_t)]\\
&=\E[\ell_t(x_t)-\ell_t(x)]-\E[\E[(g_t-\nabla \ell_t(x_t))^\top (x-x_t)\mid \mathcal{G}_{t}]]\\
&=\E[\ell_t(x_t)-\ell_t(x)]-\E[(\E[g_t\mid \mathcal{G}_{t}]-\nabla \ell_t(x_t))^\top (x-x_t)]\\
&\ge \E[\ell_t(x_t)-\ell_t(x)]-DB(t).
\end{align*}
Then
\begin{align*}
\E\left[\sum_{t=1}^T \ell_t(x_t)-\ell_t(x)\right]&\le \E\left[\sum_{t=1}^T q_t(x_t)-q_t(x)\right]+D\sum_{t=1}^T B(t)\\
&\le \E[D_{\A}(g_1,\dots,g_T)]+D\sum_{t=1}^T B(t).
\end{align*}
\end{proof}
With Corollary~\ref{cor:RFTL-D-strongly-convex} and Lemma~\ref{lem:first-order-OCO-est}, we are ready to bound the last term in the regret decomposition. 

\begin{lemma} 
\label{lem:bound-term-3}
For any sequence of loss functions $\{F_t\}_{t=H}^T$ satisfying assumptions in \ref{sec:working-assumptions}, the sequence $\{x_t\}_{t=H}^T$ returned by Algorithm \ref{alg:BCO-quadratic} satisfies $\forall x\in\K$,
\begin{align*}
\E\left[\sum_{t=2\bar{H}+1}^T f_{t}(x_t)-f_{t}(x)\right]\le 16\eta n^2B^2H^3\log^2 HT+\frac{\nu\log T}{\eta}+2HB+\frac{16\sqrt{\eta T}\beta nBDH^4}{\sqrt{\sigma}}.
\end{align*}
\end{lemma}
\begin{proof}
Recall the definition of the function $\pi_w$ with respect to $w\in\mathrm{int}(\K)$ in Proposition~\ref{prop:self-concordant-barrier}. For a given $w\in\mathrm{int}(\K)$, $\pi_w:\K\rightarrow\R_+$ is given by $\pi_w(y)=\inf\{t\ge 0: w+t^{-1}(y-w)\in\K\}$. Note that we can assume without loss of generality that $\pi_{x_{2\bar{H}+1}}(x)\le 1-T^{-1}$. Since $F_t$ is $L$-Lipschitz, if $x$ violates this assumption, i.e. $\pi_{x_{2\bar{H}+1}}(x)>1-T^{-1}$, $\exists x'\in\K$ with $\|x-x'\|_2\le\mathcal{O}(T^{-1})$ and $\pi_{x_H}(x')\le 1-T^{-1}$, and if total loss playing $x'$ is at most $O(1)$ away from playing $x$. With this assumption, Proposition~\ref{prop:self-concordant-barrier} readily bounds the quantity $R(w)-R(x_{2\bar{H}+1})$, which is always non-negative since $x_{2\bar{H}+1}=x_1=\argmin_{x\in\K}R(x)$. 

Let $\A$ be the RFTL-D algorithm with updates for $\sigma$-strongly convex functions. Then, the effective regret of bandit RFTL-D with respect to any $x\in\K$ is bounded by
\begin{align*}
&\E\left[\sum_{t=2\bar{H}+1}^Tf_t(x_t)-f_t(x)\right]\\
&= \E\left[\sum_{t=2\bar{H}+1}^T \E\left[f_t(x_t)-f_t(w)\bigg | u_{1:t},F_{H:t-H}\right]\right]\\
&=\E\left[\sum_{t=2\bar{H}+1}^T\bar{f}_t(x_t)-\bar{f}_t(w)\right]\\
&\le \E[D_{\A}(g_H,\dots,g_T)]+\frac{16\sqrt{\eta}\beta nBDH^3}{\sqrt{\sigma}}\sum_{t=2\bar{H}+1}^T \frac{1}{\sqrt{t-2\bar{H}}}  \ \ \ \text{(Corollary \ref{cor:gradient-est-error}, Lemma \ref{lem:first-order-OCO-est})}\\
&\le 2\eta \sum_{t=2\bar{H}+1}^T \E\left[\|g_{t-\bar{H}}\|_{t,t+1}^*\left\|\sum_{s=t-\bar{H}}^t g_{s-\bar{H}}\right\|_{t,t+1}^*\right]+\frac{R(x)-R(x_{2\bar{H}+1})}{\eta}+2HB \\
& \ \ \ \ \ +\frac{16\sqrt{\eta T}\beta nBDH^4}{\sqrt{\sigma}} \ \text{(Corollary \ref{cor:RFTL-D-strongly-convex})}\\
&\le 16\eta n^2B^2H^3\log^2 HT+\frac{R(x)-R(x_{2\bar{H}+1})}{\eta} +2HB+\frac{16\sqrt{\eta T}\beta nBDH^4}{\sqrt{\sigma}} \ \ \ \text{(Lemma \ref{lem:grad-est-norm-bound})}\\
&\le 16\eta n^2B^2H^3\log^2 HT+\frac{\nu\log T}{\eta}+2HB+\frac{16\sqrt{\eta T}\beta nBDH^4}{\sqrt{\sigma}} \ \ \ \text{(Proposition~\ref{prop:self-concordant-barrier})}
\end{align*}
\end{proof}
Lemma~\ref{lem:bound-term-3} establishes the bound on the expected bandit RFTL-D regret. Combining the above bounds, with $H=\mathrm{poly}(\log T)$ we have the following expected regret bound for Algorithm \ref{alg:BCO-quadratic}:
\begin{align*}
\text{Effective-Regret}_T&\le \underbrace{\frac{\beta H\log T}{2\eta\sigma}}_{\text{bound on $(1)$}}+\underbrace{\frac{16\sqrt{\eta T}nLBH^2\log H}{\sqrt{\sigma}}}_{\text{bound on $(2)$}}\\
& \ \ \ \ \ +\underbrace{16\eta n^2B^2H^3\log^2T+\frac{\nu\log T}{\eta}+2HB+\frac{16\sqrt{\eta T}\beta nBDH^4}{\sqrt{\sigma}}}_{\text{bound on $(3)$}}\\
&\le \mathcal{O}\left(\frac{\beta}{\sigma}n\mathrm{poly}(H)\sqrt{T}\right)=\mathcal{O}\left(\frac{\beta}{\sigma}n\mathrm{poly}(\log T)\sqrt{T}\right),
\end{align*}
by taking $\eta=\mathcal{O}\left(\frac{1}{nBH\log H\sqrt{T}}\right)$, with $\mathcal{O}(\cdot)$ hiding polynomials in $D,L,B$. 

\newpage

\section{Proof of \texttt{EBPC} Regret Guarantee for Known Systems}
\label{app:bco-control-knowns}

This section proves the regret bound in Theorem~\ref{thm:control-regret-known} for the BCO-M based controller outlined in Algorithm~\ref{alg:BCO-control}. We will reduce the regret analysis of our proposed bandit LQR/LQG controller to that of BCO-M by designing with-history loss functions $F_t:\M(H,R)^H\rightarrow \R_{+}$ that well-approximates $c_t(\cdot,\cdot)$ for stable systems. In Section~\ref{sec:with-history-construction}, we provide the precise definitions of the with-history loss functions and proceed to check their regularity conditions as required by Corollary~\ref{cor:regret-expected-convex} in Section~\ref{sec:control-loss-regularity}. In Section~\ref{sec:known-control-regret-decomposition}, we analyze the regret of Algorithm~\ref{alg:BCO-control} by bounding both the regret with respect to the with-history loss functions and the approximation error of the with-history loss functions to the true cost functions when evaluating on a single control policy parametrized by some $M\in\M(H,R)$. 
\subsection{Construction of with-history loss functions}
\label{sec:with-history-construction}
In the bandit control task using our proposed bandit controller outlined in Algorithm~\ref{alg:BCO-control}, there are two independent sources of noise: the gradient estimator $g_t$ used in Algorithm~\ref{alg:BCO-control} and the perturbation sequence $\{(\wstoch_t,\estoch_t)\}_{t=1}^T$ injected to the partially observable linear dynamical system. Formally, we define the following filtrations generated by these two sources of noises.
\begin{definition} [Noise filtrations]
\label{def:filtrations}
For all $1\le t\le T$, let  $\mathcal{F}_{t}\defeq \sigma(\{\eps_s\}_{0\le s\le t})$ be the filtration generated by the noises sampled to create the gradient estimator in the algorithm up to time $t$. Let $\mathcal{G}_{t}\defeq \sigma(\{(\wstoch_s,\estoch_s)\}_{0\le s\le t})$ be the filtration generated by the stochastic part of the semi-adversarial perturbation to the linear systems up till time $t$. 
\end{definition}
The main insight in the analysis of online nonstochastic control algorithms is the reduction of the control problem to an online learning with memory problem. To this end, we construct the with-history loss functions as follows:
\begin{definition} [With-history loss functions for known systems]
\label{def:with-history-construction}
Given a Markov operator $G$ of a partially observable linear dynamical system and an incidental cost function $c_t:\R^{d_\y}\times\R^{d_\uv}\rightarrow\R_+$ at time $t$, its corresponding with-history loss function at time $t$ is given a (random) function $F_t:\M(H,R)^H\rightarrow\R$ of the form
\begin{align*}
F_t(N_1,\dots,N_H)\defeq c_t\left(\ynat_t+\sum_{i=1}^{\bar{H}}G^{[i]}\sum_{j=0}^{\bar{H}}N_{H-i}^{[j]}\ynat_{t-i-j}+\sum_{i=H}^{t}G^{[i]}\sum_{j=0}^{\bar{H}}\widetilde{M}_{t-i}^{[j]}\ynat_{t-i-j},  \sum_{j=0}^{\bar{H}}N_{H}^{[j]}\ynat_{t-j}\right).
\end{align*}
Additionally, denote the unary form $f_t:\M(H,R)\rightarrow\R_+$ induced by $F_t$ as $f_t(N)\defeq F_t\underbrace{(N,\dots,N)}_{N\text{ in all $H$ indices}}$. 
\end{definition}
We immediately note a connection of the with-history loss functions constructed in Definition~\ref{def:with-history-construction} to the cost functions. Observe that by expression $\y_t,\uv_t$ resulted from running Algorithm~\ref{alg:BCO-control} explicitly, 
\begin{align*}
c_t\left(\y_t,\uv_t\right)=c_t\left(\ynat_t+\sum_{i=1}^{t}G^{[i]}\sum_{j=0}^{\bar{H}} \widetilde{M}_{t-i}^{[j]}\ynat_{t-i-j}, \ \sum_{j=0}^{\bar{H}}\widetilde{M}_t^{[j]}\ynat_{t-j}\right)= F_t(\widetilde{M}_{t-\bar{H}:t}).
\end{align*}
\begin{remark}
\label{rmk:with-history-loss-adversary-assumption}
Note that $\ynat_t$ is independent of $\mathcal{F}_{T}$. Therefore, by construction, $F_t$ is a $\mathcal{F}_{t-H}\cup \mathcal{G}_t$-measurable random function that is independent of $\eps_{t-\bar{H}:t}$. In particular, Assumption~\ref{assumption:adversary} on the adversary is satisfied. 
\end{remark}

It is left to check the regularity assumptions of $F_t$, which we defer to Section~\ref{sec:control-loss-regularity}. 

\subsection{Regularity condition of with-history loss functions}
\label{sec:control-loss-regularity}
The goal of this section is to establish the other conditions to apply the result of Corollary~\ref{cor:regret-expected-convex}. The following table summarizes the results in this section. 
\begin{center}
\begin{tabular}{ |c|c|c| } 
 \hline
 Parameter & Definition & Magnitude \\ 
 \hline
 $R_{\y}$ & $\ell_2$ bound on observations & $R_{\mathrm{nat}}(1+RR_G)$ \\ 
 \hline
 $R_{\uv}$ & $\ell_2$ bound on controls based on $\M(H,R)$ & $R_{\mathrm{nat}}R$ \\
 \hline
 $B$ & diameter bound on $c_t, F_t, f_t$ & $L_cR_{\mathrm{nat}}^2((1+RR_G)^2+R^2)$ \\
 \hline
 $D$ & diameter bound on $\M(H,R)$ & $2\sqrt{d_{\uv}\wedge d_{\y}}R$\\
 \hline
 $\sigma_f$ & conditional strong convexity parameter of $f_t$ & 
 $\sigma_c\left(\sigma_{\e}^2+\sigma_{\w}^2\frac{\sigma_{\min}(C)}{1+\|A\|_{\mathrm{op}}^2}\right)$\\
 \hline
 $\beta_F$ & smoothness parameter of $F_t$ & $4\beta_cR_{\mathrm{nat}}^2R_G^2H$\\
 \hline
 $L_F$ & Lipschitz parameter of $F_t$ & $2L_c\sqrt{(1+RR_G)^2+R^2}R_GR_{\mathrm{nat}}^2\sqrt{H}$\\
 \hline
\end{tabular}
\end{center}

We start with bounding $\ell_2$-norm on the observed signals $\y_t$ and controls $\uv_t$ played by Algorithm~\ref{alg:BCO-control}.  
\begin{lemma} [Observation and control norm bounds]
\label{lem:yu-bound-known}
Denote $R_{\y}:=\sup_{t}\|\y_t\|_2$ and $R_{\uv}:=\sup_{t}\|\uv_t\|_2$. Then, the following bounds hold deterministically:
\begin{align*}
R_{\y}\le R_{\mathrm{nat}}(1+RR_G), \ \ \ R_{\uv}\le R_{\mathrm{nat}}R. 
\end{align*}
\end{lemma}
\begin{proof}
By algorithm specification, $\y_t,\uv_t$ allow the following expansions:
\begin{align*}
\|\uv_t\|_2&=\left\|\sum_{j=0}^{\bar{H}}\widetilde{M}_{t}^{[j]}\ynat_{t-j}\right\|_2\le \max_{0\le j\le \bar{H}} \|\ynat_{t-j}\|_2\left\|\sum_{j=0}^{\bar{H}}\widetilde{M}_{t}^{[j]}\right\|_{\mathrm{op}}\le_1 R_{\mathrm{nat}}R,\\
\|\y_t\|_2&=\left\|\ynat_t+\sum_{i=1}^t G^{[i]}\uv_{t-i}\right\|_2\le R_{\mathrm{nat}}+\max_{1\le i\le t} \|\uv_{t-i}\|_2 \left\|\sum_{i=1}^t G^{[i]}\right\|_{\mathrm{op}}\le R_{\mathrm{nat}}+R_{\mathrm{nat}}RR_G,
\end{align*}
where $\le_1$ follows from $\widetilde{M}_{t}\in\M(H,R)$ for all $t$ by Remark~\ref{rmk:correctness}. 
\end{proof}

\begin{lemma} [Diameter bounds] 
\label{lem:diam-bound-known}
Given a Markov operator $G$ of a stable partially observable linear dynamical system. 
Let $\mathcal{U}:=\bigg\{\sum_{j=0}^{\bar{H}}M^{[j]}\zeta_{j}: M\in\M(H,R),\zeta_j\in\R^{d_{\y}}, \|\zeta_j\|_2\le R_{\mathrm{nat}}\bigg\}$ and $\Y:=\bigg\{\zeta+\sum_{i=1}^{T-1}G^{[i]}\xi_{i}: \zeta\in\R^{d_{\y}}, \|\zeta\|_2\le R_{\mathrm{nat}},\xi_i\in\mathcal{U}\bigg\}$. Denote $B=\underset{\y\in\Y}{\sup}\underset{\uv\in\mathcal{U}}{\sup}\underset{1\le t\le T}{\sup} \ c_t(\y,\uv)$. Denote $D=\underset{M,M'\in\M(H,R)}{\sup}\|M-M'\|_F$. Then,
\begin{align*}
B\le L_cR_{\mathrm{nat}}^2((1+RR_G)^2+R^2), \ \ \ D\le 2\sqrt{d_{\uv}\wedge d_{\y}}R.
\end{align*}
\end{lemma}
\begin{proof}
Recall the quadratic and Lipschitz assumption on $c_t$. $\forall \y\in\Y,\uv\in\mathcal{U},1\le t\le T$, 
\begin{align*}
c_t(\y,\uv)\le L_c \|(\y,\uv)\|_2^2=L_c(\|\y\|_2^2+\|\uv\|_2^2)\le L_cR_{\mathrm{nat}}^2((1+RR_G)^2+R^2). 
\end{align*}
For any $M,M'\in\M(H,R)$, we have
\begin{align*}
\|M-M'\|_F\le\sum_{j=0}^{\bar{H}} \|M^{[j]}-M'^{[j]}\|_F\le \sqrt{d_{\uv}\wedge d_{\y}} \sum_{j=0}^{\bar{H}} \|M^{[j]}-M'^{[j]}\|_{\mathrm{op}}\le  2\sqrt{d_{\uv}\wedge d_{\y}}R. 
\end{align*}
\end{proof}
In particular, Lemma~\ref{lem:diam-bound-known} implies the diameter bound for $c_t(\y_t,\uv_t)$, $\forall t$, and $c_t(\y_t^M,\uv_t^M)$, $\forall t$, $\forall M\in\M(H,R)$ as well as $F_t$ on $\M(H,R)^H$ and $f_t$ on $\M(H,R)$, $\forall t$. We proceed to check other regularity conditions for $F_t$ and $f_t$. 
\begin{lemma} [Regularity conditions of $F_t$ and $f_t$]
\label{lem:with-history-regularity-known}
Let $F_t$ and $f_t$ be given as in Definition~\ref{def:with-history-construction}, and $G$ be the Markov operator of a partially observable linear dynamical system. $F_t$ and $f_t$ satisfy the following regularity conditions $\forall t$:
\begin{itemize}
\item The function $\E[f_t(\cdot)\mid\mathcal{F}_{t-H}\cup\mathcal{G}_{t-H}]$ defined on $\M(H,R)$ is $\sigma_f$-strongly convex with strong convexity parameter $\sigma_f=\sigma_c\left(\sigma_{\e}^2+\sigma_{\w}^2\frac{\sigma_{\min}(C)}{1+\|A\|_{\mathrm{op}}^2}\right)$. 
\item $F_t$ is quadratic and $\beta_F$-smooth with $\beta_F=4\beta_cR_{\mathrm{nat}}^2R_G^2H$. 
\item $F_t$ is $L_F$-Lipschitz with $L_F= 2L_c\sqrt{(1+RR_G)^2+R^2}R_GR_{\mathrm{nat}}^2\sqrt{H}$.
\end{itemize}
\end{lemma}

\begin{proof}
First, we show the conditional strong convexity. Recall that $c_t$ is quadratic, therefore $c_t(\y_t,\uv_t)=\y_t^{\top}Q_t\y_t+\uv_t^{\top}R_t\uv_t$. Consider the following quantities:
\begin{align*}
\mathbf{S}_M\defeq\ynat_t+\sum_{i=1}^{\bar{H}} G^{[i]}\sum_{j=0}^{\bar{H}}M^{[j]}\ynat_{t-i-j}, \ \ \mathbf{F}\defeq\sum_{i=H}^{t}G^{[i]}\sum_{j=0}^{\bar{H}}\widetilde{M}_{t-i}^{[j]}\ynat_{t-i-j}, \mathbf{C}_M \defeq\sum_{j=0}^{\bar{H}}M^{[j]}\ynat_{t-j}.
\end{align*}
Note that $\mathbf{S}_M, \mathbf{C}_M$ are independent of $\mathcal{F}_{t-H}$ and $\mathbf{F}\in \mathcal{F}_{t-H}\cup \mathcal{G}_{t-H}$. Thus, 
\begin{align*}
\E\left[f_t(M)\mid \mathcal{F}_{t-H}\cup \mathcal{G}_{t-H}\right]&=\E\left[c_t(\mathbf{S}_M+F,\mathbf{C}_M)\mid \mathcal{F}_{t-H}\cup \mathcal{G}_{t-H}\right]\\
&=\E[\mathbf{S}_M^{\top}Q_t\mathbf{S}_M\mid \mathcal{G}_{t-H}]+\mathbf{F}^{\top}(Q_t+Q_t^{\top})\E[\mathbf{S}_M\mid\mathcal{G}_{t-H}]\\
& \ \ \ \ \ + \mathbf{F}^{\top}Q_t \mathbf{F}+\E[\mathbf{C}_M^{\top}R_t\mathbf{C}_M\mid \mathcal{G}_{t-H}]\\
&=\E[c_t(\mathbf{S}_M,\mathbf{C}_M)\mid \mathcal{G}_{t-H}] + \ell(M),
\end{align*}
where $\ell(M)\defeq \mathbf{F}^{\top}(Q_t+Q_t^{\top})\E[\mathbf{S}_M\mid\mathcal{G}_{t-H}]+ \mathbf{F}^{\top}Q_t \mathbf{F}$ is affine in $M$. The strong convexity of $\E[c_t(\mathbf{S}_M,\mathbf{C}_M)\mid \mathcal{G}_{t-H}]$ is established by the following Lemma from \cite{simchowitz2020improper}:
\begin{lemma} [Lemma J.10 and Lemma J.15 in~\citep{simchowitz2020improper}]
\label{lem:strong-convexity-lemma}
$\forall M\in\M(H,R)$, 
\begin{align*}
\E\left[\left\|(\mathbf{S}_{M},\mathbf{C}_{M})-(\ynat_t,\mathbf{0}_{d_{\uv}})\right\|_2^2\mid \mathcal{G}_{t-H}\right]\ge\left(\sigma_{\e}^2+\sigma_{\w}^2\frac{\sigma_{\min}(C)}{1+\|A\|_{\mathrm{op}}^2}\right)\|M\|_F^2.
\end{align*}
\end{lemma}
The above lemma implies that $\E[c_t(\mathbf{S}_M,\mathbf{C}_M)\mid \mathcal{G}_{t-H}]$ is $\sigma_f$-strongly convex for $\sigma_f=\sigma_c\left(\sigma_{\e}^2+\sigma_{\w}^2\frac{\sigma_{\min}(C)}{1+\|A\|_{\mathrm{op}}^2}\right)$ on $\M(H,R)$. 

By assumption, $c_t(\cdot,\cdot)$ is $\beta_c$-smooth. 
\begin{align*}
F_t(N_1,\dots,N_H)=(\mathbf{S}_{N_{1:H}}+\mathbf{F})^{\top} Q_t (\mathbf{S}_{N_{1:H}}+\mathbf{F})+\mathbf{C}_{N_{H}}^{\top}R_t\mathbf{C}_{N_{H}},
\end{align*}
where $\mathbf{S}_{N_{1:H}}\defeq \ynat_t+\sum_{i=1}^{\bar{H}}G^{[i]}\sum_{j=0}^{\bar{H}}N_{H-i}^{[j]}\ynat_{t-i-j}$ and $\mathbf{C}_{N_{H}} \defeq\sum_{j=0}^{\bar{H}}N_{H}^{[j]}\ynat_{t-j}$ are linear in $N_{1:H}$. $F_t$ is quadratic by the above expression. Moreover, $F_t$ is $\beta_F$-smooth if and only if $c_t(\mathbf{S}_{N_{1:H}},\mathbf{C}_{N_{H}})$ is $\beta_F$-smooth as a function of $N_{1:H}$. We proceed to bound $\beta_F$. Consider the linear operator $v:\M(H,R)^H\rightarrow \R^{2d_{\y}}$ given by $v(N_{1:H})=(\mathbf{S}_{N_{1:H}},\mathbf{C}_{N_{H}})$. Then $\forall N_{1:H},N_{1:H}'\subset\M(H,R)$, 
\begin{align*}
\|v(N_{1:H})-v(N_{1:H}')\|_2&=\left\|\left(\sum_{i=1}^{\bar{H}}G^{[i]}\sum_{j=0}^{\bar{H}}(N_{H-i}^{[j]}-N_{H-i}'^{[j]})\ynat_{t-i-j}, \sum_{j=0}^{\bar{H}}(N_{H}^{[j]}-N_{H}'^{[j]})\ynat_{t-j}\right)\right\|_2\\
&\le \left\|\sum_{i=1}^{\bar{H}}G^{[i]}\sum_{j=0}^{\bar{H}}(N_{H-i}^{[j]}-N_{H-i}'^{[j]})\ynat_{t-i-j}\right\|_2+\left\|\sum_{j=0}^{\bar{H}}(N_{H}^{[j]}-N_{H}'^{[j]})\ynat_{t-j}\right\|_2\\
&\le \sum_{i=1}^{\bar{H}}\|G^{[i]}\|_{\mathrm{op}}\left\|\sum_{j=0}^{\bar{H}}(N_{H-i}^{[j]}-N_{H-i}'^{[j]})\ynat_{t-i-j}\right\|_2+\left\|\sum_{j=0}^{\bar{H}}(N_{H}^{[j]}-N_{H}'^{[j]})\ynat_{t-j}\right\|_2\\
&\le R_G\max_{1\le i\le\bar{H}} \left\|\sum_{j=0}^{\bar{H}}(N_i^{[j]}-N_i'^{[j]})\ynat_{t-(H-i)-j}\right\|_2+\left\|\sum_{j=0}^{\bar{H}}(N_{H}^{[j]}-N_{H}'^{[j]})\ynat_{t-j}\right\|_2\\
&\le 2R_GR_{\mathrm{nat}}\max_{1\le i\le H} \|N_i-N_i'\|_{\ell_1,\mathrm{op}}\\
&\le 2R_GR_{\mathrm{nat}}\sqrt{H}\|N_{1:H}-N_{1:H}'\|_F,
\end{align*}
which bounds $\|\mathbf{D}v(N_{1:H})\|_2\le 2R_GR_{\mathrm{nat}}\sqrt{H}$ and thus
\begin{align*}
\|\nabla^2 (c_t(v(N_{1:H})))\|_{\mathrm{op}}&=\|\mathbf{D}v(N_{1:H})(\nabla^2 c_t)(v(N_{1:H}))\mathbf{D}v(N_{1:H})^{\top}\|_{\mathrm{op}}\\
&\le \beta_c \|\mathbf{D}v(N_{1:H})\|_2^2\\
&\le 4\beta_c R_{\mathrm{nat}}^2R_{G}^2H. 
\end{align*}
$F_t$ is $\beta_F\defeq 4\beta_c R_{\mathrm{nat}}^2R_{G}^2H$-smooth since $c_t(\mathbf{S}_{N_{1:H}},\mathbf{C}_{N_{H}})$ is $4\beta_c R_{\mathrm{nat}}^2R_{G}^2H$-smooth. 

It is left to bound the gradient for $F_t$. Note that
\begin{align*}
\|\nabla F_t(N_1,\dots,N_H)\|_2&=\|(\nabla c_t)(v(N_{1:H}))\|_2\|\mathbf{D}v(N_{1:H})\|_2\\
&\le 2L_c\sqrt{(1+RR_G)^2+R^2}R_GR_{\mathrm{nat}}^2\sqrt{H}\\
&=L_F.
\end{align*}
\end{proof}

\ignore{
\paragraph{Diameter and gradient bound.} 
The diameter bound of $F_t$ is the same of the diameter bound of $c_t$, which is bounded by $B=L_c(R_{\y}^2+R_{\uv}^2)$ by the Lipschitz condition on $c(\cdot,\cdot)$, where 
\begin{align*}
R_{\uv}&\defeq \sup_{t} \sup_{M\in\M(H,R)} \left\|\sum_{j=0}^{\bar{H}}M^{[j]}\ynat_{t-j}\right\|_2\le RR_{\mathrm{nat}},\\
R_{\y}&\defeq \sup_t \  \left\|\ynat_t+\sum_{i=1}^{t-1}G^{[i]}\uv_{t-i}\right\|_2\le R_{\mathrm{nat}}+R_GR_{\uv}\le R_{\mathrm{nat}}(1+RR_G). 
\end{align*}
The gradient bound $L_F$ is given by
\begin{align*}
\|\nabla F_t(N_1,\dots,N_H)\|_2&\le (\|(\nabla c_t)(v(N_{1:H}))\|_2+2\beta\|\mathbf{F}\|_2)\|\mathbf{D}v(N_{1:H})\|_2\\
&\le \left(L_c\sqrt{R_{\y}^2+R_{\uv}^2}+2\beta RR_{\mathrm{nat}}\psi_G(H)\right)(2R_{\mathrm{nat}}R_G\sqrt{H})\\
&=:L_F. 
\end{align*}
}

\subsection{Controller regret decomposition and analysis}
\label{sec:known-control-regret-decomposition}
Recall the definition of regret for the controller algorithm:
\begin{align*}
\regret_T(\texttt{controller})&= J_T(\texttt{controller})-\inf_{M\in\M(H,R)} J_T(\pi_M)\\
&=\sum_{t=1}^T c_t(\y_t,\uv_t)-\inf_{M\in\M(H,R)} \sum_{t=1}^T c_t(\y_t^M,\uv_t^M),
\end{align*}
where $\uv_t$ is the control played by the controller algorithm at time $t$ and $\y_t$ is the observation attained by the algorithm's history of controls at time $t$. $(\y_t^M,\uv_t^M)$ is the observation-control pair that would have been returned if the DRC policy $M$ were executed from the beginning of the time. 
The above regret can be decomposed in the following way.
\begin{align*}
\regret_T(\texttt{controller})
&=\underbrace{\left(\sum_{t=1}^{2\bar{H}} c_t(\y_t,\uv_t)\right)}_{(\text{burn-in loss})}+\underbrace{\left(\sum_{t=2\bar{H}+1}^T F_t(\widetilde{M}_{t-\bar{H}:t})-\inf_{M\in\M(H,R)}\sum_{t=2\bar{H}+1}^Tf_t(M)\right)}_{(\text{effective BCO-M regret})}\\
& \ \ \ \ \ + \underbrace{\left(\inf_{M\in\M(H,R)}\sum_{t=2\bar{H}+1}^Tf_t(M)-\inf_{M\in\M(H,R)}\sum_{t=2\bar{H}+1}^T c_t(\y_t^M,\uv_t^M)\right)}_{(\text{control truncation loss})}
\end{align*}
The first term is the loss incurred by the initialization stage of the algorithm. The second term entails the regret guarantee with respect to the with-history loss functions defined in Section~\ref{sec:with-history-construction}, which we bound by a combination of the result of Corollary~\ref{cor:regret-expected-convex} and the regularity conditions established in Section~\ref{sec:control-loss-regularity}. The third term is a truncation loss of the comparator used in the regret analysis. In particular, $c_t(\cdot,\cdot)$ has history of length $t$, but the constructed $f_t$ only has history of length $H$. Therefore, each term in the summand of the first term in the control truncation loss measures the counterfactual cost at time $t$ had $M$ been used in constructing the control since $H$ steps back, while each term in the summand of the second term in the control truncation loss measures the counterfactual cost at time $t$ had $M$ been applied to construct the controls from the beginning of the time. The control truncation loss is bounded by the decaying behavior of stable systems, where effects of past controls decay exponentially over time. 

We bound each term separately. First, the burn-in loss can be crudely bounded by the diameter bound $B$ of $c_t(\cdot,\cdot)$, which is established by Lemma~\ref{lem:diam-bound-known} in Section~\ref{sec:control-loss-regularity} by the Lipschitz assumption of $c_t(\cdot,\cdot)$. In particular, applying the diameter bound and under the assumption that $H=\mathrm{poly}(\log T)$,
\begin{align*}
(\text{burn-in loss})\le 2HB\le 2HL_cR_{\mathrm{nat}}^2 (R^2+(1+RR_G)^2)=\tilde{\mathcal{O}}(1),
\end{align*}
Then, we bound the control truncation loss. By the decaying behavior of stable systems, $\psi_G(H)\le O(T^{-1})$ for $H$ taken to be $\mathrm{poly}(\log T)$. 
\begin{align*}
(\text{control truncation loss})&\le \sup_{M\in\M(H,R)} \bigg\{\sum_{t=2\bar{H}+1}^T f_t(M)-c_t(\y_t^M,\uv_t^M)\bigg\}\\
&= \sup_{M\in\M(H,R)} \bigg \{\sum_{t=2\bar{H}+1}^T c_t\bigg(\ynat_t+\sum_{i=1}^{\bar{H}}G^{[i]}\sum_{j=0}^{\bar{H}}M^{[j]}\ynat_{t-i-j}\\
& \ \ \ \ \ +\sum_{i=H}^t G^{[i]}\sum_{j=0}^{\bar{H}}\widetilde{M}_{t-i}^{[j]}\ynat_{t-i-j},\uv_t^M\bigg) - c_t\bigg(\ynat_t+\sum_{i=1}^t G^{[i]}\sum_{j=0}^{\bar{H}}M^{[j]}\ynat_{t-i-j},\uv_t^M\bigg)\bigg\}\\
&\le L_cB \sup_{M\in\M(H,R)} \sum_{t=2\bar{H}+1}^T\left\|\sum_{i=H}^tG^{[i]}\sum_{j=0}^{\bar{H}}(\widetilde{M}_{t-i}^{[j]}-M^{[j]})\ynat_{t-i-j}\right\|_2\\
&\le L_cBT\psi_G(H) \cdot \sup_{t}\sup_{M\in\M(H,R)} \left\|\sum_{j=0}^{\bar{H}}(\widetilde{M}_{t-i}^{[j]}-M^{[j]})\ynat_{t-i-j}\right\|_2\\
&\le 2L_cBT\psi_G(H)RR_{\mathrm{nat}}\\
&\le \mathcal{O}(1).
\end{align*}
It is left to bound the effective BCO-M regret. By construction of $F_t$ in Section~\ref{sec:with-history-construction} and algorithm specification of our proposed bandit controller in Algorithm~\ref{alg:BCO-control}, we are essentially running BCO-M algorithm (Algorithm~\ref{alg:BCO-quadratic}) with the sequence of loss functions $\{F_t\}_{t=H}^T$ on the constraint set $\K=\M(H,R)$. 
Note that further by the analysis in Section~\ref{sec:control-loss-regularity} and Remark~\ref{rmk:with-history-loss-adversary-assumption} and substituting the parameters $\sigma=\sigma_f$, $\beta=\beta_F$, $L=L_F$, and diameter bounds $B,D$ defined in Lemma~\ref{lem:diam-bound-known}, Corollary~\ref{cor:regret-expected-convex} immediately implies that 
\begin{align*}
\E[(\text{effective BCO-M regret})]\le \tilde{\mathcal{O}}\left(\frac{\beta d_{\uv}d_{\y}}{\sigma_c}\sqrt{T}\right),
\end{align*}
since all the parameters for $F_t$ obtained in Section~\ref{sec:control-loss-regularity} differ from the parameters of $c_t$ by factors of at most logarithmic in $T$. Putting together, the regret of the bandit controller is bounded by
\begin{align*}
\E[\regret_T(\texttt{controller})]\le \tilde{\mathcal{O}}\left(\frac{\beta_cd_{\uv}d_{\y}}{\sigma_c}\sqrt{T}\right). 
\end{align*}
\newpage
\section{Proof of \texttt{EBPC} Regret Guarantee for Unknown Systems}
\label{sec:unknown-system-control-regret}
When the system is unknown, we run an estimation algorithm outlined in Algorithm~\ref{alg:est}, followed by our proposed BCO-M based control algorithm with slightly modified parameters. In particular, to compare with the single best policy in the DRC policy class parametrized by $\M(H,R)$, we let $\K=\M(H^+,R^+)$, where $H^+=3H$ and $R^+=2R$ and set history parameter to be $H^+$. Subsequently, we denote $\overline{H^+}\defeq H^+-1$. This section will be organized as the following: Section~\ref{sec:sys-est-error} introduces a previously known error guarantee for the estimation algorithm outlined in Algorithm~\ref{alg:est}; Section~\ref{sec:unknown-construction-with-history} defines the estimated with-history loss functions; 

\subsection{System estimation error guarantee}
\label{sec:sys-est-error}
When the system is unknown, we would need to first run a system estimation algorithm to obtain an estimator $\hat{G}$ for the Markov operator $G$, which we use as an input to our control algorithm outlined in Algorithm~\ref{alg:BCO-control}. It is known that the estimation algorithm we outlined in Algorithm~\ref{alg:est} has high probability error guarantee in its estimated Markov operator, formally given by the following theorem. 

\begin{theorem} [Theorem 7, \cite{simchowitz2020improper}]
\label{thm:est-alg}
With probability at least $1-\delta-N^{-(\log N)^2}$, Algorithm~\ref{alg:est} guarantees that with $\eps_G(N,\delta)\asymp\frac{1}{\sqrt{N}}H^2R_{\mathrm{nat}}\sqrt{(d_{\y}\vee d_{\uv})+\log\left(\frac{1}{\delta}\right)+\log(1+R_{\mathrm{nat}})}$, the following inequalities hold: 
\begin{enumerate}
\item $\|\uv_t\|_2\le R_{\uv,\delta}\defeq 5\sqrt{d_{\uv}+2\log\left(\frac{3}{\delta}\right)}$, $\forall t\in[N]$.
\item $\|\hat{G}-G\|_{\ell_1,\mathrm{op}}\le \eps_G(N,\delta)\le \frac{1}{2\max\{RR_G,R_{\uv,\delta}\}}$.
\end{enumerate}
\end{theorem}

\begin{remark}
Denote $E$ as the event where the two inequalities of Theorem~\ref{thm:est-alg} hold. We are interested in the expected regret of our proposed bandit controller, which is
\begin{align*}
&\E[\regret_T(\texttt{controller})\mid E] \mathbb{P}(E)+\E[\regret_T(\texttt{controller})\mid E^C] \mathbb{P}(E^C)\\
\le  &\E[\regret_T(\texttt{controller})\mid E]+(\delta+N^{-(\log N)^2})\hat{B}T,
\end{align*}
where $\hat{B}$ denotes the bound on the cost $c_t$ when performing controls assuming $\hat{G}$ is the true Markov operator. We will show in Section~\ref{sec:pseudo-loss-regularity} that $\hat{B}\lesssim B$. Therefore, when $\delta\le \frac{1}{\sqrt{T}}$ and $N\ge \sqrt{T}$, we have $(\delta+N^{-(\log N)^2})BT\le \mathcal{O}(\sqrt{T})$. Therefore, from now on we make the following assumption: 
\begin{assumption} [Estimation error]
\label{assumption:est-error}
The estimation sample size $N$ and error parameter $\delta$ are set to be $N=\lceil \sqrt{T}\rceil$ and $\delta=\frac{1}{\sqrt{T}}$. The estimated Markov operator $\hat{G}$ obtained from Algorithm~\ref{alg:est} satisfies the following two inequalities with $\eps_G\asymp\frac{1}{\sqrt{N}}H^2R_{\mathrm{nat}}\sqrt{(d_{\y}\vee d_{\uv})+\log\left(\frac{1}{\delta}\right)+\log(1+R_{\mathrm{nat}})}$:
\begin{enumerate}
\item $\|\uv_t\|_2\le R_{\uv,\delta}$, $\forall t\in[N]$.
\item $\|\hat{G}-G\|_{\ell_1,\mathrm{op}}\le \eps_G\le \frac{1}{2\max\{RR_G,R_{\uv,\delta}\}}$.
\end{enumerate}
Additionally, without loss of generality we assume that $\eps_G\le R_G$. 
\end{assumption}
\end{remark}

\subsection{Construction of estimated with-history loss functions}
\label{sec:unknown-construction-with-history}
Once we obtain $\hat{G}$ from Algorithm~\ref{alg:est} for $N$ iterations, we invoke Algorithm~\ref{alg:BCO-control} treating $\hat{G}$ as the input Markov operator on $\K=M(H^+,R^+)$ with history parameter $H^+$. In this case, the cost functions $c_t(\y_t,\uv_t)$ evaluated by the $(\y_t,\uv_t)$ resulted from playing Algorithm~\ref{alg:BCO-control} allows the following two equivalent expressions:
\begin{align*}
c_t(\y_t,\uv_t)&=c_t\left(\ynat_t+\sum_{i=1}^t G^{[i]}\uv_{t-i}, \uv_t\right)=c_t\left(\ynat_t+\sum_{i=1}^{t}G^{[i]}\sum_{j=0}^{\overline{H^+}}\widetilde{M}_{t-i}^{[j]}\hat{\y}_{t-i-j}^{\mathbf{nat}},  \sum_{j=0}^{\overline{H^+}}\widetilde{M}_{t}^{[j]}\hat{\y}_{t-j}^{\mathbf{nat}}\right)\\
&=c_t\left(\hat{\y}^{\mathbf{nat}}_{t}+\sum_{i=1}^{t}\hat{G}^{[i]}\uv_{t-i},  \uv_t\right)=c_t\left(\hat{\y}^{\mathbf{nat}}_{t}+\sum_{i=1}^{\overline{H^+}}\hat{G}^{[i]}\sum_{j=0}^{\overline{H^+}}\widetilde{M}_{t-i}^{[j]}\hat{\y}_{t-i-j}^{\mathbf{nat}},  \sum_{j=0}^{\overline{H^+}}\widetilde{M}_{t}^{[j]}\hat{\y}_{t-j}^{\mathbf{nat}}\right),
\end{align*}
where $\hat{\y}_{t}^{\mathbf{nat}}$ is the nature's $\y$ calculated by the algorithm at time $t$ using the estimated Markov operator $\hat{G}$. The last inequality follows from $\hat{G}^{[i]}=0$ for $i\ge H^+$. We construct the two estimated with-history loss functions. First, we construct with-history loss functions analogous to the $F_t$ constructed in Section~\ref{sec:with-history-construction} for the known system.
\begin{remark}
\label{rmk:ynat-dependence-unknown}
By specification in the bandit controller outlined in Algorithm~\ref{alg:BCO-control}, $\hat{\y}^{\mathbf{nat}}_t$ is obtained by the formula $\hat{\y}^{\mathbf{nat}}_t\leftarrow \y_t-\sum_{i=1}^{t-1}\hat{G}^{[i]}\uv_{t-1-i}$, and thus $\hat{\y}^{\mathbf{nat}}_t\in \mathcal{F}_{t-H}\cup\mathcal{G}_{t}$ with the filtrations defined in Definition~\ref{def:filtrations}. 
\end{remark}

\begin{definition} [With-history losses for unknown system]
\label{def:with-history-loss-unknown}
Given an estimated Markov operator $\hat{G}$ of a partially observable linear dynamical system and an incidental cost function $c_t:\R^{d_{\y}}\times \R^{d_{\uv}}\rightarrow \R_+$ at time $t$, define its with-history loss at time $t$ to be $\hat{F}_{t}:\M(H^+,R^+)^{H^+}\rightarrow\R_+$, given by
\begin{align*}
\hat{F}_t(N_{1:H^+})&\defeq c_t\left(\hat{\y}^{\mathbf{nat}}_{t}+\sum_{i=1}^{\overline{H^+}}\hat{G}^{[i]}\sum_{j=0}^{\overline{H^+}}N_{H^+-i}^{[j]}\hat{\y}_{t-i-j}^{\mathbf{nat}}, \sum_{j=0}^{\overline{H^+}}N_{H^+}^{[j]}\hat{\y}_{t-j}^{\mathbf{nat}}\right).
\end{align*}
Define $\hat{f}_{t}:\M(H,R)\rightarrow\R_+$ to be the unary form induced by $\hat{F}_{t}$, given by $\hat{f}_{t}(N)\defeq \hat{F}_{t}\underbrace{(N,\dots,N)}_{N\text{ in all $H$ indices}}$. 
\end{definition}
Note that $\hat{F}_t(\widetilde{M}_{t-\overline{H^+}:t})=c_t(\y_t,\uv_t)$. Moreover, $\hat{F}_{t}$ is a $\mathcal{F}_{t-H}\cup\mathcal{G}_{t}$-measurable random function by Remark~\ref{rmk:ynat-dependence-unknown} that is independent of $\eps_{t-\overline{H^+}:t}$. In particular, Assumption~\ref{assumption:adversary} is satisfied. In addition to the with-history losses, we introduce a new pseudo loss function $\pF_t:\M(H^+,R^+)^{H^{+}}\rightarrow\R_+$ as the following. 

\begin{definition} [With-history pseudo losses for unknown system]
\label{def:pseudo-loss}
Given a partially observable linear dynamical system with Markov operator $G$ and an incidental cost function $c_t:\R^{d_{\y}}\times \R^{d_{\uv}}\rightarrow \R_+$ at time $t$. define its with-history pseudo loss at tiem $t$ to be $\pF_t:\M(H^+,R^+)^{H^+}\rightarrow\R_+$, given by
\begin{align*}
\pF_{t}(N_{1:H^+})\defeq c_t\left(\ynat_t+\sum_{i=1}^{\overline{H^+}}G^{[i]}\sum_{j=0}^{\overline{H^+}}N_{H^+-i}^{[j]}\hat{\y}_{t-i-j}^{\mathbf{nat}}, \sum_{j=0}^{\overline{H^+}}N_{H^+}^{[j]}\hat{\y}_{t-j}^{\mathbf{nat}}\right).
\end{align*}
Define $\pf_{t}:\M(H,R)\rightarrow\R_+$ to be the unary form induced by $\pF_{t}$, given by $\pf_{t}(N)\defeq \pF_{t}\underbrace{(N,\dots,N)}_{N\text{ in all $H$ indices}}$.
\end{definition}
While the learner has no access to $\pF_t$, it is useful for regret analysis: we will show that the gradient of $\pf_{t}$ is sufficiently close to the gradient of $\hat{f}_{t}$ and therefore the running bandit-RFTL-D on the loss functions $\{\hat{f}_{t}\}_{t=N+H^+}^T$ is nearly equivalent to running bandit RFTL-D with erroneous gradients on the loss functions $\{\pf_{t}\}_{t=N+H^+}^T$. 

We analyze how error in the computed gradient affects the final regret guarantee in Section~\ref{sec:rftl-d-error}. In Section~\ref{sec:pseudo-loss-regularity}, we prove the regularity conditions needed for both Section~\ref{sec:rftl-d-error} and the downstream regret analysis performed in Section~\ref{sec:unknown-system-regret}, which gives our desired final regret bound.

\subsection{RFTL-D with erroneous gradients}
\label{sec:rftl-d-error}

We establish a regret guarantee for RFTL-with-delay (RFTL-D) with erroneous gradient against loss functions that are conditionally strongly convex and satisfying other regularity conditions stated below in Assumption~\ref{condition-7.1} and \ref{condition-8.1}. The proof follows similarly to that in \cite{simchowitz2020improper}, where they proved a similar regret guarantee for Online Gradient Descent (OGD). In particular, we establish that when run with conditionally strongly convex loss functions, (1) the error in gradient propagates quadratically in the regret bound, and (2) the regret bound has a negative movement cost term. 

We begin with the working assumptions on the feasible set $\K\subset\R^d$ and the sequence of loss functions $\{f_t\}_{t=H}^T$ defined on $\K$.
\begin{assumption} [Conditional strong convexity]
\label{condition-7.1}
    Let $\{f_t\}_{t=H}^T$ be a sequence of loss functions mapping from $\K\to \R$. Letting $\mathcal{H}_t$ be the filtration generated by algorithm history up till time $t$ for all $t\ge H$, assume that $f_{t;H}(\cdot)\defeq \E[f_t(\cdot)\mid \mathcal{H}_{t-H}]$ is $\sigma$-strongly convex on $\K$.  
\end{assumption}

\begin{assumption} [Diameter]
\label{condition-8.1}
    Assume that $\mathrm{diam}(\K)=\sup_{z,z'\in\K}\|z-z'\|_2\le D$. Moreover, assume that $\{f_t\}_{t=H}^T$ obeys the range diameter bound $\sup_{z,z'\in\K}|f_t(z)-f_t(z')|\le B$.
\end{assumption}

\begin{assumption} [Gradient error]
\label{assumption:gradient-error}
Let $\{\delta_t\}_{t=H}^T$ denote the sequence of errors injected to the gradients. $\{\delta_t\}_{t=H}^T$ satisfies that for $\tilde{\nabla_t}\defeq \nabla f_t(z_t)+\delta_t$, where $z_t$ is the algorithm's decision at time $t$, $\|\tilde{\nabla}_t\|_{(t)}\le L_{\tilde{f}}$ for some norm $\|\cdot\|_{(t)}$ with dual $\|\cdot\|_{(t),*}$ possibly varying with $t$.
\end{assumption}



We consider RFTL-D run with erroneous gradients, outlined by Algorithm~\ref{alg:rftl-d-error}. 

\begin{algorithm}
\caption{RFTL-D with erroneous gradients}
\label{alg:rftl-d-error}
\begin{algorithmic}[1]
\STATE Input: feasible set $\K\subset\R^d$, time horizon $T$, history parameter $H$, strong convexity parameter $\sigma$, step size $\eta>0$, regularization function $R(\cdot):\K\rightarrow\R$. 
\STATE Initialize: $\tilde{\nabla}_1=\dots=\tilde{\nabla}_{\bar{H}}=0$, $z_1=\dots=z_H\in\K$. 
\FOR{$t=H,\dots,T$}
\STATE Play $z_t$, incur loss $f_t(z_t)$, receive gradient with error $\tilde{\nabla}_{t}=\nabla f_t(z_t)+\delta_t$. 
\STATE Update $z_{t+1}=\argmin_{z\in\K}\left(\sum_{s=H}^t \left(\tilde{\nabla}_{s-\bar{H}}^{\top}z+\frac{\sigma}{4}\|z-z_{s-\bar{H}}\|_2^2\right)+\frac{1}{\eta}R(z)\right)$. 
\ENDFOR
\end{algorithmic}
\end{algorithm}

\begin{lemma} [Conditional regret inequality for RFTL-D] 
Under Assumption~\ref{condition-7.1}, \ref{condition-8.1}, and \ref{assumption:gradient-error}, let $\Delta_t\defeq \nabla f_t(z_t)-\nabla f_{t;H}(z_t)$ denote the difference between the true gradient and the conditional gradient. Then we have that, $\forall z\in\K$, $\{z_t\}_{t=H}^T$ output by Algorithm~\ref{alg:rftl-d-error} satisfies the following regret inequality: 
\begin{align*}
\sum_{t=H}^{T} f_{t;H}(z_t)-f_{t;H}(z)&\le L_{\tilde{f}} \sum_{t=2\bar{H}+1}^{T}\|z_{t-\bar{H}}-z_{t+1}\|_{(t),*} +2HB+\frac{R(z)}{\eta}\\
& \ \ \ \ \ - \frac{\sigma}{4} \sum_{t=H}^{T-\bar{H}} \|z_t-z\|_2^2 - \sum_{t=H}^{T-\bar{H}} {(\delta_t+\Delta_t)}^{\top}(z_t-z). 
\end{align*}
\label{lem:G1}
\end{lemma}
\begin{proof}
 Define $h_t(z)\defeq \tilde{\nabla}_{t-\bar{H}}^{\top} z+\frac{\sigma}{4}\|z_{t-\bar{H}}-z\|_2^2$ for $t\ge 2\bar{H}+1$ and $\frac{1}{\eta}R(z)$ otherwise. By standard FTL-BTL lemma, $\sum_{t=2\bar{H}}^T h_t(z)\ge \sum_{t=2\bar{H}}^T h_t(z_{t+1})$, $\forall z\in\K$. Then
\begin{align*}
\sum_{t=H}^{T-\bar{H}} f_{t;H}(z_t)-f_{t;H}(z)&\le \sum_{t=H}^{T-\bar{H}} \nabla f_{t;H}(z_t)^{\top}(z_t-z)-\frac{\sigma}{2}\sum_{t=H}^{T-\bar{H}} \|z_t-z\|_2^2\\
&=\sum_{t=H}^{T-\bar{H}}\left( \tilde{\nabla}_t^{\top}(z_t-z)-\frac{\sigma}{2}\|z_t-z\|_2^2\right)-\sum_{t=H}^{T-\bar{H}}(\delta_t+\Delta_t)^{\top}(z_t-z)\\
&= \underbrace{\sum_{t=2\bar{H}+1}^{T} h_t(z_{t-\bar{H}})-h_t(z)}_{(*)}-\sum_{t=H}^{T-\bar{H}} (\delta_t+\Delta_t)^{\top}(z_t-z)-\frac{\sigma}{4}\sum_{t=H}^{T-\bar{H}} \|z_t-z\|_2^2.
\end{align*}
Applying the FTL-BTL lemma, the first part on the right hand side is bounded by
\begin{align*}
(*)&\le\left(\sum_{t=2\bar{H}+1}^T h_t(z_{t-\bar{H}})-h_t(z_{t+1})\right)+\left(\sum_{t=\bar{H}}^{2\bar{H}}h_t(z_{t+1})-h_t(z)\right)\\
&\le \sum_{t=2\bar{H}+1}^T \underbrace{\nabla h_t(z_{t-\bar{H}})^{\top}}_{\tilde{\nabla}_{t-\bar{H}} } (z_{t-\bar{H}}-z_{t+1})+HB+\frac{R(z)}{\eta}\\
&\le L_{\tilde{f}} \sum_{t=2\bar{H}+1}^T \|z_{t-\bar{H}}-z_{t+1}\|_{(t),*} + HB +\frac{R(z)}{\eta}.
\end{align*}
Combining, 
\begin{align*}
\sum_{t=H}^{T} f_{t;H}(z_t)-f_{t;H}(z)&\le L_{\tilde{f}} \sum_{t=2\bar{H}+1}^{T}\|z_{t-\bar{H}}-z_{t+1}\|_{(t),*} +2HB+\frac{R(z)}{\eta}\\
& \ \ \ \ \ - \frac{\sigma}{4} \sum_{t=H}^{T-\bar{H}} \|z_t-z\|_2^2 - \sum_{t=H}^{T-\bar{H}} {(\delta_t+\Delta_t)}^{\top}(z_t-z). 
\end{align*}

\end{proof}

\begin{lemma} [Regret inequality for bandit RFTL-D] 
\label{lem:bandit-rftl-d-error}
Suppose RFTL-D is run with gradient estimators $g_t$ such that $g_t$ satisfies $\|\E[g_t\mid\mathcal{G}_{t}]-\E[\tilde{\nabla}_{t}\mid\mathcal{G}_t]\|_2\le B(t)$, where $\mathcal{G}_{t}$ is any filtration such that $z_t\in\mathcal{G}_{t}$, then $\forall z\in\K$, 
\begin{align*}
\E\left[\sum_{t=2\bar{H}+1}^{T}f_t(z_t)-f_t(z)\ignore{\bigg  |\mathcal{G}_{t-H}}\right]&\le
L_g \E\left[\sum_{t=2\bar{H}+1}^T\|z_{t-\bar{H}}-z_{t+1}\|_{(t),*}\right]+3HB+\frac{R(z)}{\eta}\\
& \ \ \ \ \  -\frac{\sigma}{6} \E\left[\sum_{t=H}^{T-\bar{H}}\|z_t-z\|_2^2\right]+\frac{3}{\sigma} \E\left[\sum_{t=H}^{T-\bar{H}}\|\delta_t\|_2^2\right]+2D\sum_{t=H}^{T-\bar{H}} B(t),
\end{align*}
where $L_g=\sup_{t}\|g_t\|_{(t)}$. 
\end{lemma}
\begin{proof}
Define $q_t(z)\defeq f_t(z)+(g_t-\tilde{\nabla}_t)^{\top}z+\delta_t$, $\forall t\ge H$, and note that $\nabla q_t(z_t)=g_t$ by construction. Then since RFTL-D is a first-order OCO algorithm, we have $\text{RFTL-D}(q_H,\dots,q_{t-1})=\text{RFTL-D}(g_H,\dots,g_{t-1})$, $\forall t$. Moreover, by Lemma 6.3.1 in~\cite{hazan2016introduction}, $\forall z\in\K$, we have
\begin{align*}
\sum_{t=2\bar{H}+1}^T q_t(z_t)-q_t(z)&\le \regret_T^{\text{RFTL-D}} (g_H,\dots,g_T)+\sum_{t=H}^{2\bar{H}}q_t(z)-q_t(z_t)\\
&\le \regret_T^{\text{RFTL-D}} (g_H,\dots,g_T)+\sum_{t=H}^{2\bar{H}}f_t(z)-f_t(z_t)\\
&\le \regret_T^{\text{RFTL-D}} (g_H,\dots,g_T)+HB,
\end{align*}
where the second inequality follows from $\forall H\le t\le 2\bar{H}$, $g_t=\tilde{\nabla}_{t}=0$ and thus $q_t(z)=f_t(z)+\delta_t$, $\forall z$. Additionally,
$\forall t\ge 2\bar{H}+1,z\in\K$,
\begin{align*}
\E[q_t(z_t)-q_t(z)]&=\E[f_t(z_t)-f_t(z)]-\E[(g_t-\tilde{\nabla}_t)^{\top}(z_t-z)]\\
&=\E[f_t(z_t)-f_t(z)]-\E[\E[(g_t-\tilde{\nabla}_t)^{\top}(z_t-z)\mid \mathcal{G}_t]]\\
&=\E[f_t(z_t)-f_t(z)]-\E[(\E[g_t\mid \mathcal{G}_t]-\E[\tilde{\nabla}_t\mid\mathcal{G}_t])^{\top}(z_t-z)]\\
&\ge \E[f_t(z_t)-f_t(z)]-DB(t). 
\end{align*}
Moreover, by Lemma~\ref{lem:G1}, $\forall z\in\K$, since
\begin{align*}
\E[{\Delta_t}^{\top}(z_t-z)]=\E[\E[{\Delta_t}^{\top}(z_t-z)\mid \mathcal{H}_{t-H}]]=\E[\E[{\Delta_t}\mid \mathcal{H}_{t-H}]^{\top}(z_t-z)]=0,
\end{align*}
the expected regret is bounded by
\begin{align*}
\E\left[\regret_T^{\text{RFTL-D}} (g_1,\dots,g_T)\right]&\le L_g \E\left[\sum_{t=2\bar{H}+1}^T\|z_{t-\bar{H}}-z_{t+1}\|_{(t),*}\right]+2HB+\frac{R(z)}{\eta}\\
& \ \ \ \ \  -\frac{\sigma}{4} \E\left[\sum_{t=H}^{T-\bar{H}}\|z_t-z\|_2^2\right]-\underbrace{\E\left[\sum_{t=H}^{T-\bar{H}}(g_t-\nabla f_t(z_t)))^{\top}(z_t-z)\right]}_{(*)},
\end{align*} 
where we can further decouple $(*)$ as 
\begin{align*}
(*)&=\underbrace{\E\left[\sum_{t=H}^{T-\bar{H}}(g_t-\tilde{\nabla}_{t})^{\top}(z_t-z)\right]}_{(1)}+\underbrace{\E\left[\sum_{t=H}^{T-\bar{H}}\delta_t^{\top}(z_t-z)\right]}_{(2)},\\
(1)&=\sum_{t=H}^{T-\bar{H}}(\E[g_t\mid\mathcal{G}_{t}]-\tilde{\nabla}_{t})^{\top}(z_t-z)\le D\sum_{t=H}^{T-\bar{H}} B(t), \ \ \  (2)\le \E\left[\sum_{t=H}^{T-\bar{H}}\frac{3}{\sigma}\|\delta_t\|_2^2+\frac{\sigma}{12}\|z_t-z\|_2^2\right].
\end{align*}
Combining, $\forall z\in\K$, 
\begin{align*}
\E\left[\sum_{t=2\bar{H}+1}^Tf_t(z_t)-f_t(z)\right]&\le
L_g \E\left[\sum_{t=2\bar{H}+1}^T\|z_{t-\bar{H}}-z_{t+1}\|_{(t),*}\right]+3HB+\frac{R(z)}{\eta}\\
& \ \ \ \ \  -\frac{\sigma}{6} \E\left[\sum_{t=H}^{T-\bar{H}}\|z_t-z\|_2^2\right]+\frac{3}{\sigma} \E\left[\sum_{t=H}^{T-\bar{H}}\|\delta_t\|_2^2\right]+2D\sum_{t=H}^{T-\bar{H}} B(t). 
\end{align*}
\end{proof}

\subsection{Regularity conditions for estimated with-history loss functions and iterates}
\label{sec:pseudo-loss-regularity}
This section is analogous to Section~\ref{sec:control-loss-regularity}, and establishes regularity conditions for $\hat{F}_{t},\pF_t,\K=\M(H^+,R^+)$. 

The following table summarizes the results in this section. 
\begin{center}
\begin{tabular}{ |c|c|c| } 
 \hline\
 Parameter & Definition & Magnitude \\ 
 \hline
 $\hat{R}_{\mathrm{nat}}$ & $\ell_2$ bound on the signals & $2R_{\mathrm{nat}}$ \\
 \hline
 $R_{\hat{\y}}$ & $\ell_2$ bound on observations & $2R_{\mathrm{nat}}+4R_G\max\{R_{\uv,\delta},RR_{\mathrm{nat}}\}$ \\ 
 \hline
 $R_{\hat{\uv}}$ & $\ell_2$ bound on controls based on $\M(H^+,R^+)$ & $2\max\{R_{\uv,\delta}, RR_{\mathrm{nat}}\}$ \\
 \hline
 $\hat{B}$ & diameter bound on $c_t$ & $4L_c((R_{\mathrm{nat}}^2+3R_G\max\{R_{\uv,\delta},RR_{\mathrm{nat}}\})^2)$ \\
 \hline
 $\hat{D}$ & diameter bound on $\M(H,R)$ & $2\sqrt{d_{\uv}\wedge d_{\y}}R^+$\\
 \hline
 $\sigma_{\pf}$ & conditional strong convexity parameter of $\pf_{t}$ & 
 $\frac{\sigma_c}{4}\left(\sigma_{\e}^2+\sigma_{\w}^2\frac{\sigma_{\min}(C)}{1+\|A\|_{\mathrm{op}}^2}\right)$
 \\
 \hline
 $\sigma_{\hat{f}}$ & conditional strong convexity parameter of $\hat{f}_{t}$ & 
 $\frac{\sigma_c}{4}\left(\sigma_{\e}^2+\sigma_{\w}^2\frac{\sigma_{\min}(C)}{1+\|A\|_{\mathrm{op}}^2}\right)$
 \\
 \hline
 $\beta_{\pF}$ & smoothness parameter of $\pF_t$ & $16\beta_c R_G^2R_{\mathrm{nat}}^2H^+$\\
 \hline
 $\beta_{\hat{F}}$ & smoothness parameter of $\hat{F}_{t}$ & $64\beta_c R_G^2R_{\mathrm{nat}}^2H^+$\\
 \hline
 $L_{\pF}$ & Lipschitz parameter of $\pF_t$ & $4L_c\sqrt{R_{\hat{\y}}^2+R_{\hat{\uv}}^2} R_{\mathrm{nat}}R_G\sqrt{H^+}$\\
 \hline
 $L_{\hat{F}}$ & Lipschitz parameter of $\hat{F}_{t}$ & $8L_c\sqrt{R_{\hat{\y}}^2+R_{\hat{\uv}}^2} R_{\mathrm{nat}}R_G\sqrt{H^+}$\\
 \hline
\end{tabular}
\end{center}

We start with proving $\ell_2$ bounds on the observations and controls. 
\begin{lemma} [Control, signal, and observation norm bounds for unknown systems]
\label{lem:yyu-norm-bound-unknown}
Under Assumption~\ref{assumption:est-error} on the obtained estimator $\hat{G}$ for the Markov operator and suppose the bandit controller outlined in Algorithm~\ref{alg:BCO-control} is run with $\hat{G}$. Denote $\hat{R}_{\mathrm{nat}}:=\sup_t\|\ynat_t\|_2$,  $R_{\hat{\y}}:=\sup_t\|\y_t\|_2$ and $R_{\hat{\uv}}:=\sup_t\|\uv_t\|_2$, where $(\y_t,\uv_t)$ are the observation-control pair resulted by executing the bandit controller, then the following bounds hold deterministically:
\begin{align*}
\hat{R}_{\mathrm{nat}}\le 2R_{\mathrm{nat}}, \ \ \ R_{\hat{\y}}\le 2R_{\mathrm{nat}}+4R_G\max\{R_{\uv,\delta}, RR_{\mathrm{nat}}\}, \ \ \  R_{\hat{\uv}}\le 2\max\{R_{\uv,\delta}, RR_{\mathrm{nat}}\} , 
\end{align*}
\end{lemma}
\begin{proof}
By Assumption~\ref{assumption:est-error}, $\|\hat{G}\|_{\ell_1,\mathrm{op}}\le 2R_G$ and $\forall t$, 
\begin{align*}
\max_{s\le t}\|\uv_s\|_{2} &\le \max\left\{R_{\uv,\delta}, \max_{s\le t}\left\|\sum_{i=0}^{\overline{H^+}}\widetilde{M}_s^{[i]}\hat{\y}_{s-i}^{\mathbf{nat}}\right\|_2\right\}\\
&\le \max\left\{R_{\uv,\delta}, R\max_{s\le t}\max_{0\le i\le\overline{H^+}} \|\hat{\y}_{s-i}^{\mathbf{nat}}\|_2\right\}\\
&\le \max\left\{R_{\uv,\delta}, R\left(R_{\mathrm{nat}}+\max_{s\le t}\max_{0\le i\le\overline{H^+}} \|\hat{\y}_{s-i}^{\mathbf{nat}}-\ynat_{s-i}\|_2\right)\right\}\\
&\le \max\left\{R_{\uv,\delta}, R\left(R_{\mathrm{nat}}+\max_{s\le t}\max_{0\le i\le\overline{H^+}} \left\|\sum_{j=1}^{s-i}(G^{[i]}-\hat{G}^{[i]})\uv_{s-i-j}\right\|_2\right)\right\}\\
&\le \max\left\{R_{\uv,\delta}, R\left(R_{\mathrm{nat}}+\eps_G\max_{s\le t-1}\|\uv_{s}\|_{2}\right)\right\}\\
&\le \max\{R_{\uv,\delta},RR_{\mathrm{nat}}\}+\frac{\max_{s\le t}\|\uv_{s}\|_{2}}{2},
\end{align*}
where the last inequality follows from $\eps_G\le \frac{1}{2\max\{RR_G,R_{\uv,\delta}\}}$ in Assumption~\ref{assumption:est-error}. The above inequality implies $R_{\hat{\uv}}\le 2\max\{R_{\uv,\delta}, RR_{\mathrm{nat}}\}$. Immediately, $\forall t$,
\begin{align*}
\|\hat{\y}_{t}^{\mathbf{nat}}\|_2&\le R_{\mathrm{nat}}+\|\hat{\y}_{t}^{\mathbf{nat}}-\ynat_t\|_2= R_{\mathrm{nat}}+\left\|\sum_{j=1}^{t}(G^{[i]}-\hat{G}^{[i]})\uv_{t-j}\right\|_2\le R_{\mathrm{nat}}+\eps_GR_{\hat{\uv}}\le 2R_{\mathrm{nat}},\\
\|\y_t\|_2&=\left\|\hat{\y}_{t}^{\mathbf{nat}}+\sum_{i=1}^t \hat{G}^{[i]}\uv_{t-i}\right\|_2\le 2R_{\mathrm{nat}}+2R_GR_{\hat{\uv}}\le  2R_{\mathrm{nat}}+4R_G\max\{R_{\uv,\delta}, RR_{\mathrm{nat}}\}. 
\end{align*}
\end{proof}

\begin{lemma} [Diameter bounds]
\label{lem:diameter-bounds-unknown}
Consider the following sets
\begin{align*}
\hat{\mathcal{U}}&\defeq \left\{\sum_{j=0}^{\overline{H^+}}M^{[j]}\zeta_j: M\in\M(H^+,R^+), \zeta_j\in\R^{d_{\y}}, \|\zeta_j\|_2\le \hat{R}_{\mathrm{nat}}\right\},\\
\hat{\Y}&\defeq \left\{\zeta+\sum_{i=1}^{T-1}G^{[i]}\xi_{i}: \zeta\in\R^{d_{\y}}, \|\zeta\|_2\le \hat{R}_{\mathrm{nat}},\xi_i\in\hat{\mathcal{U}}\right\},
\end{align*}
and $\hat{B}:=\underset{t}{\sup}\underset{\y\in\hat{\Y},\uv\in\hat{\mathcal{U}}}{\sup} c_t(\y,\uv)$. Let $\hat{D}:=\underset{M,M'\in\M(H^+,R^+)}{\sup}\|M-M'\|_F$. Then,
\begin{align*}
\hat{B}\le 4L_c((R_{\mathrm{nat}}^2+3R_G\max\{R_{\uv,\delta},RR_{\mathrm{nat}}\})^2), \ \ \ \hat{D}\le 2\sqrt{d_{\uv}\wedge d_{\y}}R^+.
\end{align*}
\end{lemma}
\begin{proof}
First, we calculate the bound on $\hat{D}$. $\forall M,M'\in\M(H^+,R^+)$,
\begin{align*}
\|M-M'\|_F\le \sum_{j=0}^{\overline{H^+}} \|M^{[j]}-M'^{[j]}\|_F\le \sqrt{d_{\uv}\wedge d_{\y}} \sum_{j=0}^{\overline{H^+}} \|M^{[j]}-M'^{[j]}\|_{\mathrm{op}}\le 2\sqrt{d_{\uv}\wedge d_{\y}}R^+. 
\end{align*}
To see the bound on $\hat{B}$, note that $\forall t, \forall \y\in\hat{\Y},\uv\in\hat{\mathcal{U}}$, by the quadratic and Lipschitz condition on $c_t$,
\begin{align*}
c_t(\y,\uv)\le L_c(\|\y\|_2^2+\|\uv\|_2^2)\le 4L_c((R_{\mathrm{nat}}^2+3R_G\max\{R_{\uv,\delta},RR_{\mathrm{nat}}\})^2).
\end{align*}
\end{proof}

\begin{lemma} [Regularity conditions for $\pF_t, \pf_t$ and $\hat{F}_{t},\hat{f}_{t}$] 
\label{lem:regularity-unknown}
$\pF_t, \pf_t$ and $\hat{F}_{t},\hat{f}_{t}$ follow the following regularity conditions under the assumption that $\eps_G\le \frac{1}{4R_GR_{\hat{\uv}}\sqrt{H^+}}\sqrt{\frac{\sigma_f}{\sigma_c}}$,
\begin{itemize}
\item $\pF_t$ is $L_{\pF}$-Lipschitz with $L_{\pF}= L_c\sqrt{R_{\hat{\y}}^2+R_{\hat{\uv}}^2} (4R_{\mathrm{nat}}R_G\sqrt{H^+})$; $\hat{F}_{t}$ is $L_{\hat{F}}$-Lipschitz with $L_{\hat{F}}= L_c\sqrt{R_{\hat{\y}}^2+R_{\hat{\uv}}^2} (8R_{\mathrm{nat}}R_G\sqrt{H^+})$. 
\item $\pF_t$ is $\beta_{\pF}$-smooth with $\beta_{\pF}=16\beta_c R_G^2R_{\mathrm{nat}}^2H^+$; $\hat{F}_{t}$ is $\beta_{\hat{F}}$-smooth with $\beta_{\hat{F}}=64\beta_c R_G^2R_{\mathrm{nat}}^2H^+$.
\item $\pf_t,\hat{f}_{t}$ are $\sigma_{\pf},\sigma_{\hat{f}}$-conditionally strongly convex with $\sigma_{\pf}=\sigma_{\hat{f}}=\frac{\sigma_f}{4}$.
\end{itemize}
\end{lemma}
\begin{proof}
Consider the following quantities:
\begin{align*}
\pS_{N_{1:H^+}}&\defeq \ynat_t+\sum_{i=1}^{\overline{H^+}} G^{[i]}\sum_{j=0}^{\overline{H^+}}N_{H^+-i}^{[j]}\hat{\y}_{t-i-j}^{\mathbf{nat}}, \ \ \ \pC_{N_{H^+}}=\sum_{j=0}^{\overline{H^+}} N_{H^+}^{[j]}\hat{\y}_{t-j}^{\mathbf{nat}},\\
\hat{\mathbf{S}}_{N_{1:H^+}}&\defeq \hat{\y}^{\mathbf{nat}}_{t}+\sum_{i=1}^{\overline{H^+}} \hat{G}^{[i]}\sum_{j=0}^{\overline{H^+}}N_{H^+-i}^{[j]}\hat{\y}_{t-i-j}^{\mathbf{nat}}, \ \ \ \hat{\mathbf{C}}_{N_{H^+}}=\pC_{N_{H^+}}.
\end{align*}
Consider the linear operator $\mathring{v},\hat{v}:\M(H^+,R)^{H^+}\rightarrow \R^{d_{\y}}\times \R^{d_{\uv}}$ given by $\mathring{v}(N_{1:H^+})=(\pS_{N_{1:H^+}},\pC_{N_{H^+}})$, $\hat{v}(N_{1:H^+})=(\hat{\mathbf{S}}_{N_{1:H^+}},\hat{\mathbf{C}}_{N_{H^+}})$. Similar to the analysis in Section~\ref{sec:control-loss-regularity}, 
$\forall N_{1:H^+},N_{1:H^+}'\in\M(H^+,R^+)^{H^+}$, 
\begin{align*}
\|\mathring{v}(N_{1:H^+})-\mathring{v}(N_{1:H^+}')\|_2&=\left\|\left(\sum_{i=1}^{\overline{H^+}}G^{[i]}\sum_{j=0}^{\bar{H}}(N_{H^+-i}^{[j]}-N_{H^+-i}'^{[j]})\hat{\y}_{t-i-j}^{\mathbf{nat}}, \sum_{j=0}^{\overline{H^+}}(N_{H^+}^{[j]}-N_{H^+}'^{[j]})\hat{\y}_{t-j}^{\mathbf{nat}}\right)\right\|_2\\
&\le \left\|\sum_{i=1}^{\overline{H^+}}G^{[i]}\sum_{j=0}^{\overline{H^+}}(N_{H^+-i}^{[j]}-N_{H^+-i}'^{[j]})\hat{\y}_{t-i-j}^{\mathbf{nat}}\right\|_2+\left\|\sum_{j=0}^{\overline{H^+}}(N_{H^+}^{[j]}-N_{H^+}'^{[j]})\hat{\y}_{t-j}^{\mathbf{nat}}\right\|_2\\
&\le \sum_{i=1}^{\overline{H^+}}\|G^{[i]}\|_{\mathrm{op}}\left\|\sum_{j=0}^{\overline{H^+}}(N_{H^+-i}^{[j]}-N_{H^+-i}'^{[j]})\hat{\y}_{t-i-j}^{\mathbf{nat}}\right\|_2+\left\|\sum_{j=0}^{\overline{H^+}}(N_{H^+}^{[j]}-N_{H^+}'^{[j]})\hat{\y}_{t-j}^{\mathbf{nat}}\right\|_2\\
&\le R_G\max_{1\le i\le\overline{H^+}} \left\|\sum_{j=0}^{\overline{H^+}}(N_i^{[j]}-N_i'^{[j]})\hat{\y}_{t-i-j}^{\mathbf{nat}}\right\|_2+\left\|\sum_{j=0}^{\overline{H^+}}(N_{H^+}^{[j]}-N_{H^+}'^{[j]})\hat{\y}_{t-j}^{\mathbf{nat}}\right\|_2\\
&\le 4R_GR_{\mathrm{nat}}\max_{1\le i\le H^+} \|N_i-N_i'\|_{\ell_1,\mathrm{op}}\\
&\le 4R_GR_{\mathrm{nat}}\sqrt{H^+}\|N_{1:H^+}-N_{1:H^+}'\|_F,
\end{align*}
which bounds $\|\mathbf{D}\mathring{v}(N_{1:H^+})\|_2\le 4R_GR_{\mathrm{nat}}\sqrt{H^+}$. Similarly, we can bound $\|\mathbf{D}\hat{v}(N_{1:H^+})\|_2\le 8R_GR_{\mathrm{nat}}\sqrt{H^+}$, $\forall N_{1:H^+}\in\M(H^+,R^+)^{H^+}$. 
The gradient bounds $L_{\pF}, L_{\hat{F}}$ are thus given by
\begin{align*}
\|\nabla \pF_{t}(N_1,\dots,N_{H^+})\|_2&= \|(\nabla c_t)(\mathring{v}(N_{1:H^+}))\|_2 \|\mathbf{D}\mathring{v}(N_{1:H^+})\|_2\le L_c\sqrt{R_{\hat{\y}}^2+R_{\hat{\uv}}^2} (4R_{\mathrm{nat}}R_G\sqrt{H^+}),\\
\|\nabla \hat{F}_{t}(N_1,\dots,N_{H^+})\|_2&= \|(\nabla c_t)(\hat{v}(N_{1:H^+}))\|_2 \|\mathbf{D}\hat{v}(N_{1:H^+})\|_2\le L_c\sqrt{R_{\hat{\y}}^2+R_{\hat{\uv}}^2} (8R_{\mathrm{nat}}R_G\sqrt{H^+}). 
\end{align*}
The smoothness parameters $\beta_{\pF}, \beta_{\hat{F}}$ is given by
\begin{align*}
\|\nabla^2 c_t(\mathring{v}(N_{1:H^+}))\|_{\mathrm{op}}&=\|\mathbf{D}\mathring{v}(N_{1:H^+})(\nabla^2 c_t)(\mathring{v}(N_{1:H^+}))\mathbf{D}\mathring{v}(N_{1:H^+})^{\top}\|_{\mathrm{op}}\le 16\beta_c R_G^2R_{\mathrm{nat}}^2H^+,\\
\|\nabla^2 c_t(\hat{v}(N_{1:H^+}))\|_{\mathrm{op}}&=\|\mathbf{D}\hat{v}(N_{1:H^+})(\nabla^2 c_t)(\hat{v}(N_{1:H^+}))\mathbf{D}\hat{v}(N_{1:H^+})^{\top}\|_{\mathrm{op}}\le 64\beta_c R_G^2R_{\mathrm{nat}}^2H^+.
\end{align*}
To bound the conditional strong convexity parameters $\sigma_{\pf},\sigma_{\hat{f}}$, it suffices to show an analogue to Lemma~\ref{lem:strong-convexity-lemma} that $\forall M\in\M(H^+,R^+)$,
\begin{align*}
\underbrace{\E\left[\|(\pS_{M},\pC_{M})-(\ynat_t,\mathbf{0}_{d_{\uv}})\|_2^2\mid \mathcal{F}_{t-H}\cup\mathcal{G}_{t-H}\right]}_{(1)}&\ge \frac{\sigma_{\pf}}{\sigma_c}\|M\|_F^2,\\
\underbrace{\E\left[\|(\hat{\mathbf{S}}_{M},\hat{\mathbf{C}}_{M})-(\ynat_t,\mathbf{0}_{d_{\uv}})\|_2^2\mid \mathcal{F}_{t-H}\cup\mathcal{G}_{t-H}\right]}_{(2)}&\ge \frac{\sigma_{\hat{f}}}{\sigma_c}\|M\|_F^2. 
\end{align*}

\ignore{
\begin{align*}
\underbrace{\E\left[\left\|\left((\hat{\y}_{t}^{\mathbf{nat}}-\ynat_t)+\sum_{i=1}^{\overline{H^+}}\hat{G}^{[i]}\sum_{j=0}^{\overline{H^+}} M^{[j]}\hat{\y}_{t-i-j}^{\mathbf{nat}},\sum_{j=0}^{\overline{H^+}}M^{[j]}\hat{\y}_{t-j}^{\mathbf{nat}}\right)\right\|_2^2\bigg | \mathcal{F}_{t-H}\cup\mathcal{G}_{t-H}\right]}_{(2)}\ge \frac{\sigma_{\hat{f}}}{\sigma_c}\|M\|_F^2.
\end{align*}
}

As $(a-b)^2\ge \frac{1}{2}a^2-b^2$, $\forall a,b\in\R$,
\begin{align*}
(1)&=\E\left[\left\|\left(\sum_{i=1}^{\overline{H^+}}G^{[i]}\sum_{j=0}^{\overline{H^+}} M^{[j]}\hat{\y}_{t-i-j}^{\mathbf{nat}},\sum_{j=0}^{\overline{H^+}}M^{[j]}\hat{\y}_{t-j}^{\mathbf{nat}}\right)\right\|_2^2\bigg | \mathcal{F}_{t-H}\cup\mathcal{G}_{t-H}\right]\\
&\ge -\E\left[\left\|\left(\sum_{i=1}^{\overline{H^+}}G^{[i]}\sum_{j=0}^{\overline{H^+}} M^{[j]}(\hat{\y}_{t-i-j}^{\mathbf{nat}}-\ynat_{t-i-j}),\sum_{j=0}^{\overline{H^+}}M^{[j]}(\hat{\y}_{t-j}^{\mathbf{nat}}-\ynat_{t-j})\right)\right\|_2^2\bigg | \mathcal{F}_{t-H}\cup\mathcal{G}_{t-H}\right]\\
& \ \ \ \ \ +
\frac{1}{2}\E\left[\left\|\left(\sum_{i=1}^{\overline{H^+}}G^{[i]}\sum_{j=0}^{\overline{H^+}} M^{[j]}\ynat_{t-i-j},\sum_{j=0}^{\overline{H^+}}M^{[j]}\ynat_{t-j}\right)\right\|_2^2\bigg | \mathcal{F}_{t-H}\cup\mathcal{G}_{t-H}\right]\\
&\ge \left(\frac{\sigma_f}{2\sigma_c}-R_G^2\eps_G^2R_{\hat{\uv}}^2H^+\right)\|M\|_F^2\\
&\ge \frac{\sigma_f}{4\sigma_c}\|M\|_F^2,
\end{align*}
and similarly $(2)\ge \left(\frac{\sigma_f}{2\sigma_c}-4R_G^2\eps_G^2R_{\uv}^2H^+\right)\ge \frac{\sigma_f}{4\sigma_c}\|M\|_F^2$. 
\end{proof}

\subsection{Unknown system regret analysis}
\label{sec:unknown-system-regret}
Before the decomposition of regret, we introduce a result from \cite{simchowitz2020improper}. Define $\varphi:\M(H,R)\rightarrow \R^{d_{\y}\times d_{\uv}\times H^+}$ such that $\varphi(M)=(M,\mathbf{0}_{d_{\y}\times d_{\uv}},\dots,\mathbf{0}_{d_{\y}\times d_{\uv}})$. Note that $\varphi(\M(H,R))\subset \M(H^+,R^+)$. 
\begin{proposition} [Proposition F.8 in~\cite{simchowitz2020improper}]
\label{prop:max-approx-error}
$\exists M_0\in\M(H,R)$ such that $\forall\tau>0$, 
\begin{align*}
&\sum_{t=N+2\overline{H^+}+1}^T \pf_{t}(\varphi(M_0))-\inf_{M\in\M(H,R)} \sum_{t=N+2\overline{H^+}+1}^T f_t(M\mid G,\ynat_{1:t})\\
&\le 36(H^+)^2R_G^4R_{\mathrm{nat}}^4(R^+)^3(H^++T\eps_G^2)\left(L_c\vee\frac{L_c^2}{\tau}\right)+\tau\sum_{t=N+2\overline{H^+}+1}^T \|M_t-\varphi(M_0)\|_F^2, 
\end{align*}
where $f_t(\cdot\mid G,\ynat_{1:t}):\M(H,R)\rightarrow\R_+$ is given by 
\begin{align*}
f_t(M\mid G,\ynat_{1:t})\defeq c_t\left(\ynat_t+\sum_{i=1}^{H}G^{[i]}\sum_{j=0}^{\bar{H}}M^{[j]}\ynat_{t-i-j}, \sum_{j=0}^{\bar{H}}M^{[j]}\ynat_{t-j}\right). 
\end{align*}
\end{proposition}

Let $M_0\in\M(H,R)$ satisfy the inequality in Proposition~\ref{prop:max-approx-error}, and consider the decomposition of regret into four parts, which we will proceed to bound each separately. 
\begin{align*}
\regret_T&=\underbrace{\left(\sum_{t=1}^{N+2\bar{H}}c_t(\y_t,\uv_t)\right)}_{\text{(burn-in loss)}}+\underbrace{\left(\sum_{t=N+2\overline{H^+}+1}^{T}c_t(\y_t,\uv_t)-\pF_{t}(\widetilde{M}_{t-\overline{H^+}:t})\right)}_{(\text{algorithm estimation loss})}\\
& \ \ \ \ \ +\underbrace{\left(\sum_{t=N+2\overline{H^+}+1}^{T}\pF_{t}(\widetilde{M}_{t-\overline{H^+}:t})-\pf_{t}(\varphi(M_0))\right)}_{(\pf_{t}\text{-BCO-M-regret})}\\
& \ \ \ \ \ + \underbrace{\left(\sum_{t=N+2\overline{H^+}+1}^{T}\pf_{t}(\varphi(M_0))-\inf_{M\in\M(H,R)}\sum_{t=N+2\overline{H^+}+1}^{T} c_t(\y_t^M,\uv_t^M)\right)}_{(\text{comparator estimation loss})}. 
\end{align*}

The choice of $M_0$ and Proposition~\ref{prop:max-approx-error} directly allows us to bound the comparator estimation loss. In particular, note that
\begin{align*}
\inf_{M\in\M(H,R)} \sum_{t=N+2\overline{H^+}+1}^T &f_t(M\mid G,\ynat_{1:t})-\inf_{M\in\M(H,R)} \sum_{t=N+2\overline{H^+}+1}^T c_t(\y_t^M,\uv_t^M)\\
\le \sup_{M\in\M(H,R)} \sum_{t=N+2\overline{H^+}+1}^T  &f_t(M\mid G,\ynat_{1:t}) - c_t(\y_t^M,\uv_t^M)\\
\le \sup_{M\in\M(H,R)} \sum_{t=N+2\overline{H^+}+1}^T   &c_t\left(\ynat_t+\sum_{i=1}^{H}G^{[i]}\sum_{j=0}^{\bar{H}}M^{[j]}\ynat_{t-i-j}, \sum_{j=0}^{\bar{H}}M^{[j]}\ynat_{t-j}\right) - \\
&c_t\left(\ynat_t+\sum_{i=1}^{t}G^{[i]}\sum_{j=0}^{\bar{H}}M^{[j]}\ynat_{t-i-j}, \sum_{j=0}^{\bar{H}}M^{[j]}\ynat_{t-j}\right)\\
\le L_c \sqrt{(1+R_GR)^2+R^2}&R_{\mathrm{nat}}^2R\psi_G(H)T\le \mathcal{O}(1).
\end{align*}
Therefore, combining terms and taking $H=\mathrm{poly}(\log T)$, we have that for some constants $C_{\mathrm{param}}^0,C_{\mathrm{param}}^1$ depending on the natural parameters and universal constants $C$, $\forall \tau>0$,
\begin{align*}
(\text{comparator estimation loss})-\tau\sum_{t=N+2\overline{H^+}+1}^T\|M_t-M_0\|_F^2&\le \tilde{\mathcal{O}}(1)\frac{\eps_GT}{\tau}+\tilde{\mathcal{O}}(1)\\
&\le \frac{1}{\tau}\tilde{\mathcal{O}}(\sqrt{T})+\tilde{\mathcal{O}}(1),
\end{align*}
where the last inequality comes from taking $N=\lceil \sqrt{T}\rceil$, $\delta=\frac{1}{\sqrt{T}}$  as in Assumption~\ref{assumption:est-error}, and $\eps_G\asymp\frac{1}{\sqrt{N}}H^2R_{\mathrm{nat}}\sqrt{(d_{\y}\vee d_{\uv})+\log\frac{1}{\delta}+\log(1+R_{\mathrm{nat}})}$ as in Proposition~\ref{prop:max-approx-error}. 

Then, we proceed to bound the burn-in loss and the algorithm estimation loss. The burn-in loss can be crudely bounded by the diameter bound on $c_t$ established in Section~\ref{sec:pseudo-loss-regularity}. Take $N=\lceil \sqrt{T}\rceil$,
\begin{align*}
(\text{burn-in loss})\le (N+2\overline{H^+})\hat{B}\le \mathcal{O}(\sqrt{T})+\tilde{\mathcal{O}}(1).
\end{align*}
The algorithm estimation loss can be bounded as follows:
\begin{align*}
(\text{algorithm estimation loss})&\le L_c\sqrt{R_{\hat{\y}}^2+R_{\hat{\uv}}^2} \sum_{t=N+2\overline{H^+}+1}^T\left\|\sum_{i=H}^tG^{[i]}\sum_{j=0}^{\overline{H^+}} \widetilde{M}_{t-i}^{[j]}\hat{\y}_{t-i-j}^{\mathbf{nat}}\right\|_2\\
&\le L_c\sqrt{R_{\hat{\y}}^2+R_{\hat{\uv}}^2} R_{\hat{\uv}}\psi_G(H)T\\
&\le \tilde{\mathcal{O}}(1). 
\end{align*}
It is left to bound the $\pf_t$-BCO-M regret term, which is given by the following lemma:
\begin{proposition}
\label{lem:estimated-bco-m-regret}
The BCO-M regret against the estimated with-history unary functions $\hat{f}_{t}$ has the following bound in expectation:
\begin{align*}
\E[(\pf_{t}\text{-BCO-M-regret})]&\le \tilde{\mathcal{O}}(\sqrt{T})+\frac{1}{\sigma_{\pf}}\tilde{\mathcal{O}}(\sqrt{T})-\frac{\sigma_{\pf}}{6}\E\left[\sum_{t=N+2\overline{H^+}+1}^{T}\|M_t-\varphi(M_0)\|_F^2\right]. 
\end{align*}
\end{proposition}
\begin{proof}
First, we decompose the regret with respect to $\hat{F}_{t}$ into two parts:
\begin{align*}
\underbrace{\left(\sum_{t=N+2\overline{H^+}+1}^{T} \pF_{t}(\widetilde{M}_{t-\bar{H}:t})-\pf_{t}(M_{t})\right)}_{(\text{estimation $+$ movement cost})}+\underbrace{\left(\sum_{t=N+2\overline{H^+}+1}^{T} \pf_{t}(M_{t})-\sum_{t=N+2\overline{H^+}+1}^T \pf_{t}(\varphi(M_0))\right)}_{(\text{bandit RFTL-D regret with erroneous gradient})}. 
\end{align*}
The movement cost is bounded similarly as in the analysis of Algorithm~\ref{alg:BCO-quadratic}. In particular, 
\begin{align*}
&\E\left[\sum_{t=N+2\overline{H^+}+1}^{T} \pF_{t}(\widetilde{M}_{t-\bar{H}:t})-\pF_{t}(M_{t-\bar{H}:t})\right]\\
&\le \sum_{t=N+2\overline{H^+}+1}^{T}\E\left[\E\left[\langle\nabla\pF_{t}(M_{t-\overline{H^+}:t}),A_{t-\overline{H^+}:t}\eps_{t-\overline{H^+}:t}\rangle+\frac{\beta_{\pF}}{2}\sum_{s=t-\overline{H^+}}^t \|A_s\eps_s\|_2^2\mid \mathcal{F}_{t-H^+}\cup\mathcal{G}_{t-H^+}\right]\right]\\
&\le \frac{\beta_{\pF}}{2}  \sum_{t=N+2\overline{H^+}+1}^{T}\sum_{s=t-\overline{H^+}}^t \|A_s^2\|_{\mathrm{op}}\le \frac{\beta_{\pF}}{\eta\sigma_{\pf}} \sum_{t=N+2\overline{H^+}+1}^{T}\sum_{s=t-\overline{H^+}}^t \frac{1}{s}\le \frac{\beta_{\pF}H^+\log T}{\eta\sigma_{\pf}}\le \tilde{\mathcal{O}}\left(\frac{\beta_c}{\sigma_c}\sqrt{T}\right),
\end{align*}
and 
\begin{align*}
\E\left[\sum_{t=N+2\overline{H^+}+1}^T \pF_{t}(M_{t-\overline{H^+}:t})-\pf_t(M_t)\right]&\le L_{\pF} \sum_{t=N+2\overline{H^+}+1}^T \|M_{t-\overline{H^+}:t}-(M_t,\dots,M_t)\|_F\\
&\le \frac{16\sqrt{\eta}d_{\y}d_{\uv}L_{\pF}\hat{B}(H^+)^3}{\sqrt{\sigma_{\pf}}} \sum_{t=N+2\overline{H^+}+1}^T \frac{1}{\sqrt{t-2\overline{H^+}}} \\
&\le  \frac{16\sqrt{\eta T}d_{\y}d_{\uv}L_{\pF}\hat{B}(H^+)^3}{\sqrt{\sigma_{\pf}}}.
\end{align*}
Combining, we have a bound on the estimation and movement cost:
\begin{align*}
(\text{estimation $+$ movement cost})\le \tilde{\mathcal{O}}\left(\frac{\beta_c}{\sigma_c}d_{\y}d_{\uv}\sqrt{T}\right). 
\end{align*}
To bound the second term, we first bound the gradient error $\|\nabla \hat{f}_t(M_t)-\nabla \pf_t(M_t)\|_F$, which measures the gradient error in using $\hat{f}_{t}$ to approximate $\pf_t$ in the algorithm.

\begin{lemma} [Gradient error in estimating pseudo-losses] 
\label{lem:gradient-error-estimate-pseudo}
Let $\hat{f}_{t}$ and $\pf_t$ be given as in Definition~\ref{def:with-history-loss-unknown} and \ref{def:pseudo-loss}. Let $M_t\in\M(H^+,R^+)$ played by Algorithm~\ref{alg:BCO-control}. Then, $\forall t$,
\begin{align*}
\left\|\nabla \pf_t(M_t)-\nabla \hat{f}_{t}(M_t)\right\|_F \le \tilde{\mathcal{O}}(1)\eps_G. 
\end{align*}
\end{lemma}
\begin{proof}
Consider the function $\hat{v}(\cdot\mid G):\M(H^+,R^+)\rightarrow \R^{d_{\y}+d_{\uv}}$ parametrized by $G$ given by $\hat{v}(N\mid G)=(\hat{\mathbf{S}}_{N,G},\hat{\mathbf{C}}_{N})$, where $\hat{\mathbf{S}}_{N,G}=\ynat_t+\sum_{i=1}^{\bar{H}}G^{[i]}\sum_{j=0}^{\overline{H^+}}N^{[j]}\hat{\y}_{t-i-j}^{\mathbf{nat}}$ and $\hat{\mathbf{C}}_{N}=\sum_{j=0}^{\overline{H^+}}N^{[j]}\hat{\y}_{t-j}^{\mathbf{nat}}$. Then the gradient difference can be decomposed as
\begin{align*}
\nabla\pf_t(M_t)-\nabla \hat{f}_{t}(M_t)&=\mathbf{D}\hat{v}(M_t\mid\hat{G})\cdot(\nabla c_t)(\hat{\mathbf{S}}_{M_t,\hat{G}},\hat{\mathbf{C}}_{M_t}) -  \mathbf{D}\hat{v}(M_t\mid G)\cdot(\nabla c_t)(\hat{\mathbf{S}}_{M_t,G},\hat{\mathbf{C}}_{M_t}))\\
&= \underbrace{\mathbf{D}\hat{v}(M_t\mid\hat{G})\cdot((\nabla c_t)(\hat{\mathbf{S}}_{M_t,\hat{G}},\hat{\mathbf{C}}_{M_t}))-(\nabla c_t)(\hat{\mathbf{S}}_{M_t,G},\hat{\mathbf{C}}_{M_t})))}_{(1)} \\
& \ \ \ \ \ + \underbrace{(\mathbf{D}\hat{v}(M_t\mid\hat{G})-\mathbf{D}\hat{v}(M_t\mid G))\cdot(\nabla c_t)(\hat{\mathbf{S}}_{M_t,G},\hat{\mathbf{C}}_{M_t}))}_{(2)},
\end{align*}
where
\begin{align*}
(1)&\le \|\mathbf{D}\hat{v}(M_t\mid \hat{G})\|_{2}\beta_c \left\|\sum_{i=H}^t(G^{[i]}-\hat{G}^{[i]})\sum_{j=0}^{\overline{H^+}}M_t^{[j]}\hat{\y}_{t-i-j}^{\mathbf{nat}}\right\|_2\le \|\mathbf{D}\hat{v}(M_t\mid\hat{G})\|_{2}\eps_GR_{\hat{\uv}},\\
(2)&\le \|\mathbf{D}\hat{v}(M_t\mid\hat{G})-\mathbf{D}\hat{v}(M_t\mid G)\|_{2} L_c\sqrt{R_{\hat{\y}}^2+R_{\hat{\uv}}^2},
\end{align*}
where $R_{\hat{\y}}, R_{\hat{\uv}}$ are established in Lemma~\ref{lem:yyu-norm-bound-unknown}. We may further bound $\|\mathbf{D}\hat{v}(M\mid \hat{G})\|_{2}\le 4R_{G}R_{\mathrm{nat}}\sqrt{H^+}$ and $\|\mathbf{D}\hat{v}(M\mid \hat{G}-G)\|_{2}\le 4\eps_{G}R_{\mathrm{nat}}\sqrt{H^+}$ by identical analysis as in Lemma~\ref{lem:regularity-unknown}. 
Combining, we have established the bound on the gradient difference between the true and pseudo-loss functions:
\begin{align*}
\left\|\nabla\hat{f}_{t}(M_t)-\nabla \pf_t(M_t)\right\|_F\le \tilde{\mathcal{O}}(1)\eps_G. 
\end{align*}
\end{proof}

With Lemma~\ref{lem:gradient-error-estimate-pseudo}, we are ready to establish the following corollary to Lemma~\ref{lem:bandit-rftl-d-error} that gives the regret inequality with respect to $\pf_t$:
\begin{corollary} [Pseudo-loss regret inequality]
\label{cor:erroneous-regret-inequality} Let $L_g=\sup_t\|g_t\|_{t,t+1}^*$, where $g_t$ is the gradient estimator used in the bandit controller outlined in Algorithm~\ref{alg:BCO-control}. Then the following regret inequality holds:
\begin{align*}
\E\left[\sum_{t=N+2\overline{H^+}+1}^T \pf_{t}(M_t)-\pf_{t}(\varphi(M_0))\right]&\le L_g \E\left[\sum_{t=N+2\overline{H^+}+1}^T\left\|M_{t-\overline{H^+}}-M_{t+1}\right\|_{t,t+1}\right]+3H^+\hat{B}\\
& \ \ \ \ \ +\frac{\nu\log(T)}{\eta}-\frac{\sigma_{\pf}}{6} \E\left[\sum_{t=N+H^+}^{T-\overline{H^+}}\|M_t-M\|_F^2\right]\\
& \ \ \ \ \ +\frac{3}{\sigma_{\pf}}\tilde{\mathcal{O}}(1)\eps_G^2T+ \frac{32\sqrt{\eta T}\beta_{\hat{F}}d_{\uv}d_{\y}\hat{B}\hat{D}(H^+)^4}{\sqrt{\sigma_{\hat{f}}}}. 
\end{align*}
\end{corollary}
\begin{proof}
By Proposition~\ref{prop:gradient-est-main-prop}, the gradient estimator $g_t$ constructed in Algorithm~\ref{alg:BCO-control} satisfies the following bias guarantee $\forall t\ge N+2\overline{H^+}+1$,
\begin{align}
\label{equation:bandit-control-gradient-error-unknown}
\left\|\E[g_t\mid \mathcal{F}_{t-H^+}\cup\mathcal{G}_{t-H^+}]-\E[\nabla \hat{f}_{t}(M_t)\mid \mathcal{F}_{t-H^+}\cup\mathcal{G}_{t-H^+}]\right\|_2 \le \frac{16\sqrt{\eta}\beta_{\hat{F}}d_{\uv}d_{\y}\hat{B}(H^+)^3}{\sqrt{\sigma_{\hat{f}}(t-N-2\overline{H^+})}}. 
\end{align}
In the setting of Section~\ref{sec:rftl-d-error}, take $\delta_{t}=\nabla\hat{f}_{t}(M_t)-\nabla \pf_t(M_t)$. The gradient estimator $g_t$ used in Algorithm~\ref{alg:BCO-control} obeys $\|\E[g_t\mid \mathcal{F}_{t-H}\cup\mathcal{G}_{t-H}]-\E[\nabla \hat{f}_{t}(M_t)\mid \mathcal{F}_{t-H}\cup\mathcal{G}_{t-H}]\|_F\le B(t)$ with $B(t)=\frac{16\sqrt{\eta}\beta_{\hat{F}}d_{\uv}d_{\y}\hat{B}(H^+)^3}{\sqrt{\sigma_{\hat{f}}(t-N-2\overline{H^+})}}$ as given by \ref{equation:bandit-control-gradient-error-unknown}. Take $\|\cdot\|_{(t)}=\|\cdot\|_{t,t+1}^*$ and $\|\cdot\|_{(t),*}=\|\cdot\|_{t,t+1}$, $B=\hat{B}$, $H=H^+$, $\sigma=\sigma_{\pf}$, Lemma~\ref{lem:bandit-rftl-d-error} and Lemma~\ref{lem:gradient-error-estimate-pseudo} imply
\begin{align*}
\E\left[\sum_{t=N+2\overline{H^+}+1}^T \pf_{t}(M_t)-\pf_{t}(\varphi(M_0))\right]&\le L_g \E\left[\sum_{t=N+2\overline{H^+}+1}^T\left\|M_{t-\overline{H^+}}-M_{t+1}\right\|_{t,t+1}\right]+3H^+\hat{B}\\
& \ \ \ \ \ +\frac{\nu\log(T)}{\eta}-\frac{\sigma_{\pf}}{6} \E\left[\sum_{t=N+H^+}^{T-\overline{H^+}}\|M_t-\varphi(M_0)\|_F^2\right]\\
& \ \ \ \ \ +\frac{3}{\sigma_{\pf}}\tilde{\mathcal{O}}(1)\eps_G^2T+ \frac{32\sqrt{\eta T}\beta_{\hat{F}}d_{\uv}d_{\y}\hat{B}\hat{D}(H^+)^4}{\sqrt{\sigma_{\hat{f}}}}. 
\end{align*}
\end{proof}

We proceed to establish a local norm bound for $H$-steps-apart iterates. Let $\Phi:\M(H^+,R^+)^{H^+}\rightarrow\R$ be given by $\Phi_t(M)\defeq \eta\bigg(\sum_{s=N+H^+}^t \langle g_{s-\overline{H^+}},M\rangle+\frac{\sigma_{\pf}}{4}\|M-M_{s-\overline{H^+}}\|_F^2\bigg)+R(M)$. Recall that $R_t(M)\defeq R(M)+\frac{\eta\sigma_{\pf}}{4}\sum_{s=N+H^+}^t\|M-M_{s-\overline{H^+}}\|_F^2$. By definition of Algorithm~\ref{alg:BCO-control}, the optimality condition, and linearity of $\Phi_t-R_t$, we have
\begin{align*}
\frac{1}{2}\|M_{t-\overline{H^+}}-M_{t+1}\|_{t,t+1}^2&=D_{R_t}(M_{t-\overline{H^+}},M_{t+1})\\
&\le \Phi_t(M_{t-\overline{H^+}})-\Phi_t(M_{t+1})\\
&= \underbrace{(\Phi_{t-H^+}(M_{t-\overline{H^+}})-\Phi_{t-H^+}(M_{t+1}))}_{\le 0}+\eta\sum_{s=t-\overline{H^+}}^t \langle g_{s-\overline{H^+}}, M_{t-\overline{H^+}}-M_{t+1} \rangle\\
& \ \ \ \ \ +\frac{\eta\sigma_{\pf}}{4} \sum_{s=t-\overline{H^+}}^t \|M_{t-\overline{H^+}}-M_{s-\overline{H^+}}\|_F^2-\|M_{t+1}-M_{s-\overline{H^+}}\|_F^2\\
&\le \eta \|M_{t-\overline{H^+}}-M_{t+1}\|_{t,t+1}\left\|\sum_{s=t-\overline{H^+}}^t g_{s-\overline{H^+}}\right\|_{t,t+1}^*+\frac{\eta\sigma_{\pf} \hat{D}^2H^+}{4},
\end{align*}
which implies
\begin{align*}
\|M_{t-\overline{H^+}}-M_{t+1}\|_{t,t+1}\le \max\left\{4\eta \left\|\sum_{s=t-\overline{H^+}}^t g_{s-\overline{H^+}}\right\|_{t,t+1}^*, 2\sqrt{\eta\sigma_{\pf} H^+}\hat{D}\right\}. 
\end{align*}
Lemma~\ref{lem:grad-est-norm-bound} established that $\forall t-\overline{H^+}\le s\le t$, $\|g_{s-\overline{H^+}}\|_{t,t+1}^*\le 8d_{\uv}d_{\y}\hat{B}{H^+}^4$ and $\|g_{t}\|_{t,t+1}^*\le 4d_{\uv}d_{\y}\hat{B}{H^+}^4$ hold deterministically. Plugging $L_g=4d_{\uv}d_{\y}\hat{B}{H^+}^4$ and the iterate bounds into the bound obtained in Corollary~\ref{cor:erroneous-regret-inequality} and take step size $\eta=\mathcal{O}(\frac{1}{d_{\y}d_{\uv}\hat{B}H^3\sqrt{T}})$, we have
\begin{align*}
\E[(\pf_{t}\text{-BCO-M-regret})]&\le \tilde{\mathcal{O}}(\sqrt{T})+\frac{1}{\sigma_{\pf}}\tilde{\mathcal{O}}(1)\eps_G^2T-\frac{\sigma_{\pf}}{6}\E\left[\sum_{t=N+H^+}^{T-\overline{H^+}}\|M_t-\varphi(M_0)\|_F^2\right]\\
&=\tilde{\mathcal{O}}(\sqrt{T})+\frac{1}{\sigma_{\pf}}\tilde{\mathcal{O}}(\sqrt{T})-\frac{\sigma_{\pf}}{6}\E\left[\sum_{t=N+2\overline{H^+}+1}^{T}\|M_t-\varphi(M_0)\|_F^2\right]. 
\end{align*}
\end{proof}
Combining the bounds on burn-in loss, algorithm estimation loss, $\pf_t$-BCO-M regret, and comparator estimation loss and taking $\tau=\frac{\sigma_{\hat{f}}}{6}$, 
\begin{align*}
\E\left[\regret_T(\texttt{controller})\right]&\le \underbrace{\tilde{\mathcal{O}}(\sqrt{T})}_{(\text{burn-in loss})} + \underbrace{\tilde{\mathcal{O}}(1)}_{(\text{algorithm estimation loss})}+\\
& \ \ \ \ \ \underbrace{\left(\tilde{\mathcal{O}}(\sqrt{T})+\frac{1}{\sigma_{\pf}}\tilde{\mathcal{O}}(\sqrt{T})-\frac{\sigma_{\pf}}{6}\E\left[\sum_{t=N+2\overline{H^+}+1}^T\|M_t-\varphi(M_0)\|_F^2\right]\right)}_{(\text{$\pf_t$-BCO-M-regret})}+\\
& \ \ \ \ \ \underbrace{\left(\tilde{\mathcal{O}}\left(\frac{\beta_c}{\sigma_c}d_{\y}d_{\uv}\sqrt{T}\right)+\frac{1}{\tau}\tilde{\mathcal{O}}(\sqrt{T})+\tau\E\left[\sum_{t=N+2\overline{H^+}+1}^T \|M_t-\varphi(M_0)\|_F^2\right]\right)}_{(\text{comparator estimation loss})}\\
&\le \tilde{\mathcal{O}}\left(\frac{\beta_c}{\sigma_c}d_{\y}d_{\uv}\sqrt{T}\right). 
\end{align*}
\newpage

\end{document}